\documentclass[twoside]{article}
\usepackage{amsfonts, amsmath} 

\usepackage[preprint]{aistats2026}

\usepackage[round]{natbib}

\usepackage{booktabs}
\usepackage{hyperref}
\usepackage{enumitem}
\usepackage{amsmath}
\usepackage{amssymb}
\usepackage{mathtools}
\usepackage{amsthm}
\usepackage{bm}
\usepackage{xcolor}
\usepackage{algorithm}
\usepackage{algpseudocode}
\usepackage{multicol}
\usepackage{titletoc}
\usepackage{etoc}

\newcommand{\R}{\mathbb{R}} 
\newcommand{\E}{\mathbb{E}}
\newcommand{\EE}{\mathbb{E}}
\newcommand{\PP}{\mathbb{P}}
\newcommand{\AAA}{\mathcal{A}}
\newcommand{\FFF}{\mathcal{F}}
\newcommand{\HH}{\mathcal{H}}

\newcommand{\mat}[1]{\mathbf{#1}}
\newcommand{\ERMM}{\text{ERMM}}
\newcommand{\IRG}{\text{IRG}}

\theoremstyle{plain}
\newtheorem{theorem}{Theorem}[section]
\newtheorem{proposition}[theorem]{Proposition}
\newtheorem{lemma}[theorem]{Lemma}
\newtheorem{corollary}[theorem]{Corollary}
\theoremstyle{definition}

\newtheorem{assumption}[theorem]{Assumption}
\theoremstyle{remark}
\newtheorem{remark}[theorem]{Remark}

\begin{document}

\runningtitle{A Pure Hypothesis Test for IRG Models Based on a KSD}
\twocolumn[

\aistatstitle{A Pure Hypothesis Test for Inhomogeneous Random Graph Models Based on a Kernelised Stein Discrepancy}

\aistatsauthor{ Anum Fatima \And Gesine Reinert}

\aistatsaddress{ Department of Statistics,  University of Oxford and \\
Lahore College for Women University \And Department of Statistics \\
University of Oxford} 
]

\begin{abstract}
  Complex data are often represented as a graph, which in turn can often be viewed as a realisation of a random graph, such as an inhomogeneous random graph model (IRG). For general fast goodness-of-fit tests in high dimensions, kernelised Stein discrepancy (KSD) tests are a powerful tool. Here, we develop a KSD-type test for IRG models that can be carried out with a single observation of the network. The test applies to a network of any size, but is particularly interesting for small networks for which asymptotic tests are not warranted. We also provide theoretical guarantees.
\end{abstract}

\section{INTRODUCTION}
 
Networks are often used to represent complex data for data mining in application areas such as transportation (\cite{transp1}, \cite{transp2}), management (\cite{manag1}, \cite{manag2}, \cite{manag3}), biology (\cite{bio1}, \cite{bio2}), and social science (\cite{social1}, \cite{social2}, \cite{social3}). For example, social networks represent actors as vertices and the relations between them as edges. The edges in these networks are often modelled as random, with different probabilities for different edges. In this paper, we focus on a model in which edges are assumed to occur independently of each other; this network model is known as the inhomogeneous random graph model (IRG). Particular instances are stochastic blockmodels with known vertex classes and edge probabilities, also called Erd\H{o}s-R\'enyi mixture graph model (ERMMs) in \cite{daudin2008mixture}, as well as some degree-corrected stochastic blockmodels suggested in \citet{karrer_newman_2011}. Under additional assumptions on how edge probabilities arise, model fitting in an IRG is often computationally feasible; see, for example, \cite{Santos2021, liu2025stochastic, karwa2024monte}. However, for downstream tasks such as edge validation, it is important that the fitted model actually fits the data reasonably well. How can we assess such a fit?
 
Classical statistical goodness-of-fit tests based, for example, on chi-square asymptotics can show poor performance in high dimensions, see \citet{arias2018remember}. For general IRG models, the only rigorous available test, given by \cite{dan2020goodness}, tests the hypothesis that the observed network is generated with edge probabilities given by a reference edge probability matrix,  against the specific alternative that some norm of the difference between estimated and reference edge probability matrices is large. Thus, it is not designed to capture alternatives with the same edge probabilities but dependence between the edge indicators. The test is asymptotic with the asymptotic distribution of the test statistic depending on the sparsity regime of the network (see also \citet{chakrabarty2021spectra}); for a given network, the choice of asymptotic regime may not be clear. 

For stochastic blockmodels (SBMs), aside from likelihood ratio tests as in \cite{karwa2024monte},  many tests provide methods to estimate the number of blocks in an SBM using spectral properties of the network, with a setup that tests the hypothesis of $K_0$ blocks against an alternative of more than $K_0$ blocks, see for example, \citet{Lei2016}, \citet{DONG2020}, \citet{Hu2021}, and \citet{WU2024}. \citet{Jin2025} developed goodness-of-fit metrics for models in the block-model family, which, if the assumed model is correct, converge to a standard normal distribution, with theoretical guarantees under asymptotic conditions which cannot be verified for a fixed number of vertices. 

For small to moderately sized networks, asymptotic approximations may be unreliable while exact methods are computationally infeasible. The likelihood ratio tests based on the assumption of independent edges, such as in \citet{karwa2024monte}, assume that the alternative model also has independent edges; in real networks, however, such as in friendship networks, edges may not occur independently of each other. Hence, as pointed out in \citet{Jin2025}, there is a severe gap regarding tests for the fit of IRG models for small to moderate size networks against alternatives which do not assume independent edges.

\citet{xu2021stein} devised the so-called  {\it graph kernel Stein statistic} (gKSS) to assess the fit of exponential random graph models, employing theoretical results from \citet{reinert2019approximating}, and adapting methodology from \citet{chwialkowski2016kernel} and \citet{liu2016kernelized}. This pure hypothesis test does not require any assumptions on an asymptotic regime, does not assume any form of the alternative model, and applies to any size of network, even when only a single network observation is available. 

In this paper, we extend the methodology from \cite{xu2021stein} to the IRG model. We give theoretical guarantees on the distribution of the test statistic, which, unlike the results from \cite{xu2021stein}, hold even when the networks are asymptotically sparse.  

Instead of testing the fit of a model family, the test proposed in this paper tests the fit of a specified model, so that the null hypothesis is simple. 

Our main contributions are as follows.

\begin{enumerate}[nosep]
    \item We provide a novel { non-asymptotic kernelised} test to assess the fit of a IRG model.
    \item {We show that this test is well-suited to analyse networks for which an asymptotic regime may not be justified.}
    \item We illustrate on synthetic data that the test can outperform likelihood-based tests on networks with dependent edges.
    \item We show that on four real networks, the test gives plausible results.
    \item We provide theoretical guarantees for the test procedure.
\end{enumerate}

\section{BACKGROUND}
 
\subsection{Inhomogeneous Random Graph Models}\label{sec:irg}

An IRG model, as proposed in \citet{Bollobas2007}, is a model for a simple, unweighted, undirected graph $\mathcal{G} = (\mathcal{V},\mathcal{E})$ with vertex set $\mathcal{V}$, of size $n$ and edge set $\mathcal{E}$. The graph can be represented by a collection of 0-1 entries $\mat{x}= (x_{u,v}, 1 \le u < v \le n)$ of size $N = \binom{n}{2}$. We set $E = \{ (u,v), 1 \le u < v \le n\}$,  for the set of vertex pairs, so that $|E| = N$.  With $s = (u,v) \in E$  a vertex pair we write $p_s := p_{u,v} =\mathbb{P}((u,v) \in \mathcal{E})$, and we denote by $\mat{p} = \{p_{u,v}, u, v = 1, \ldots, n\}$, the matrix of edge probabilities of all vertex pairs. Edges occur independently of all other edges in the network. The likelihood of $\mat{x}$ under an IRG model with edge probability matrix $\mat{p}$, or short: under IRG$(\mat{p})$, is
\begin{equation} \label{eq:irg}
    \mathbb{P}(\mat{X} = \mat{x}) = \prod_{{1 \le u < v \le n}} \left(1-p_{u,v} \right) \left(\frac{p_{u,v}}{1-p_{u,v}} \right)^{x_{u,v}}.
\end{equation}

Two instances of IRG models that have attracted particular attention in network science start with the premise that each vertex $u$ has an associated value (or feature) $g_u$; and that the edge probabilities depend on features of the two vertices which the edge connects. In this paper, we assume that the vertex features $g_u$, $1 \le u \le n$, are known. The first instance, the Erd\H{o}s-R\'enyi Mixture Models (ERMM), as in \cite{daudin2008mixture}, is a stochastic blockmodel with vertex set $\mathcal{V}$ which is divided into $L$ homogeneous groups of vertices, with the probability of an edge depending on the known group membership of the vertices it connects. In this case, $g_u$ is the group membership of the vertex $u$; the probability $Q_{i,j}$ that two vertices $u$ and $v$ form an edge depends only on their types $i$ and $j$. Thus, the probability of an edge between a vertex $u$ and $v$ under ERMM is  $p_{u,v} = Q_{g_u,g_v}$; we set $\mat{Q} = \{Q_{i,j}; i,j = 1, \ldots, L\} $, which is an $L \times L$ symmetric matrix with entries in $[0,1]$. In this paper we denote an ERMM with parameters $\mat{n}$ and $\mat{Q}$ by $\ERMM(\mat{n}, \mat{Q})$, where $\mat{n} = \{n_i; i, \ldots, L\}$ is the vector of group sizes, so that $\sum_{i=1}^L n_i = n$.
 
The second special case of an IRG is the degree-corrected stochastic block model (DCSBM) by \citet{karrer_newman_2011}.  In the setting of ERMMs, a DCSBM has as extra parameters a vector $\bm{\theta} = \{\theta_v\}_{v=1}^n$ that is interpreted as propensities of vertices to form an edge, and sets $p_{u,v} = \theta_u \theta_v Q_{g_u, g_v}$.

When fitting a DCSBM to data, to avoid complications arising when for estimated parameters, {$\hat{\theta}_u \hat{\theta}_v \hat{Q}_{g_u, g_v}>1$}, we use the Poissonised version, as in \cite{karrer_newman_2011}. This version assumes that the number of edges between distinct vertices $u$ and $v$ follows a Poisson distribution with mean $\theta_u\theta_v Q_{g_u,g_v}$, independently of the edge indicators for other vertex pairs, and that $X_{u,u}$ follows a Poisson distribution with mean $\theta_u^2 Q_{g_u,g_u}/2$. We then use the {\it Bernoulli-Poisson} link $\widehat{p}_{u,v} = 1- e^{-\hat{\theta}_u \hat{\theta}_v \hat{Q}_{g_u,g_v}},  \label{eq:dcsbm}$ see \cite{bernoulli-poisson}.
In sparse graphs, 
$ 1- e^{-\theta_u\theta_vQ_{g_u,g_v}} \approx \theta_u\theta_vQ_{g_u,g_v}$ and hence this model approximates the edge probabilities of the DCSBM.

\subsection{Stein's Method}

Stein's method, originating in \citet{stein1972bound}, provides a means to compare probability distributions through characterising operators. In a nutshell, for $p$ a target distribution, a {\it{Stein operator}} ${\AAA}_p$ with Stein class $\FFF(\AAA_p)$ is such that if $X $ has distribution $ p$ (short: $X \sim p$) then 
$\E \AAA_p f(X) = 0  \mbox{ for all }  f \in \FFF(\AAA_p).$ If $W\sim q$ is any random element which is close in distribution to $X$, then intuitively, it should hold that 
$\E  \AAA_p f(W) \approx 0.$ In \citet{gorham2015measuring}, it is proposed to quantify this intuition by assessing a so-called {\it Stein discrepancy}
\begin{align}\label{eq:steindiscr}
    S(p,q; \mathcal{H}) = \sup_{f \in \mathcal{H}} |  {\E}\AAA_p f(W)|.
\end{align}
Choosing as $\mathcal{H}$ the class of all 1-Lipschitz functions gives the Wasserstein-1 distance $ || p - q ||_1$. 

\subsection{Kernel Stein Discrepancies}

Evaluating the supremum in \eqref{eq:steindiscr} over a large class of functions, such as that needed for Wasserstein distance, is usually computationally not possible, as observed in \citet{gorham2015measuring}. Instead, to test the fit of a continuous distribution, \citet{chwialkowski2016kernel} and \citet{liu2016kernelized} propose to use a so-called kernel Stein discrepancy, obtained by taking set $\mathcal{H}$ as a reproducing kernel Hilbert space (RKHS) associated with kernel $k$, inner product $\langle \cdot, \cdot \rangle$ and unit ball $B_1(\mathcal H)$.  Let  $Y \sim q$. Formally, the {\it{kernel Stein discrepancy (KSD)}}   between  $p$ and $q$ is given by 
$
 \mbox{KSD}(p,q; k) =\sup_{f \in B_1(\mathcal H)} |  \mathbb{E}[\AAA_p f (Y) ] |.
$

Under some assumptions, $ \mbox{KSD}(p,q; k)$ can be evaluated explicitly. For a continuous pdf $p$, under suitable conditions, the operator $\AAA_p f(w) = {f(w) \nabla \log p(w)} + \nabla f(w)$ is a Stein operator acting on vector-valued functions $f$. Set $h_p (x,y) =  \langle \AAA_p k(x, \cdot), \AAA_p k(\cdot, y) \rangle.$ Then if $Y,Y' \sim q$ are independent, we have $ \mbox{KSD}(p,q; k)^2 =  \E h_p (Y,Y').$ Moreover, in this continuous setting, if the distribution $q$ is not available in closed form, then KSD can be estimated from i.i.d.\,samples $(y_i)_ {i=1, \ldots, n} \sim q$ and $ (y_j')_{j=1, \ldots, n} \sim q$. Manipulations using the structure of the RKHS show that a natural estimator for $ \mbox{KSD}(p,q; k)^2$ is given by 
$\widehat{ \mbox{KSD}(p,q; k)^2}  =  \frac{1}{n^2} \sum_{i=1}^n  \sum_{j=1}^n h_p (y_i, y_j').$ More background on RKHS and KSD is provided in the Supplementary Material (Supp.\,Mat.) Section \ref{sec:RKHS}. While restricting the Stein discrepancy to an RKHS is convenient, it introduces a trade-off, particularly when the kernel is weak; see Supp.\,Mat.\,Section \ref{app:kernel} for further discussion.

For assessing goodness of fit to network distributions, two complications arise: First, the distribution is discrete on the set of networks, and hence, a different Stein operator is needed. Second, there is often only one observed network; hence, the KSD cannot be easily estimated. In Section \ref{sec:gof}, we give a pure hypothesis test for IRG models developed by borrowing the ideas from KSD tests.

\section{A STEIN TEST}\label{sec:gof}

Here, we develop a KSD-type test for IRG models; theoretical guarantees are given in Section \ref{sec:theory}. First, we introduce some notation. For a simple, unweighted, undirected graph $\mathcal{G} = (\mathcal{V},\mathcal{E})$ with vertex set $\mathcal{V}$, of size $n$, and edge set $\mathcal{E}$, and collection of edge indicators $\mat{x} = (x_{u,v}) \in E$, as in \cite{xu2021stein} we denote, for $s = (u,v)$ 
    
\noindent    $\mat{x}^{(s,1)}$, the collection with $1$ at the $s-$ coordinate and the same as $\mat{x}$ otherwise,
\newline 
    $\mat{x}^{(s,0)}$, the collection with $0$ at the $s-$ coordinate and the same as $\mat{x}$ otherwise,
    \newline 
    $\mat{x}_{-s}$, the collection $\{x_{r,t}, 1 \le r < t \le n, (r,t) \neq (u,v) \}$, without the $s-$ coordinate.

\subsection{An IRG Stein Operator}
    
For a function $f: \{0,1\}^{N} \rightarrow \mathbb{R} $ we set $\Delta_s f(\mat{x}) =  f(\mat{x}^{(s,1)}) - f(\mat{x}^{(s,0)})$, and we set
\begin{align}\label{SE_ermm_op}
 \AAA_{\IRG} f(\mat{x}) = \frac{1}{N} \sum_{s \in E} \mathcal{A}_{\IRG}^{(s)} f(\mat{x}),
\end{align}
with
\begin{align}  \label{GD_Stein_Eq}
    \mathcal{A}_{\IRG}^{(s)} f(\mat{x}) = & p_s \left( f(\mat{x}^{(s,1)}) - f(\mat{x})\right) \nonumber \\
    & + \left(1 - p_s     \right) \left( f(\mat{x}^{(s,0)}) - f(\mat{x}) \right).
\end{align}
A detailed theoretical underpinning of this operator choice {and the proof of the following Proposition} can be found in Supp.\,Mat.\,Section \ref{app:stein}.

\begin{proposition}\label{prop:stein}
For a graph $\mathcal{G} = (\mathcal{V},\mathcal{E})$ with adjacency  matrix $\mat{X} \sim \IRG(\mat{p})$, the operator \eqref{SE_ermm_op} is a Stein operator with Stein class $\mathcal{F}(\AAA) = \{ f: \{0,1\}^{N} \rightarrow \R\}$, that is, for all $f: \{0,1\}^{N} \rightarrow \R$, 
\begin{equation} \label{stein_op}
    \E\AAA_{\IRG} f(\mat{X}) = 0.
\end{equation}
\end{proposition}

\subsection{A Graph Kernel Stein Statistic for IRGs }

Let $\mathcal{H}$ be an RKHS with kernel $K$ and inner product $\langle \cdot, \cdot \rangle$. As for an IRG, $\sup_{f \in B_1(\mathcal H)} |  \mathbb{E}[\AAA_{IRG} f (Y) ] |$ is usually not observed, in analogy to \citet{xu2021stein} we introduce the empirical IRG-graph kernel Stein statistic IRG-gKSS
$$\text{IRG-gKSS}(\mat{p};\mat{x}) = \underset{\|f\|_{\mathcal{H}} \le 1}{\sup} \Big| \frac{1}{N} \sum_{s \in E} \AAA_{IRG}^{(s)}f(\mat{x})\Big|.$$ 
Here $f \in \{f \in \mathcal{H}: \|f\|_{\mathcal{H}} \le 1\}$ are functions in the unit ball of ${\mathcal{H}} $, and $\mathcal{A}_{\IRG}^{(s)}$ is given in \eqref{GD_Stein_Eq}. Since the functions $f$ are in the unit ball of the underlying RKHS, the supremum can be calculated exactly, giving
\begin{equation} \label{t_stat_s_net}
   {{\text{IRG-gKSS}}^2}(\mat{p};\mat{x}) = \frac{1}{N^2} \sum_{s,s' \in E} h_{\mat{x}}(s,s'),
\end{equation}
with $h_{\mat{x}} (s,s') = \left\langle \AAA_{\IRG}^{(s)}K(\mat{x},\cdot) , \AAA_{\IRG}^{(s')} K(\cdot,\mat{x})\right\rangle $.

\subsection{The IRG-gKSS Test}

For an observed network $\mat{x}$, for which we want  to test the hypothesis $\text{H}_0: \mat{X} \sim {\text{IRG}(\mat{p}_0)}$, against a general alternative, we carry out a Monte Carlo test using Algorithm \ref{alg:algo1}. We do not detail here how $\mat{p}_0$ in the null hypothesis is obtained; sometimes it is determined from first principles, and sometimes a parametric model with learned parameter values is used, see for example \citet{liu2025stochastic}. 

\begin{algorithm}[t]
   \caption{Assessing the fit of $\IRG(\mat{p}_0)$}
   \label{alg:algo1}
\begin{algorithmic}
   \State {\bfseries Input:} Observed network $\mat{x}$ \\
   \qquad \quad null edge probably matrix $\mat{p}_0$\\
   \qquad \quad size of null set $M$ \\
    \qquad \quad choice of graph kernel
   \For{$i=1$ {\bfseries to} $M$}
   \State{Simulate $M$ networks from $\IRG(\mat{p}_0)$}
   \State{Compute a set $\bm{\phi}$ of $\text{IRG-gKSS}^2$ using \eqref{t_stat_s_net} \\
    \quad  \, from the simulated networks} 
   \EndFor
   \State Find the empirical quantiles $\gamma_{\alpha/2}$ and $\gamma_{1-\alpha/2}$ from the set $\bm{\phi}$
   \State Compute $\phi = {{\text{IRG-gKSS}}^2}(\mat{p}_0;\mat{x})$ using \eqref{t_stat_s_net}, from the observed network $\mat{x}$
   \State {\bfseries Output:} If $\phi < \gamma_{\alpha/2}$ or $\phi > \gamma_{1-\alpha/2}$ reject the hypothesis that the $\IRG(\mat{p}_0)$ fits the observed network against a general alternative. 
\end{algorithmic}
\end{algorithm}

\begin{algorithm}
   \caption{Assessing the fit of $\IRG(\mat{p}_0)$ with edge re-sampling}
   \label{alg:algo2}
\begin{algorithmic}
   \State {\bfseries Input:} Observed network $\mat{x}$ \\
   \qquad \quad null edge probably matrix $\mat{p}_0$\\
   \qquad \quad re-sampling size B \\
   \qquad \quad size of null set $M$ \\
    \qquad \quad choice of graph kernel
   \For{$i=1$ {\bfseries to} $M$}
   \State Simulate $M$ networks from $\IRG(\mat{p}_0)$
   \State Compute a set $\bm{\phi}$ of $\widehat{\text{IRG-gKSS}^2}$ given in \eqref{t_stat_s_net_resamp}, \\
   \quad \, from a sample of vertex pairs $s_{1,i}, \ldots, s_{B,i}$, \\
   \quad \, chosen uniformly at random and with \\
   \quad \, replacement, from each simulated network. 
   \EndFor
   \State Find the empirical quantiles $\gamma_{\alpha/2}$ and $\gamma_{1-\alpha/2}$ from the set $\bm{\phi}$
   \State Sample a set of vertex pairs $\{s_1, \ldots, s_B$\} uniformly at random and with replacement from the observed network $\mat{x}$
   \State Compute $\phi = {\widehat{\text{IRG-gKSS}}^2}(\mat{p}_0;\mat{x})$ using only the sample of vertex pairs as given in \eqref{t_stat_s_net_resamp}
   \State {\bfseries Output:} If $\phi < \gamma_{\alpha/2}$ or $\phi > \gamma_{1-\alpha/2}$ reject the hypothesis that the $\IRG(\mat{p}_0)$ fits the observed network against a general alternative. 
\end{algorithmic}
\end{algorithm}

For large networks, we use IRG-gKSS with edge re-sampling as given in Algorithm \ref{alg:algo2}. In these tests, $\gamma_\alpha$ is used to denote the empirical $\alpha$-quantile of the set of test statistics computed from the networks simulated under IRG$(\mat{p}_0)$. Ties are broken at random, and interpolation is used when required. We note that this test is a pure significance test, in the sense that only the distribution under the null hypothesis is specified.

In constructing the IRG-gKSS statistic, we evaluate the Stein operator using graph kernels that capture the similarity between the observed graph and its one-edge–perturbed versions (i.e., neighbours at Hamming distance one), and then average only over this neighbourhood, not over the whole space of graphs under the null. As a result, the resulting statistic is not centred at zero and does not represent a discrepancy measure in the usual sense. Consequently, the IRG-gKSS test is two-sided: it takes the supremum of an absolute value, and both unusually low or unusually high values relative to the null distribution are considered evidence against the null hypothesis. The  two-sided $P$-value is $2 \min \left(\frac{\#(\phi_i \le \phi)}{M+1}, \frac{\#(\phi_i \ge \phi)}{M+1}\right)$, with $\bm{\phi} = \{\phi_i\}_{i=1}^M$, the set of IRG-gKSS statistics calculated from $M$ networks simulated under the null model.

\subsection{The Re-sampled IRG-gKSS Test} 

For large networks, evaluating \eqref{t_stat_s_net} can be computer-intensive. As in \citet{xu2021stein}, an edge re-sampling procedure can instead be used. We re-sample a fixed number $B$ of vertex pairs from the network with replacement; let $s_b$, $b = 1, \ldots, B$, be the vertex pairs sampled from $E$. The re-sampled Stein operator is 
\begin{equation}
    \widehat{\AAA_{\IRG}^B}f(\mat{x}) = \frac1B \sum_{b \in [B]} \AAA_{\IRG}^{(s_b)}f(\mat{x}); 
\end{equation}
we estimate its supremum by 
\begin{equation} \label{t_stat_s_net_resamp}
   {\widehat{{\text{IRG-gKSS}}}^2}(\mat{p};\mat{x}) = \frac{1}{B^2} \sum_{b,b' \in [B]} h_{\mat{x}}(s_b,s_b').
\end{equation}
We use the test statistic \eqref{t_stat_s_net} and \eqref{t_stat_s_net_resamp} as a statistic for a Monte Carlo testing procedure. The detailed test procedure is given in Algorithm \ref{alg:algo2}.

\subsection{Choice of Kernel in IRG-gKSS}

The calculated values of IRG-gKSS statistics depend on the kernel used and its corresponding RKHS. Following \citet{xu2021stein}, since the operator in \eqref{SE_ermm_op} embeds a notion of conditional probability, we use a vector-valued reproducing kernel Hilbert space (vvRKHS) introduced in \citet{xu2021stein} to separate the treatment of $x_s$ and $x_{-s}$. For more details on vvRKHS and graph kernels, see the Supplementary Material of \citet{xu2021stein}; Supp.\,Mat.\,Sections \ref{app:kernel} and \ref{sec:add_syn_exp} provide a detailed discussion on kernel choice.

\begin{figure*}
  \centering
  \includegraphics[width=\textwidth]{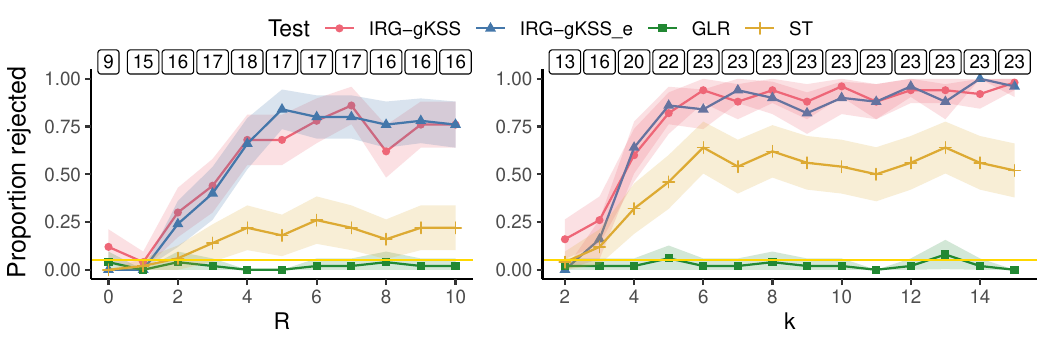}
  \caption{Power of the test to assess the fit of an ERMM$(\mat{n}_{{ub}}, \mat{Q})$ to the network of size {30} with planted hubs. The numbers in boxes at the top of the plot are the average maximum degree observed in $m=50$ repetitions of the test for each setting on the $x$-axis. Left: we fix $k={3}$, the size of the hub, and let $R$, the number of hubs, vary; right: we fix $R= {2}$ and let $k$ vary. Straight yellow line: 5\% level.
  }
  \label{fig:planted_hubs_main}
\end{figure*}

\section{EXPERIMENTS} 
\label{sec:applications}
In this section, we illustrate the IRG-gKSS test by a series of experiments. We compare the IRG-gKSS test to its estimated version, the spectral test (ST) of \citet{Lei2016} with true group memberships and the generalised likelihood ratio (GLR) test. {We note that these three tests target similar, though not identical, null hypotheses; see Supp.\,Mat.\,Section \ref{sec:background} for further details.} All tests are carried out at the 5\% level of significance. In the figures in this section, the shaded area around each point is a confidence band which is computed as $\hat{p}\pm 2\hat{p}(1-\hat{p})/m$, where $\hat{p}$ denotes the proportion of rejections of the null hypothesis across $m$ repetitions of the test. 
First, we use simulated networks before moving to real-world benchmark networks. In all synthetic experiments, unless otherwise stated, we simulate $m=50$ networks for each setting and record the proportion of times the test rejects the fit of the null model. We use the WL graph kernel with $h = 3$.

\paragraph{Implementation Details} 
All experiments are executed on a PC with an Intel(R) Core(TM) i7-4790S CPU and 16.0 GB RAM. Using our {\texttt{R}} code, a single computation of IRG-gKSS$^2$ (WL kernel with $h=3$) for a network, simulated from an ERMM$(\mat{n}, \mat{Q})$ with $\mat{n}$ and $\mat{Q}$ given in the next section, takes $13.69$ seconds. Further discussion about the computation time is given in Supp.\,Mat.\,Section \ref{app:complex}. The code for the experiments can be found at {\url{https://github.com/AnumFatima89/IRG-gKSS}}.

\subsection{Synthetic Data Experiments: Alternatives with Independent Edges}

First, we simulate networks from two alternative models with complex vertex features and test the fit of an ERMM with parameters assumed fixed, and set as the maximum likelihood estimates. Tables \ref{tab:clfalse} and \ref{tab:cltrue} show the results (the rejection rate in $m= 50$ repetitions of the experiment); WL denotes IRG-gKSS using the WL kernel with $h=3$, Graphlet denotes IRG-gKSS using the graphlet kernel, and hats denote the versions in which the edge probabilities are re-estimated from the generated networks. 

\textbf{1.} 
Similarly to \cite{karwa2024monte}
we simulate $m=50$ networks of size $n = 27$  from two DCSBMs with edge probabilities
$\mathbb{P}(u \sim v) = e^{\beta_u + \beta_v + \alpha_{g_u g_v}}$, where the $\beta$'s are chosen uniformly and independently at random,   $\beta_u \sim \log(\text{Unif}(0,1))$, for each simulated network, while the matrix $\alpha$ is fixed. We consider two instances, $\alpha_1 = \begin{pmatrix}
    \log(0.6) & \log(0.2) \\
    \log(0.2) & \log(0.6)
\end{pmatrix}$ and 
$\alpha_2 = \begin{pmatrix}
     \log(0.6) & \log(0.1) \\
   \log(0.1) & \log(0.3)
\end{pmatrix}$. As the expectation of $e^{\beta_u + \beta_v}$ is $\frac14$, up to random fluctuations, the approximating ERMM should have edge probabilities $0.25*e^{\alpha_{g_u,g_v}}$, matching the expected edge probabilities in the DCSBM. For the approximating ERMM, in $m=50$ simulations, we always obtain a total variation distance of 1 (up to computer precision), so these two models should be easy to distinguish. In this experiment using IRG-gKSS, for each of the simulated networks, we test first the fit of an ERMM with $\mat{p}_0 = \{0.25*e^{\alpha_{g_u,g_v}}\}_{u, v = 1, \ldots n}$ and second, the fit of an ERMM with $\hat{\mat{p}}$ estimated from the simulated network using the true group memberships. We note that ST may fit a different model as it tests the null hypothesis of a two-group SBM given the provided group memberships.

The rejection rates in $m=50$ runs of this experiment are given in Table \ref{tab:clfalse}. As observed in this experiment, ST rejects the hypothesis of a two-group SBM in only a small number of runs, failing to detect the vertex degree heterogeneity within groups. The IRG-gkSS test with WL and graphlet kernel detects the departure from the ERMM structure with edge probabilities given by $\mat{p}_0$ in a higher number of runs than ST. For the fit of ERMM with estimated edge probabilities given by $\hat{\mat{p}}$, the IRG-gKSS test with WL kernel fails to reject the null in any of the runs, while using the graphlet kernel, the rejection rate is higher.

\textbf{2.} 
A key advantage of  IRG-gKSSt compared to ST is that it applies to data from IRG models, no matter whether or not they have a block structure. To illustrate this point, we generate graphs on $n=30$  vertices from a Chung-Lu model (see \cite{newman2018networks}), as follows. For each run we choose $w_1, \ldots w_n$ uniform from $[2,8]$ and take as edge probabilities $p_{ij} = w_i w_j / \sum_k w_k,$ truncated at 1. The value 2 is then a lower bound on the smallest expected degree, and the value 8 is an upper bound on the largest expected degree, resulting in graphs which are moderately dense. For the alternative hypothesis, we take $v_i = w_i + 3$ and repeat the above construction; now the lower bound on the smallest expected degree is 5, and the upper bound on the largest expected degree is 11.  We repeat this experiment 50 times and test at the 5\% level. 

ST has a rejection rate of only 8\%, testing the null hypothesis of $K = 1$, and typically not detecting the change in distribution. The IRG-gKSS test, with both kernels, rejects the fit of IRG$(\mat{p}_0)$ in a large number of runs, whereas for the IRG model with estimated edge probabilities, it rejects the fit in only a few runs. Similarly, when actually simulating from the null distribution, so that the null hypothesis is true, ST has a rejection rate of 12\% (not included in the tables).

\begin{table}
    \centering
    \caption{Rejection rate for  the false null model}
    \begin{tabular}{c c c c c c}
    Model & ST & WL & Graphlet & $\widehat{\text{WL}}$  & $\widehat{\text{Graphlet}}$ \\
    \hline \\
    DCSBM$_1$ & 0.08 & 0.42 & 0.32 & 0.00 & 0.26 \\
    DCSBM$_2$ & 0.00  & 0.44 & 0.4 & 0.00 & 0.16 \\
    Chung Lu & 0.08 & 1.00 & 0.80 & 0.00 & 0.04 \\
    \end{tabular}
    \label{tab:clfalse}
\end{table}

\begin{table}
    \centering
    \caption{Rejection rate for the true null model}
    \begin{tabular}{c c c c c}
    Model & WL & Graphlet & $\widehat{\text{WL}}$  & $\widehat{\text{Graphlet}}$ \\
    \hline \\
    DCSBM$_1$ & 0.1 & 0.08 & 0.00  & 0.04 \\
    DCSBM$_2$ & 0.02 & 0.04 & 0.00  & 0.04 \\
    Chung Lu & 0.06 & 0.04 & 0.00  & 0.04 \\
    \end{tabular}
    \label{tab:cltrue}
\end{table}

The values in Table \ref{tab:cltrue} show that the IRG-gKSS test is well calibrated at the $0.05$ level of significance.

Supp.\,Mat.\,Section \ref{ssec:nlpa} also gives results for a non-linear preferential attachment model with dependent edges as an alternative; in these experiments, ST performs comparably to IRG-gKSS. Moreover, Supp.\,Mat.\,Section \ref{experi:1} shows the test performance when the alternative is an ERMM; this is the setting for which a likelihood ratio test applies, which naturally has higher power than IRG-gKSS, but not by a large margin.

\subsection{Synthetic Data Experiments: Planting Anomalies}

Next, we simulate synthetic networks from an IRG$(\mat{p})$ model, plant anomalies, and test the fit of IRG$(\mat{p})$ using the IRG-gKSS test, its estimated version, ST with true group memberships, and the generalised likelihood ratio (GLR) test, to assess and compare the performance of these tests in detecting anomalies in the network. For IRG-gKSS, we use Algorithm \ref{alg:algo1} with $M = 200$. To evaluate the effect of parameter estimation, we also re-estimate the edge probabilities from each simulated network and treat them as the edge probability matrix $\mat{p}$ for the null model and again use Algorithm \ref{alg:algo1} with $M = 200$; we denote the corresponding test by IRG-gKSS\_e.

\subsubsection{Planted Hubs} \label{sec:plantedhubs}

In this experiment, we simulate networks from an ERMM$(\mat{n}, \mat{Q})$ model, with $\mat{n}_{{ub}} = ({8}, {22})$ and $\mat{Q} = \begin{pmatrix}
    {0.29} & 0.01 \\
    0.01 & {0.22}
\end{pmatrix}$. 
We plant hubs (vertices with high degree) in the simulated networks without disturbing the edge density, using Algorithms \ref{alg:hubs} and \ref{alg:hub} given in Supp.\,Mat.\,Section \ref{ssec:planted}, and test the fit of ERMM$(\mat{n}, \mat{Q})$. In this experiment, $k$ is a parameter for the number by which we try to increase the degree of the selected vertex to turn it into a hub, and $R$ is the number of repetitions of Algorithm \ref{alg:hub}, intended to create $R$ hubs in the network. Our construction leaves the overall edge density invariant. Thus, the original increase in the degree and number of hubs created might be smaller than $k$ and $R$ if the number of edges present in the network cannot accommodate these choices. More details are found in Supp.\,Mat.\,Section \ref{sec:add_syn_exp}. The results illustrated in Figure \ref{fig:planted_hubs_main} show that IRG-gKSS can detect some deviation from the model with independent edges even when $R$, the intended number of hubs, is as low as $2$, and $k$, the intended increase in the degree of selected hub vertices, is as low as $3$. ST detects the hubs in a much smaller proportion of runs. The GLR test fails to detect the anomaly. In this experiment, the performance of IRG-gKSS and the estimated version, IRG-gKSS\_e, is very similar.

In Supp.\,Mat.\,Section \ref{ssec:planted}, we also show experimental results for the balanced case of an equal number of vertices in each group;  while the qualitative behaviour is similar, in the balanced case, IRG-gKSS has lower power.

While the number of edges in the simulated network is fixed, planting hubs introduces heterogeneity. In the top line of Figure \ref{fig:planted_hubs_main}, we also report the maximum degree as an indicator of the heterogeneity. The larger the heterogeneity, the easier it is for IRG-gKSS and ST to detect the difference; for the GLR test, increased heterogeneity has no advantage. 

\subsubsection{Planted Cliques}

Next, we simulate $m= 50$ repetitions of ER$({30}, 0.06)$ networks and attempt to plant a clique of size $K$ (a complete subgraph of $K$ vertices) in each of these networks, without changing the edge density, {using Algorithm \ref{alg:clique}} in Supp.\,Mat.\,Section \ref{ssec:clique}; if there are not at least $\binom{K}{2}$  edges, permitting the construction of cliques, then the network is not used. We then test the fit of the ER$({30}, 0.06)$ model on these networks with a planted clique. The results of this experiment are illustrated in Figure \ref{fig:planted_clique_main}. ST begins to signal the presence of more than one community in the network once a clique of size 4 appears, the IRG-gKSS test, its version with estimated edge probabilities IRG-gKSS\_e, and the edge resampling versions starts detecting the lack of fit of the ER$(30,0.06)$ model when $K$ is at least $6$, whereas the GLR test fails to identify any discrepancy. We further observe that the IRG-gKSS test exhibits very similar power when using the estimated edge probabilities compared to the true probabilities, while the edge resampling estimates, denoted by IRG-gKSS\_B, with $B\%$  edge resampling, $B \in \{5, 25\}$ (Algorithm \ref{alg:algo2}), exhibit good power to detect a clique while providing a significant gain in computation time. These results support our theoretical argument that parameter estimation has a limited impact on the distribution of the test statistic and, consequently, on the validity of the test. The computational gain using edge-resampling to estimate IRG-gKSS is illustrated in Supp.\,Mat.\,Section \ref{app:complex}.

The results for more parameter settings are given in Supp.\,Mat.\,Section \ref{ssec:clique}; the qualitative behaviour is similar. 

     \begin{figure}[hbt!]
         \centering
        \includegraphics[width=0.45\textwidth]{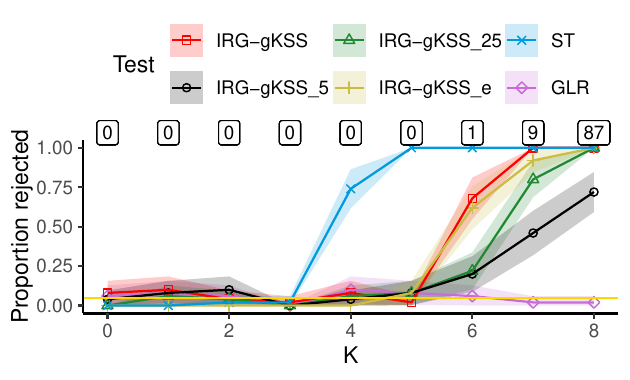} 
        \vspace{.1in}
         \caption{Power of the tests for the fit of an ER$(30,0.06)$ to the network of size 30 with a planted clique of size $K$ using different proportions of edge resampling as well as the no edge resampling version of the IRG-gKSS test statistic. The numbers in the boxes are the total number of networks sampled that had fewer than $\binom{K}{2}$ edges and were, therefore, not used in the experiment.
         }
         \label{fig:planted_clique_main}
     \end{figure}
     
\subsection{Real Data Networks} \label{real_net}

In this section, we use IRG-gKSS to test the fit of IRG models to two often used network data sets, Lazega's lawyers networks from \citet{lazega2001collegial} and Zachary's karate club network from \citet{zachary1977information}. 
Further real-world networks are analysed in Supp.\,Mat.\,Section \ref{sec:applic_contd}. As a proxy for the true edge probabilities, we use the edge probabilities estimated from the observed network; we thus use the IRG-gKSS\_e test.

\subsubsection{Lazega's Lawyers' Networks}

Lazega's lawyers' networks are constructed from questionnaire data collected in a US corporate law firm. This data set is used to create a Work network, an Advice network, and a Friendship network between the 71 attorneys (vertices) of this firm. Various vertex attributes are also part of the dataset, including seniority, formal status, gender, office in which they work, years with the firm, age, practice, and law school they attended. For more details, see \cite{lazega2001collegial}.

We test the fit of different IRG models to the three networks in the data set using IRG-gKSS\_e and ST, with ER, ERMM, and DCSBM as null hypotheses. For the ERMM and the DCSBM, groups are assigned according to formal status, or on office in which they work. The test results for the Friendship network are recorded in Table \ref{table:lazega}; both tests reject the fit of an ER model and the ERMM with `formal status' as group membership. For the ERMM with `office where the lawyers work' as group membership, the IRG-gKSS\_e test does not reject the model, but ST does. Furthermore, IRG-gKSS\_e does not reject the fit of the DCSBM with `formal status' as group membership. We cannot use ST to test the fit of a DCSBM as it is not tailored to this task. The detailed parameter settings and results for ERMMs and DCSBMs with groups constructed using other vertex attributes, and for the other two lawyer networks, are given in Supp.\,Mat.\,Section \ref{ssec:lazega}.

\begin{table}[ht]
    \caption{$P$-values for testing the fit of IRG models on Lazega lawyers' friendship networks} 
    \label{table:lazega}
    \begin{center}
        \begin{tabular}{lcc}
            \textbf{MODEL}  &  \textbf{IRG-gKSS\_e} & \textbf{ST}\\
            \hline \\
            ER & 0.02985 & 0.0000 \\
            ERMM (Status) & 0.02985 & 0.0000 \\
            DCSBM (Status) & 0.42786 & - \\       
            ERMM (Office) & 0.13930 & 0.0000 \\ 
        \end{tabular}
    \end{center}
\end{table}

\subsubsection{Zachary's Karate Club Network} \label{kclub_main}
Zachary’s karate club network from \cite{zachary1977information} is a friendship network of 34 members of a karate club at a US university, which split into two factions as a result of an internal dispute (the group memberships after the split is shown in Figure \ref{fig:karate} in the Supp.\,Mat.). This network is often used as a benchmark network dataset for community detection algorithms. \citet{karrer_newman_2011} fit a stochastic block model (SBM) and a degree-corrected stochastic block model (DCSBM) to this network. For this data set, the test proposed by \citet{karwa2024monte} rejects the fit of SBM with two groups and does not reject the SBM with four groups, at the 5\% level of significance, whereas in \citet{Jin2025}  the fit of an SBM with two groups is not rejected. 

We test the fit of an ER, two ERMMs and a DCSBM using IRG-gKSS\_e, and use ST to test the hypothesis of an ERMM with two groups and an ERMM with four groups in Zachary's karate club network. The $P$-values for the IRG-gKSS\_e and ST tests are presented in Table \ref{table:kclub}. The fit of all these models is rejected by both tests, suggesting dependence between edges. The parameters used for the null models and the results of ST are reported in Supp.\,Mat.\,Section \ref{ssec:karate}.

\begin{table}[ht]
    \caption{{$P$-values for }testing the fit of some IRG models on Zachary's Karate Club network } \label{table:kclub}
    \begin{center}
        \begin{tabular}{lcl}
            \textbf{MODEL}  &  \textbf{IRG-gKSS\_e} & \textbf{ST}\\
            \hline \\
            ER & 0.00995 & 0.0000 \\
            ERMM (2 groups) & 0.00995 & 0.00112 \\
            ERMM (4 groups) &  0.00995 & 0.00989 \\
            DCSBM (two groups) & 0.00995 & - \\ 
        \end{tabular}
    \end{center}
\end{table}

\section{THEORETICAL PROPERTIES}
\label{sec:theory}

An advantage of the KSD approach is that it is possible to provide theoretical guarantees. We take a unit ball of a tensor product RKHS $\HH$ with product kernel $k$ and inner product $\langle \cdot, \cdot \rangle$. For kernels satisfying Assumptions \ref{ass:kernel}, Theorem \ref{normal_approx} addresses the asymptotic distribution of the test statistic $\text{IRG-gKSS}^2(\mat{p},\mat{x})$ in \eqref{t_stat_s_net}. We derive this result for a better understanding of the IRG-gKSS test procedure; we stress that the procedure itself does not rely on asymptotic results.

\medskip
\begin{assumption}\label{ass:kernel}
For the kernel $k: \{0,1\}^N \times \{0,1\}^N \rightarrow \R$ of  an RKHS
$\mathcal{H}$, and a family of kernels $l_s: \{0,1\} \times \{0,1\} \rightarrow \R$, for 
$s \in E$, with an associated RKHS denoted by $\mathcal{H}_s$, 
    \begin{enumerate}[noitemsep]
    \item  $\HH$ is a tensor product RKHS, $\HH = \otimes_{s \in E}\HH_s $;
    \item  $k$ is a product kernel, $k(x, y) = \prod_{s \in E} l_s(x_s, y_s)$; 
    \item  $ \langle l_s (x_s, \cdot), l_s (x_s, \cdot) \rangle_{\HH_s} = l_s(x_s, x_s)=1$;
    \item  $l_s(1, \cdot) - l_s(0, \cdot) \ne 0$ for all $s \in E$;
    \item $l_s(1,0) - l_s(0,0) $ is of the same sign for all $s$.
\end{enumerate}
\end{assumption}
For example, a Gaussian vertex-edge histogram kernel
$k(\mat{x}, \mat{y}) = \exp \left\{ - \frac{1}{2\sigma^2} \sum_{s \in E} (x_s -y_s)^2\right\}$ satisfies Assumption \ref{ass:kernel}, taking  $l_s (x_s, y_s) = e^{- (x_s - y_s)^2/2\sigma^2}$. 

For an IRG$({\mathbf{p}})$ let 
\begin{eqnarray} \label{N_0}
       N_0 = | \{ s: p_s \ne 0, 1\}|. 
   \end{eqnarray}

\begin{theorem} \label{normal_approx}
For fixed $N$ and $\mat{p}$ let $\mat{x}\sim \text{IRG}({\mathbf{p}})$. Let $\mu = \E \, ( \text{IRG-gKSS}^2(\mat{p},\mat{x}))$ and $\sigma^2 = \mathbb{V}ar \left(\text{IRG-gKSS}^2(\mat{p},\mat{x})\right)$. Let $Z$ be a standard normal variable. For $W = \frac{1}{\sigma} \left(\text{IRG-gKSS}^2(\mat{p},\mat{x}) - \mu \right)$, under Assumption \ref{ass:kernel}, the Wasserstein-1 distance between the distributions satisfies
    \begin{equation} \label{bound}
        \|\mathcal{L}(W) - \mathcal{L}(Z)\|_1 \le \frac{C_1}{\sqrt{N}} \left( \frac{N}{N_0}\right)^3,
    \end{equation}
where $C_1=C_1 ({\mathbf{p}},K)$ is an explicit constant which depends on ${\mathbf{p}}$ and the kernel $K$, but not on $N$. 
\end{theorem}

Hence, for reasonably large networks, the test statistic \eqref{t_stat_s_net} behaves like a normal random variable around its mean $\mu$. 

\begin{remark}\label{rem:sssparse}
\begin{enumerate}[noitemsep]
    \item A detailed inspection shows that $C_1$ can be bounded above  by an expression $C_1(n)$ which depends on $n$ and on $\mathbf{p}$ in a smooth way; if there are constants $c, C$ and a function $p(n)$ such that $cp(n) \le p_s (n) \le C p (n) $ and if there are $f$ and $F$ such that $0 < f \le \frac{N}{N_0} \le F <1$, then $C_1 (n) \sim p + \frac{1}{N p^{3/2}}.$ Details are found in Supp.\,Mat.\,Section \ref{app:proofs}. In particular, the result covers some sparse settings as long as $p(n) \gg \frac{1}{n^2}$ and shows some robustness against small variations between the probabilities. As a consequence, when the edge probabilities are estimated, the approximating normal distributions will be similar, and the Wasserstein distance between the two distributions will be small. Details are found in Supp.\,Mat.\,Remark \ref{rem:sparsetoo}.
    
    \item Theorem \ref{normal_approx} can be used for theoretical power guarantees when the alternative distribution is also a model with independent edges. In this case, if the two models differ, so will their approximating normal variables, and the difference between their normals, combined with the approximation error, gives the asymptotic power of the test. Details can be found in Supp.\,Mat.\,Remark \ref{rem:power} 
\end{enumerate}
\end{remark}

Here is a sketch of the proof of Theorem 5.2.
\begin{proof}
    Under Assumptions \ref{ass:kernel}, with $\alpha = (s,s'); s,s' \in E$ and $\mathcal{I} = \{(s,s'): s,s' \in E\}$, we write \eqref{t_stat_s_net} as an average of locally dependent random variables 
    $$X_{\alpha} = \frac{1}{N^2} (p_s  - Y_s)   (p_{s'}  - Y_{s'})  \langle l_{s} (Y_{s}, \cdot) ,   l_{s'} (Y_{s'}, \cdot) \rangle c(s,s').$$
    Here $\mat{Y} =(Y_s, s \in E) \in \{0,1\}^N$ denotes a random vector representing edge indicators in an IRG$(\mat{p})$ graph of size $n$, and $c (s,s') = \langle l_s (1, \cdot ) - l_s( 0 ,  \cdot), l_{s'}(1, \cdot ) - l_{s'}( 0 ,  \cdot) \rangle$, which does not depend on $Y_s$ or $Y_{s'}$. With $\mu_{\alpha} = \EE X_{\alpha}$, the standardised count
    $W = \sum_{\alpha \in \mathcal{I}}(X_{\alpha} - \mu_{\alpha})/{\sigma},$ has zero mean and unit variance. Using Theorem 4.13, p.134, of \citet{chen2011normal} and simplifying gives \eqref{bound}. The detailed proof is in Supp.\,Mat.\,Section \ref{app:proofs}.
\end{proof}

For large networks, the use of IRG-gKSS in \eqref{t_stat_s_net} is not economical and often not possible. Instead, we can estimate the test statistic from a sample of edges from the observed network using the edge re-sampling version of the statistic given in \eqref{t_stat_s_net_resamp}. We  re-write the test statistic in \eqref{t_stat_s_net_resamp} using the count $k_s {= \sum_{b \in [B]} \mathbb{I}(b =s)}$, the number of times the vertex pair $s$ is included in the sample $\mathcal B$, with $|\mathcal B| = B$, as 
\begin{equation} \label{t.st.resamp}
    {\widehat{{\text{IRG-gKSS}}}^2}(\mat{p};\mat{x}) = \frac{1}{B^2} \sum_{s,s' \in E} k_s k_{s'} h_{\mat{x}}(s_b,s_b').
\end{equation}
In this version of statistic, the randomness lies only in the counts $\mat{K} = (k_s, s \in E)$, which are exchangeable and follow a multinomial distribution. Hence, the statistic \eqref{t.st.resamp} is a sum of weakly dependent random variables. Proposition 2 in \citet{xu2021stein} proves that, for any fixed network $\mat{x}$ and a fixed sampling fraction $F = {B}/{N}$, as $N \rightarrow \infty$ the statistic $\widehat{\operatorname{IRG-gKSS}}^2(\mat{p};\mat{x})$ is approximately normal with approximate mean $ \operatorname{IRG-gKSS}^2(\mat{p};\mat{x})$.

\section{CONCLUSION} 
\label{sec:conclusion}
In this paper, we develop IRG-gKSS, a KSD-type pure hypothesis test for IRG models with pre-specified edge probabilities, with a version IRG-gKSS\_e in which the edge probabilities are estimated from the observed network using known vertex features. The IRG-gKSS test does not require more than one observation of the network, and we give theoretical guarantees that do not depend on any asymptotic assumption and that hold for any network size. Our empirical results demonstrate that the IRG-gKSS test detects signals that differ from those of GLRT or ST tests, and thus may provide new insights into potential mechanisms that could have generated the data. 

Here are some avenues for future work. Testing a composite null hypothesis would require a variation of this test that accommodates this generality and is one of our future research directions. Moreover, the performance of IRG-gKSS depends on the choice of the graph kernel. Hence, the graph kernel to be used needs to be chosen carefully. Future research will address this issue and will include IMQ-type kernels as in \cite{kanagawa2023kernel}. Furthermore, IRG-gKSS is computationally expensive for even moderately large networks. Hence, we propose an edge-resampling version of the test statistic, while to some extent, sacrificing the power of the test. Speeding up the code by exploiting more parallelisation and fast matrix product algorithms will also be addressed in future work.

Finally, as often in empirical analysis, the test results are only a first step for a further investigation, which should involve specific domain knowledge. 

\textbf{Acknowledgements and Disclosure of Funding}
We would like to thank the referees for their insightful comments and suggestions. 
AF is supported by the Commonwealth Scholarship Commission, United Kingdom, and in part by an EPSRC grant EP/X002195/1. 
GR is supported in part by the UKRI EPSRC grants EP/T018445/1, EP/R018472/1, EP/X002195/1 and  EP/Y028872/1. For the purpose of Open Access, the authors have applied a CC BY public copyright licence to any Author Accepted Manuscript version arising from this submission.

\bibliography{references.bib}

\begin{thebibliography}{51}
\providecommand{\natexlab}[1]{#1}
\providecommand{\url}[1]{\texttt{#1}}
\expandafter\ifx\csname urlstyle\endcsname\relax
  \providecommand{\doi}[1]{doi: #1}\else
  \providecommand{\doi}{doi: \begingroup \urlstyle{rm}\Url}\fi

\bibitem[{\AA}kerblom et~al.(2023){\AA}kerblom, Hoseini, and Haghir~Chehreghani]{transp2}
Niklas {\AA}kerblom, Fazeleh~Sadat Hoseini, and Morteza Haghir~Chehreghani.
\newblock Online learning of network bottlenecks via minimax paths.
\newblock \emph{Machine Learning}, 112:\penalty0 131--150, 2023.

\bibitem[Arias-Castro et~al.(2018)Arias-Castro, Pelletier, and Saligrama]{arias2018remember}
Ery Arias-Castro, Bruno Pelletier, and Venkatesh Saligrama.
\newblock Remember the curse of dimensionality: the case of goodness-of-fit testing in arbitrary dimension.
\newblock \emph{Journal of Nonparametric Statistics}, 30:\penalty0 448--471, 2018.

\bibitem[Barbour(1988)]{barbour1988stein}
Andrew~D Barbour.
\newblock Stein's method and {P}oisson process convergence.
\newblock \emph{Journal of Applied Probability}, 25\penalty0 (A):\penalty0 175--184, 1988.

\bibitem[Bellamy and Basole(2013)]{manag2}
Marcus~A. Bellamy and Rahul~C. Basole.
\newblock Network analysis of supply chain systems: A systematic review and future research.
\newblock \emph{Systems Engineering}, 16:\penalty0 235--249, 2013.

\bibitem[Berlinet and Thomas-Agnan(2011)]{berlinet2011reproducing}
A.~Berlinet and C.~Thomas-Agnan.
\newblock \emph{Reproducing Kernel Hilbert Spaces in Probability and Statistics}.
\newblock Springer: New York, 2011.

\bibitem[Blondel et~al.(2008)Blondel, Guillaume, Lambiotte, and Lefebvre]{blondel2008fast}
Vincent~D Blondel, Jean-Loup Guillaume, Renaud Lambiotte, and Etienne Lefebvre.
\newblock Fast unfolding of communities in large networks.
\newblock \emph{Journal of Statistical Mechanics: Theory and Experiment}, 2008:\penalty0 P10008, 2008.

\bibitem[Bollobás et~al.(2007)Bollobás, Janson, and Riordan]{Bollobas2007}
Béla Bollobás, Svante Janson, and Oliver Riordan.
\newblock The phase transition in inhomogeneous random graphs.
\newblock \emph{Random Structures \& Algorithms}, 31:\penalty0 3--122, 2007.

\bibitem[Breiger and Pattison(1986)]{BREIGER1986215}
Ronald~L Breiger and Philippa~E Pattison.
\newblock Cumulated social roles: The duality of persons and their algebras.
\newblock \emph{Social Networks}, 8:\penalty0 215--256, 1986.

\bibitem[Caimo and Friel(2011)]{CAIMO201141}
Alberto Caimo and Nial Friel.
\newblock Bayesian inference for exponential random graph models.
\newblock \emph{Social Networks}, 33:\penalty0 41--55, 2011.

\bibitem[Chakrabarty et~al.(2021)Chakrabarty, Hazra, den Hollander, and Sfragara]{chakrabarty2021spectra}
Arijit Chakrabarty, Rajat~Subhra Hazra, Frank den Hollander, and Matteo Sfragara.
\newblock Spectra of adjacency and {L}aplacian matrices of inhomogeneous {E}rd{\H{o}}s--{R}{\'e}nyi random graphs.
\newblock \emph{Random Matrices: Theory and Applications}, 10:\penalty0 2150009, 2021.

\bibitem[Chen et~al.(2011)Chen, Goldstein, and Shao]{chen2011normal}
L.~H.~Y. Chen, L.~Goldstein, and Q.-M. Shao.
\newblock \emph{Normal Approximation by {S}tein's Method}, volume~2.
\newblock Springer: Verlag, Berlin, Heidelberg, 2011.

\bibitem[Chen et~al.(2024)Chen, Lehmann, and Malin]{manag1}
You Chen, Christoph~U Lehmann, and Bradley Malin.
\newblock Digital information ecosystems in modern care coordination and patient care pathways and the challenges and opportunities for {AI} solutions.
\newblock \emph{Journal of Medical Internet Research}, 26:\penalty0 e60258, 2024.

\bibitem[Chwialkowski et~al.(2016)Chwialkowski, Strathmann, and Gretton]{chwialkowski2016kernel}
Kacper Chwialkowski, Heiko Strathmann, and Arthur Gretton.
\newblock A kernel test of goodness of fit.
\newblock In Maria~Florina Balcan and Kilian~Q. Weinberger, editors, \emph{Proceedings of The 33rd International Conference on Machine Learning}, volume~48 of \emph{Proceedings of Machine Learning Research}, pages 2606--2615, New York, New York, USA, 20--22 Jun 2016. PMLR.

\bibitem[Dan and Bhattacharya(2020)]{dan2020goodness}
Soham Dan and Bhaswar~B. Bhattacharya.
\newblock Goodness-of-fit tests for inhomogeneous random graphs.
\newblock In Hal~Daumé III and Aarti Singh, editors, \emph{Proceedings of the 37th International Conference on Machine Learning}, volume 119 of \emph{Proceedings of Machine Learning Research}, pages 2335--2344. PMLR, 13--18 Jul 2020.

\bibitem[Daudin et~al.(2008)Daudin, Picard, and Robin]{daudin2008mixture}
J-J Daudin, Franck Picard, and St{\'e}phane Robin.
\newblock A mixture model for random graphs.
\newblock \emph{Statistics and Computing}, 18:\penalty0 173--183, 2008.

\bibitem[Dinkelberg et~al.(2025)Dinkelberg, Santana, and Bazzan]{social3}
Alejandro Dinkelberg, Breda~Salenave Santana, and Ana L.~C. Bazzan.
\newblock Endorsement networks in the 2022 {B}razilian presidential elections: a case study on {T}witter data.
\newblock \emph{Journal of the Brazilian Computer Society}, 31:\penalty0 219–228, 2025.

\bibitem[Dong et~al.(2020)Dong, Wang, and Liu]{DONG2020}
Zhishan Dong, Shuangshuang Wang, and Qun Liu.
\newblock Spectral based hypothesis testing for community detection in complex networks.
\newblock \emph{Information Sciences}, 512:\penalty0 1360--1371, 2020.

\bibitem[Gorham and Mackey(2015)]{gorham2015measuring}
Jackson Gorham and Lester Mackey.
\newblock Measuring sample quality with {S}tein\textquotesingle s method.
\newblock In C.~Cortes, N.~Lawrence, D.~Lee, M.~Sugiyama, and R.~Garnett, editors, \emph{Advances in Neural Information Processing Systems}, volume~28. Curran Associates, Inc., 2015.

\bibitem[G\"otze(1991)]{gotze1991rate}
F~G\"otze.
\newblock On the rate of convergence in the multivariate {CLT}.
\newblock \emph{The Annals of Probability}, 19:\penalty0 724--739, 1991.

\bibitem[Henao et~al.(2015)Henao, Gan, Lu, and Carin]{bernoulli-poisson}
Ricardo Henao, Zhe Gan, James Lu, and Lawrence Carin.
\newblock Deep {P}oisson factor modeling.
\newblock In \emph{Proceedings of the 29th International Conference on Neural Information Processing Systems - Volume 2}, NIPS'15, page 2800–2808, Cambridge, MA, USA, 2015. MIT Press.

\bibitem[Hu et~al.(2021)Hu, Zhang, Qin, Yan, and Zhu]{Hu2021}
Jianwei Hu, Jingfei Zhang, Hong Qin, Ting Yan, and Ji~Zhu.
\newblock Using maximum entry-wise deviation to test the goodness of fit for stochastic block models.
\newblock \emph{Journal of the American Statistical Association}, 116:\penalty0 1373--1382, 2021.

\bibitem[Jin et~al.(2013)Jin, Yuan, and Liang]{jin2013bayesian}
Ick~Hoon Jin, Ying Yuan, and Faming Liang.
\newblock Bayesian analysis for exponential random graph models using the adaptive exchange sampler.
\newblock \emph{Statistics and Its Interface}, 6:\penalty0 559–576, 2013.

\bibitem[Jin et~al.(2025)Jin, Ke, Tang, and Wang]{Jin2025}
Jiashun Jin, Zheng~Tracy Ke, Jiajun Tang, and Jingming Wang.
\newblock Network goodness-of-fit for the block-model family.
\newblock \emph{Journal of the American Statistical Association}, 0\penalty0 (0):\penalty0 1--14, 2025.

\bibitem[Kanagawa et~al.(2023)Kanagawa, Jitkrittum, Mackey, Fukumizu, and Gretton]{kanagawa2023kernel}
Heishiro Kanagawa, Wittawat Jitkrittum, Lester Mackey, Kenji Fukumizu, and Arthur Gretton.
\newblock A kernel {S}tein test for comparing latent variable models.
\newblock \emph{Journal of the Royal Statistical Society Series B: Statistical Methodology}, 85:\penalty0 986--1011, 2023.

\bibitem[Kapucu and Demiroz(2011)]{social1}
Naim Kapucu and Fatih Demiroz.
\newblock Measuring performance for collaborative public management using network analysis methods and tools.
\newblock \emph{Public Performance \& Management Review}, 34:\penalty0 549--579, 2011.

\bibitem[Karrer and Newman(2011)]{karrer_newman_2011}
Brian Karrer and M.~E.~J. Newman.
\newblock Stochastic blockmodels and community structure in networks.
\newblock \emph{Physical Review E}, 83:\penalty0 016107, 2011.

\bibitem[Karwa et~al.(2024)Karwa, Pati, Petrovi{\'c}, Solus, Alexeev, Rai{\v{c}}, Wilburne, Williams, and Yan]{karwa2024monte}
Vishesh Karwa, Debdeep Pati, Sonja Petrovi{\'c}, Liam Solus, Nikita Alexeev, Mateja Rai{\v{c}}, Dane Wilburne, Robert Williams, and Bowei Yan.
\newblock Monte {C}arlo goodness-of-fit tests for degree corrected and related stochastic blockmodels.
\newblock \emph{Journal of the Royal Statistical Society Series B: Statistical Methodology}, 86:\penalty0 90--121, 2024.

\bibitem[Khademizadeh et~al.(2025)Khademizadeh, Sajid, Khan, Shoukat, and Hussain]{social2}
Shahnaz Khademizadeh, Ahthasham Sajid, Abdul~Qadeer Khan, Fariha Shoukat, and Waqar Hussain.
\newblock Sentiment analysis of users tweets for polarity opinions detection using deep learning for health care service.
\newblock In \emph{AIoT Innovations in Digital Health}, pages 1--24. CRC Press; Boca Raton, 1st edition, 2025.

\bibitem[Lazega(2001)]{lazega2001collegial}
Emmanuel Lazega.
\newblock \emph{The Collegial Phenomenon: The Social Mechanisms of Cooperation Among Peers in a Corporate Law Partnership}.
\newblock Oxford University Press: New York, 2001.

\bibitem[Lee et~al.(2018)Lee, Chong, Liao, and Wang]{manag3}
Cen-Ying Lee, Heap-Yih Chong, Pin-Chao Liao, and Xiangyu Wang.
\newblock Critical review of social network analysis applications in complex project management.
\newblock \emph{Journal of Management in Engineering}, 34:\penalty0 04017061, 2018.

\bibitem[Lei(2016)]{Lei2016}
Jing Lei.
\newblock {A goodness-of-fit test for stochastic block models}.
\newblock \emph{The Annals of Statistics}, 44\penalty0 (1):\penalty0 401 -- 424, 2016.

\bibitem[Liggett(2010)]{liggett2010continuous}
Thomas~M Liggett.
\newblock \emph{Continuous time Markov processes: an introduction}.
\newblock American Mathematical Society: Providence, Rhode Island, 2010.

\bibitem[Liu et~al.(2017)Liu, Li, Pan, Lan, Zheng, Wu, and Wang]{bio1}
Jin Liu, Min Li, Yi~Pan, Wei Lan, Ruiqing Zheng, Fang-Xiang Wu, and Jianxin Wang.
\newblock Complex brain network analysis and its applications to brain disorders: A survey.
\newblock \emph{Complexity}, 2017:\penalty0 8362741, 2017.

\bibitem[Liu et~al.(2016)Liu, Lee, and Jordan]{liu2016kernelized}
Qiang Liu, Jason Lee, and Michael Jordan.
\newblock A kernelized {S}tein discrepancy for goodness-of-fit tests.
\newblock In Maria~Florina Balcan and Kilian~Q. Weinberger, editors, \emph{Proceedings of The 33rd International Conference on Machine Learning}, volume~48 of \emph{Proceedings of Machine Learning Research}, pages 276--284, New York, New York, USA, 20--22 Jun 2016. PMLR.

\bibitem[Liu et~al.(2025)Liu, Song, Musial, Li, Zhao, and Yang]{liu2025stochastic}
Xueyan Liu, Wenzhuo Song, Katarzyna Musial, Yang Li, Xuehua Zhao, and Bo~Yang.
\newblock Stochastic block models for complex network analysis: A survey.
\newblock \emph{ACM Transactions on Knowledge Discovery from Data}, 19:\penalty0 55, 2025.

\bibitem[Lusseau(2003)]{Lusseau_David_2003}
David Lusseau.
\newblock The emergent properties of a dolphin social network.
\newblock \emph{Proceedings of the Royal Society of London. Series B: Biological Sciences}, 270:\penalty0 S186--S188, 2003.

\bibitem[Lusseau et~al.(2003)Lusseau, Schneider, Boisseau, Haase, Slooten, and Dawson]{lusseau2003bottlenose}
David Lusseau, Karsten Schneider, Oliver~J Boisseau, Patti Haase, Elisabeth Slooten, and Steve~M Dawson.
\newblock The bottlenose dolphin community of {D}oubtful {S}ound features a large proportion of long-lasting associations: Can geographic isolation explain this unique trait?
\newblock \emph{Behavioral Ecology and Sociobiology}, 54:\penalty0 396--405, 2003.

\bibitem[Newman(2018)]{newman2018networks}
M.~Newman.
\newblock \emph{Networks}.
\newblock Oxford University Press: New York, 2nd edition, 2018.

\bibitem[Nourdin and Peccati(2012)]{nourdin2012normal}
Ivan Nourdin and Giovanni Peccati.
\newblock \emph{Normal approximations with Malliavin calculus: from Stein's method to universality}, volume 192.
\newblock Cambridge University Press, 2012.

\bibitem[Ouyang et~al.(2023)Ouyang, Dey, and Zhang]{ouyang_dey_zhang_2023}
Guang Ouyang, Dipak~K. Dey, and Panpan Zhang.
\newblock A mixed-membership model for social network clustering.
\newblock \emph{Journal of Data Science}, 21:\penalty0 508--522, 2023.

\bibitem[Padgett(1987)]{padget}
J.~F. Padgett.
\newblock Social mobility in hieratic control systems.
\newblock In R.L. Breiger, editor, \emph{Social Mobility and Social Structure}. Cambridge University Press: New York, 1987.

\bibitem[Reinert and Ross(2019)]{reinert2019approximating}
G.~Reinert and N.~Ross.
\newblock Approximating stationary distributions of fast mixing {G}lauber dynamics, with applications to exponential random graphs.
\newblock \emph{The Annals of Applied Probability}, 29:\penalty0 3201--3229, 2019.

\bibitem[Santos et~al.(2021)Santos, Fujita, and Matias]{Santos2021}
Suzana de~Siqueira Santos, André Fujita, and Catherine Matias.
\newblock Spectral density of random graphs: convergence properties and application in model fitting.
\newblock \emph{Journal of Complex Networks}, 9:\penalty0 cnab041, 2021.

\bibitem[Shekelyan et~al.(2015)Shekelyan, Jossé, and Schubert]{transp1}
Michael Shekelyan, Gregor Jossé, and Matthias Schubert.
\newblock Linear path skylines in multicriteria networks.
\newblock In \emph{2015 IEEE 31st International Conference on Data Engineering}, pages 459--470, 2015.

\bibitem[Stein(1972)]{stein1972bound}
Charles Stein.
\newblock A bound for the error in the normal approximation to the distribution of a sum of dependent random variables.
\newblock In Lucien M.~Le Cam, Jerzy Neyman, and Elizabeth~L. Scott, editors, \emph{Proceedings of the Sixth Berkeley Symposium on Mathematical Statistics and Probability, Volume 2: Probability Theory}, volume~6, pages 583--603. University of California Press, 1972.

\bibitem[Steinwart and Christmann(2008)]{steinwart2008support}
Ingo Steinwart and Andreas Christmann.
\newblock \emph{Support Vector Machines}.
\newblock Springer New York, NY, 2008.

\bibitem[Wang et~al.(2014)Wang, Peng, Peng, and Wu]{bio2}
Jianxin Wang, Xiaoqing Peng, Wei Peng, and Fang-Xiang Wu.
\newblock Dynamic protein interaction network construction and applications.
\newblock \emph{Proteomics}, 14:\penalty0 338--352, 2014.

\bibitem[Wu and Hu(2024)]{WU2024}
Qianyong Wu and Jiang Hu.
\newblock A spectral based goodness-of-fit test for stochastic block models.
\newblock \emph{Statistics \& Probability Letters}, 209:\penalty0 110104, 2024.

\bibitem[Xu and Reinert(2021)]{xu2021stein}
Wenkai Xu and Gesine Reinert.
\newblock A {S}tein goodness-of-test for exponential random graph models.
\newblock In Arindam Banerjee and Kenji Fukumizu, editors, \emph{Proceedings of The 24th International Conference on Artificial Intelligence and Statistics}, volume 130 of \emph{Proceedings of Machine Learning Research}, pages 415--423. PMLR, 13--15 Apr 2021.

\bibitem[Zachary(1977)]{zachary1977information}
W.~W. Zachary.
\newblock An information flow model for conflict and fission in small groups.
\newblock \emph{Journal of Anthropological Research}, 33:\penalty0 452--473, 1977.

\bibitem[Zeitouni et~al.(1992)Zeitouni, Ziv, and Merhav]{zeitouni1992generalized}
Ofer Zeitouni, Jacob Ziv, and Neri Merhav.
\newblock When is the generalized likelihood ratio test optimal?
\newblock \emph{IEEE Transactions on Information Theory}, 38:\penalty0 1597--1602, 1992.

\end{thebibliography}
\bibliographystyle{plainnat}

%%%%%%%%%%%%%%%%%%%%%%%%%%%%%%%%%%%%%%%%%%%%%%%%%%%%%%%%%%%%

\newpage
\quad
\newpage

\section*{Checklist}

\begin{enumerate}

  \item For all models and algorithms presented, check if you include:
  \begin{enumerate}
    \item A clear description of the mathematical setting, assumptions, algorithm, and/or model. [Yes]
    We add clear descriptions of the mathematical setting when they appear in the paper. 
    \item An analysis of the properties and complexity (time, space, sample size) of any algorithm. [Yes] See Section \ref{app:complex}.
    \item (Optional) Anonymized source code, with specification of all dependencies, including external libraries. [Yes] We provide a link in the Implementation Details in Section \ref{sec:applications}.
  \end{enumerate}

  \item For any theoretical claim, check if you include:
  \begin{enumerate}
    \item Statements of the full set of assumptions of all theoretical results. [Yes] For our theoretical results, we give Assumptions \ref{ass:kernel} in Section \ref{sec:theory}.
    \item Complete proofs of all theoretical results. [Yes] The complete proofs are deferred to Section \ref{app:proofs} of the Supplementary Material.
    \item Clear explanations of any assumptions. [Yes] See Section \ref{sec:theory}.     
  \end{enumerate}

  \item For all figures and tables that present empirical results, check if you include:
  \begin{enumerate}
    \item The code, data, and instructions needed to reproduce the main experimental results (either in the supplemental material or as a URL). [Yes] We provide a URL to an anonymised code directory. 
    \item All the training details (e.g., data splits, hyperparameters, how they were chosen). [Not Applicable]
    \item A clear definition of the specific measure or statistics and error bars (e.g., with respect to the random seed after running experiments multiple times). [Not Applicable]
    \item A description of the computing infrastructure used. (e.g., type of GPUs, internal cluster, or cloud provider). [Yes] See second paragraph of Section \ref{sec:applications}. 
  \end{enumerate}

  \item If you are using existing assets (e.g., code, data, models) or curating/releasing new assets, check if you include:
  \begin{enumerate}
    \item Citations of the creator If your work uses existing assets. [Yes] We have referenced the packages used in our experiments and the datasets used in our paper where necessary. 
    \item The license information of the assets, if applicable. [Not Applicable]
    \item New assets either in the supplemental material or as a URL, if applicable. [Yes] We provide the code that contains the implementation of IRG-gKSS with an example. No new data was produced during the research. The data sets used as examples are publicly  available free of charge.
    \item Information about consent from data providers/curators. [Not Applicable]
    \item Discussion of sensible content if applicable, e.g., personally identifiable information or offensive content. [Not Applicable]
  \end{enumerate}

  \item If you used crowdsourcing or conducted research with human subjects, check if you include:
  \begin{enumerate}
    \item The full text of instructions given to participants and screenshots. [Not Applicable] This work does not involve crowdsourcing or any research using human subjects. 
    \item Descriptions of potential participant risks, with links to Institutional Review Board (IRB) approvals if applicable. [Not Applicable]
    \item The estimated hourly wage paid to participants and the total amount spent on participant compensation. [Not Applicable]
  \end{enumerate}

\end{enumerate}

\clearpage
\appendix
\thispagestyle{empty}

% Supplementary material: To improve readability, you must use a single-column format for the supplementary material.
\onecolumn

%\clearpage
%Supp.\,Mat.\,
%\thispagestyle{empty}

% Supplementary material: To improve readability, you must use a single-column format for the supplementary material.
\onecolumn
\aistatstitle{Supplementary Material}

% This command tells etoc to start a local list here

\startcontents[appendices]
\printcontents[appendices]{}{1}{\section*{Table of Contents}}
\vspace{1cm}

\newpage
\section*{LIST OF NOTATIONS}

\begin{table}[ht]
\begin{tabular}{l l}\\ 
{\bf Sets} & \\
$\Omega$ & $\Omega = [0,1]^n$\\
${\mathcal H}_K$ & RKHS with kernel $K$, inner product $\langle \cdot, \cdot \rangle,$ and norm $|| \cdot ||_{\mathcal{H}_K}$\\ 
{\bf Graph notations} & \\
$\mathcal{G} = (\mathcal{V},\mathcal{E})$ & graph with vertex set $\mathcal{V}$  and edge set $\mathcal{E}$\\
$g_u$ & (given) group assignment for vertex $u$ \\ 
$E$ &  $\{ 1, 2 \ldots, N\}$; in a network on $n$ vertices with $N = {n \choose 2}$ vertex pairs, \\
& we use as {a} slight abuse of notation 
$E = \{ (u,v): 1 \le u < v \le n\}$
\\ $\mat{x}$ & $\mat{x}= (x_{u,v}, 1 \le u < v \le n)$, representing the edge indicators of a graph; \\
& as a slight abuse of notation, $\mat{x}$ is also used for the adjacency matrix of a graph\\
$\mat{x}^{(s,1)}$ & for a vertex pair $s$: the collection with $1$ at the $s-$ coordinate and the same as $\mat{x}$\\
& otherwise
\\ 
    $\mat{x}^{(s,0)}$ & the collection with $0$ at the $s-$ coordinate and the same as $\mat{x}$ otherwise
  \\ 
    $\mat{x}_{-s}$& the collection $\{x_{r,t}, 1 \le r < t \le n, (r,t) \neq (u,v) \}$, without the $s=(u,v)$- coordinate\\
    {\bf Operators} & \\
 $\Delta_s f(\mat{x})$& $   \Delta_s f(\mat{x}) =  f(\mat{x}^{(s,1)}) - f(\mat{x}^{(s,0)})$\\ 
 $|| \Delta || $ & $|| \Delta f || = \max_{s \in E}  | \Delta_s f(\mat{x})| $ \\
 $\EE$ & expectation \\ 
 $\sum_{i \le j}^{L}$ & $\sum_{1 \le i \le  j \le L}$ \\ 
{\bf Network models}& \\
DCSBM & degree-corrected stochastic blockmodel  (here: the edge indicators depend on the \\
& known group membership as well as a vertex-specific parameter) \\ 
ER & Erd\"os-R\'enyi random graph (with independent and identically distributed edges) \\ 
ERMM & Erd\"os-R\'enyi mixture model (a SBM with given group assignments) \\ 
IRG & inhomogeneous random graph model (edge indicators are independent) \\
NLPA & a non-linear preferential attachment model; the attachment probability is  \\
&proportional to a power of the degree \\ 
SBM & stochastic blockmodel  (edge indicators are independent\\&  and depend only on the group memberships, which are generally unknown) \\
$N_0$ & $ N_0 = | \{ s: p_s \ne 0, 1\}|$\\ 
$p_s$ & probability of an edge at vertex pair $s$ \\ 
${\mathbf p}$ & the matrix of edge probabilities of all vertex pairs \\ 
${\mathbf Q}$ & the $L \times L$ matrix of probabilities of edges between different groups, \\
& in an ERMM with $L$ groups 
\\
{\bf Stein's method} & \\
${\mathcal A}$ & Stein operator \\
${\mathcal{F}} ({\mathcal A}) $
& Stein class (domain) of the Stein operator ${\mathcal A}$ \\
$ S(p,q; \mathcal{H})$ & Stein discrepancy; 
$ S(p,q; \mathcal{H}) = \sup_{f \in \mathcal{H}} |  {\E}\AAA_p f(W)|$\\
{\bf Tests } & \\
GLR & generalised likelihood-ratio test\\ 
IRG-gKSS & graph kernel Stein statistic for IRG models \\
ST & the bootstrap corrected version of the spectral test by \cite{Lei2016}
\end{tabular}
\end{table}

In this supplementary material, for convenience, equation labels (1)--(11) refer to the main paper, and equation labels (12)--(61) refer to this supplementary material. 

\newpage

\section{PROOFS OF RESULTS STATED IN THE MAIN TEXT} \label{app:proofs}

\textbf{Proof of Theorem \ref{normal_approx}}

Before giving the proof we give a refined statement of the result in Theorem \ref{normal_approx}. To this purpose 
let 
   \begin{equation} \label{eq:d}
       d_{\mathrm{min}}:= \min_{s \in E} | l_s(1,0) - 1| 
   \end{equation}
   and
   \begin{equation} \label{eq:p}
       p_{\mathrm{min}}:= \min_{s \in E} \{p_s(1-p_s): 0 < p_s < 1\}; 
   \end{equation}
   similarly, 
   \begin{equation} \label{eq:dmax}
       d_{\mathrm{max}}:= \max_{s \in E} | l_s(1,0) - 1| 
   \end{equation}
   and
   \begin{equation} \label{eq:pmax}
       p_{\mathrm{max}}:= \max_{s \in E} \{p_s(1 - p_s) \}.
   \end{equation}

The more detailed formulation of Theorem \ref{normal_approx} is as follows. Let $\mat{x}\sim IRG({\mathbf{p}})$, and set $\mu = \E \, ( \text{IRG-gKSS}^2(\mat{p},\mat{x})) $ and $\sigma^2 = \mathbb{V}ar \left(\text{IRG-gKSS}^2(\mat{p},\mat{x})\right)$. For $W = \frac{1}{\sigma} \left(\text{IRG-gKSS}^2(\mat{p},\mat{x}) - \mu \right)$, and a standard normal random variable $Z$, under Assumption \ref{ass:kernel} we have 
    \begin{equation} \label{eq:overallbound}
        \|\mathcal{L}(W) - \mathcal{L}(Z)\|_1 \le \frac{C_1}{\sqrt{N}},
    \end{equation}
    where 
    \begin{eqnarray}
        \label{eq:c1}
        C_1 \le  C_1(N, d_0, p_{\rm{max}}, d_{\rm{max}} )= 16 \frac{\sqrt{2} d_{\mathrm{max}}^4 p_{\mathrm{max}}^4}{N \sigma^2\sqrt{\pi} }   + \frac{8}{\sqrt{N}}  \left(\frac{2 N_0}{N}\right)^2  \frac{d_{\mathrm{max}}^6  p_{\mathrm{max}}^2}{(N \sigma^2)^{3/2}}
         \sqrt{C(N,d_0, p_{\rm{max}}, d_{\rm{max}} )}
    \end{eqnarray} 
    and 
    \begin{equation}\label{eq:cn}
   C(N,d_0, p_{\rm{max}}, d_{\rm{max}} )
   =  p_{\rm{max}}^4 + \left( \frac{N+2}{N^2}p_{\rm{max}}^2 + \frac{d_0 p_{\rm{max}}}{N^3 d_{\rm{max}}^4} \right).  
   \end{equation}
   Moreover, 
\begin{equation}\label{eq:sigmalower}
       \sigma^2 \ge \frac{32 N_0^3}{N^4} d_{\mathrm{min}}^4 p_{\mathrm{min}}^3 \ge 32 d_{\mathrm{min}}^4 p_{\mathrm{min}}^3 \frac{1}{N} \left( \frac{N_0}{N}\right)^3. 
   \end{equation}

Here, the Wasserstein-1 distance between two probability distributions  $P$ and $Q$ on $\mathbb{R}$ is given by 
$$ || P - Q||_1 = \sup_{h \in Lip(1) } | Ph - Qh|,$$
where $Lip(1)$ is the set of all Lipschitz functions with Lipschitz constant 1, and $Ph = \E  h(Y)$ with $Y \sim P.$

\begin{remark}\label{rem:sparsetoo}
   The refined statement of Theorem \ref{normal_approx} shows that, in particular, when the $p_s$ do not depend on $n$ and $N_0$ is of the same order as $N$ then the bound \eqref{eq:sigmalower} gives that $(N \sigma^2)^{-1}$ is of order 1. Moreover, in this regime, $ C(N,d_0, p_{\rm{max}}, d_{\rm{max}} )$ is of order 1, and  $C_1$ is of order 1, and hence the overall bound is of order $N^{-1/2}$, that is, of order $n^{-1}$. 

However, it is not necessary to assume that the $p_s$ do not depend on ${n}$. For example, if there are constants $c, C$ and $p=p(n)$ such that $ c p(n) \le p_{\mathrm{min}} (n) \le C p(n)$ for all $s$, and if $p(n) \rightarrow 0$ as $n \rightarrow \infty$, and if $N_0$ is of the same order as $N$, then $(N \sigma^2)^{-1}$ is bounded above by a constant times $p_{\min}^{-3},$ and so the first term in $C_1$ is bounded by a constant times $p_{\max}^4/ p_{\min}^3$.
The term $C(N,d_0, p_{\rm{max}}, d_{\rm{max}} )$ in \eqref{eq:cn} is bounded by a constant times 
$$  \max\{ p_{\max}^4, p_{\max}^2/N, p_{\max}/N^3\} \le  p_{\max} $$
and hence $C_1$ is bounded by a constant times 
$p_{\max}^4/ p_{\min}^3 
+ \frac{1}{n} p_{\max}^{\frac52} p_{\min}^{-\frac92},$
giving an overall bound in \eqref{eq:overallbound} of a constant times 
$$\frac{1}{n} \left\{p_{\max}^4/ p_{\min}^3 
+ \frac{1}{n} \left(\frac{p_{\max}}{p_{\min}} \right)^\frac52 \frac{1}{p_{\min}^{2}}   \right\}$$
which tends to $\infty$ if $p(n) \gg n^{-2}$. In this case, the bound \eqref{eq:overallbound} still tends to 0 as $N \rightarrow \infty$. This bound thus covers any parameter regime in which the expected number of edges grows with the number of vertices, no matter how slowly. Thus, the bound also covers a sparse regime in the sense that the expected number of edges can be of smaller order than $n^2$. 
\end{remark} 
\begin{proof}
For $s \in E$ in an $\IRG(\mat{p})$, from \eqref{SE_ermm_op}
\begin{eqnarray*}
 \AAA_{\IRG} f(\mat{x}) = \frac{1}{N} \sum_{s \in E} \left[ p_s \Delta_s f(\mat{x}) + \left( f(\mat{x}^{(s,0)}) - f(\mat{x}) \right) \right] = \frac{1}{N} \sum_{s \in E} \mathcal{A}^{(s)} f(\mat{x}),
\end{eqnarray*} 
with
$$\AAA_{\IRG}^{(s)} f(x)  =
p_s \left(f(x^{(s,1)}) - f(x) \right) + (1-p_s)  \left(f(x^{(s,0)}) - f(x) \right).
$$
Thus, using \eqref{t_stat_s_net},
\begin{eqnarray*}
\lefteqn{\text{IRG-gKSS}^2(\mat{p}, x) = 
\frac{1}{N^2} \sum_{s,s' \in E} \left\langle  \AAA_{\IRG}^{(s)}  K(x, \cdot) ,   \AAA_{\IRG}^{(s')} K(x, \cdot) \right\rangle} \\
&=& \frac{1}{N^2} \sum_{s,s' \in E} \Big\langle p_s \left(  K(x^{(s,1)}, \cdot) -  K(x, \cdot) \right) + (1-p_s) \left(K(x^{(s,0)}, \cdot) - K(x, \cdot) \right) \bm{,} \\
&& \quad \quad \quad \quad \quad  \quad \quad  p_{s'} \left(  K(x^{(s',1)}, \cdot) - K(x, \cdot) \right) + (1-p_{s'}) \left( K(x^{(s',0)}, \cdot) - K(x, \cdot) \right)\Big\rangle\\
&=&  \frac{1}{N^2} \sum_{s,s' \in E} h_{\mat{x}}(s,s')
\end{eqnarray*}
with $h_{\mat{x}} (s,s') = \left\langle \AAA_{\IRG}^{(s)}K(\mat{x},\cdot) , \AAA_{\IRG}^{(s')} K(\cdot,\mat{x})\right\rangle $ as in \eqref{t_stat_s_net}. 
Under Assumption \ref{ass:kernel}, we have
\begin{eqnarray*}
 K(x^{(s,1)}, \cdot) -  K(x, \cdot) &=& 
 (l_s(1 ,  \cdot) - l_s( x_s,   \cdot) ) \prod_{t \ne s} l_t (x_t,  ,  \cdot)\\
 &=& (1 - x_s)  (l_s(1, \cdot ) - l_s( 0 ,  \cdot) ) l_{s'} (x_{s'}, \cdot ) \prod_{t \ne s, s'} l_t (x_t,    \cdot),
\end{eqnarray*}
and 
\begin{eqnarray*}
 K(x^{(s,0)}, \cdot) -  K(x, \cdot) &=& 
 - x_s  (l_s(1, \cdot ) - l_s( 0 ,  \cdot) ) l_{s'} (x_{s'}, \cdot) \prod_{t \ne s, s'} l_t (x_t,    \cdot).
\end{eqnarray*}
Hence, 
\begin{eqnarray*}
   h_{\mat{x}} (s,s')
   &=& ( p_s (1-x_s) - (1-p_s) x_s)
   ( p_{s'} (1-x_{s'}) - (1-p_{s'}) x_{s'})\\
   && 
   \langle(l_s(1, \cdot ) - l_s( 0 ,  \cdot) ) l_{s'} (x_{s'}, \cdot) \prod_{t \ne s, s'} l_t (x_t,    \cdot) , 
   (l_{s'}(1, \cdot ) - l_{s'}( 0 ,  \cdot) ) l_{s} (x_{s}, \cdot) \prod_{t \ne s, s'} l_t (x_t,    \cdot)
   \rangle .
\end{eqnarray*}
Assumption \ref{ass:kernel} gives $\langle \prod_{t \ne s, s'} l_t (x_t,    \cdot) , \prod_{t \ne s, s'} l_t (x_t,    \cdot) \rangle = \prod_{t \ne s, s'} \langle l_t (x_t,    \cdot) ,  l_t (x_t,    \cdot) \rangle= 1$. Hence, 
\begin{eqnarray}
   h_{\mat{x}} (s,s')
   &=& ( p_s (1-x_s) - (1-p_s) x_s)
   ( p_{s'} (1-x_{s'}) - (1-p_{s'}) x_{s'}) \nonumber \\
   && 
   \langle(l_s(1, \cdot ) - l_s( 0 ,  \cdot) ) l_{s'} (x_{s'}, \cdot)  , 
   (l_{s'}(1, \cdot ) - l_{s'}( 0 ,  \cdot) ) l_{s} (x_{s}, \cdot) 
   \rangle . \label{eq:hss}
    \end{eqnarray}
For $s \ne s'$, we  obtain 
\begin{eqnarray*}
  \lefteqn{ h_{\mat{x}} (s,s')}\\
   &=& ( p_s (1-x_s) - (1-p_s) x_s)
   ( p_{s'} (1-x_{s'}) - (1-p_{s'}) x_{s'})\\
   && 
   \langle(l_s(1, \cdot ) - l_s( 0 ,  \cdot) ) l_{s'} (x_{s'}, \cdot) , 
   (l_{s'}(1, \cdot ) - l_{s'}( 0 ,  \cdot) ) l_{s} (x_{s}, \cdot) 
   \rangle \\
   &=& ( p_s (1-x_s) - (1-p_s) x_s)
   ( p_{s'} (1-x_{s'}) - (1-p_{s'}) x_{s'}) \\
   && \langle (l_s(1, \cdot ) - l_s( 0 ,  \cdot) ), l_{s} (x_{s}, \cdot) 
   \rangle \langle (l_{s'}(1, \cdot ) - l_{s'}( 0 ,  \cdot) ),  l_{s'} (x_{s'}, \cdot) \rangle\\
   &=& ( p_s (1-x_s) - (1-p_s) x_s)
   ( p_{s'} (1-x_{s'}) - (1-p_{s'}) x_{s'}) (l_s(1, x_s) - l_s(0, x_s)) (l_{s'}(1, x_{s'}) - l_{s'}(0, x_{s'})) \\
   &=& f(x_s) f(x_{s'})  
\end{eqnarray*} 
with 
\begin{equation}\label{eq:f}
    f(x_s) = ( p_s (1-x_s) - (1-p_s) x_s)(l_s(1, x_s) - l_s(0, x_s)).
\end{equation}
For $s=s'$,  with \eqref{eq:hss}
\begin{eqnarray}
   h_{\mat{x}} (s,s')
   &=& ( p_s (1-x_s) - (1-p_s) x_s)^2
  || l_s(1, \cdot ) - l_s( 0 ,  \cdot) ) l_{s} (x_{s}, \cdot) ||_{{\mathcal{H}}_s}^2 \nonumber \\
  &=& x_s ( p_s (1-x_s) - (1-p_s) x_s)^2
  || l_s(1, \cdot ) - l_s( 0 ,  \cdot) ) l_{s} (x_{s}, \cdot) ||_{{\mathcal{H}}_s}^2 \nonumber \\
  && + (1-x_s)  ( p_s (1-x_s) - (1-p_s) x_s)^2
  || l_s(1, \cdot ) - l_s( 0 ,  \cdot) ) l_{s} (x_{s}, \cdot) ||_{{\mathcal{H}}_s}^2\nonumber \\&=& - x_s (1-p_s)^2
  || l_s(1, \cdot ) - l_s( 0 ,  \cdot) ) l_{s} (1, \cdot) ||_{{\mathcal{H}}_s}^2 \nonumber \\
  && + (1-x_s)   p_s^2
  || l_s(1, \cdot ) - l_s( 0 ,  \cdot) ) l_{s} (0, \cdot) ||_{{\mathcal{H}}_s}^2 \nonumber \\
  &=& (1-x_s) p_s^2 g_s(0) - x_s (1 - p_s)^2 g_s(1)  \label{eq:equals}\\
  &=:& g(x_s)  \nonumber 
\end{eqnarray}
with 
\begin{equation} \label{eq:norms}
g_s(1) = || (l_s(1, \cdot ) - l_s( 0 ,  \cdot) ) l_{s} (1, \cdot) ||_{{\mathcal{H}}_s}^2; \quad \quad g_s(0) = || (l_s(1, \cdot ) - l_s( 0 ,  \cdot) ) l_{s} (0, \cdot) ||_{{\mathcal{H}}_s}^2.
\end{equation}

Hence, with $\alpha = (s,s'); s,s' \in E$, and $\mathcal{I} = \{(s,s'): s,s' \in E\}$ so that $|\mathcal{I}| = N^2$, using \eqref{t_stat_s_net}, we have 
$$\text{IRG-gKSS}^2(\mat{p}, \mat{x}) = \frac{1}{N^2} \left( \sum_{\alpha = (s,s ) \in \mathcal{I}} g(X_s) + \sum_{\alpha=(s\ne s') \in \mathcal{I}} f(X_s) f(X_{s'}) \right) . $$
For $\alpha = (s,s')$ with $s \ne s'$ we abbreviate 
$Y_\alpha =\frac{f(X_s) f(X_{s'})}{N^2}$; for $\alpha = (s,s')$ we abbreviate $Y_\alpha =\frac{g(X_s) }{N^2}$, and we set $$ W = \sum_{\alpha=(s,s') \in \mathcal{I}}\frac{Y_\alpha - \mu_\alpha }{\sigma} $$
with $\mu_\alpha = \E Y_\alpha$ and $\sigma^2 = \mathrm{Var}(\text{IRG-gKSS}^2(\mat{p}, \mat{x})).$ 

This representation writes IRG-gKSS$^2$ as an average of locally dependent random variables; to clarify the dependence, for any $\alpha = (s,s'), \beta = (t,t')\in \mathcal{I}$, $f(X_s) Y_\alpha = f(X_{s'})$ and $Y_\beta f(X_t) f(X_{t'})$ are independent unless $\alpha$ and $\beta$ share at least one element. We note that $W$ has zero mean and unit variance. Next, recalling that  $\|\cdot\|_1$ denotes the Wasserstein-1 distance, we apply Theorem 4.13, p.134, from \citet{chen2011normal} to get a bound on between $\mathcal{L}(W)$ and $\mathcal{L}(Z)$, the distribution of a standard normal random variable $Z$, respectively, obtaining 
\begin{eqnarray} \label{chen_theo}
    \| {\mathcal L}(W) - {\mathcal L}(Z)\|_1 \le \sqrt{\frac{2}{\pi}} \E \left| \sum_{\alpha \in  {\mathcal I}}
    ( \xi_\alpha \eta_\alpha - \E ( \xi_\alpha \eta_\alpha)) 
    \right| + \sum_{\alpha \in  {\mathcal I}} \E | \xi_\alpha \eta_\alpha^2| \nonumber \\
    \le \sqrt{\frac{2}{\pi}}  \sqrt{ \mathbb{V}ar ( \sum_{\alpha \in  {\mathcal I} } \xi_\alpha \eta_\alpha ) } + \sum_{\alpha \in  {\mathcal I}} \E | \xi_\alpha \eta_\alpha^2|,
\end{eqnarray} 
with $\xi_\alpha = ( Y_\alpha - \mu_{\alpha} ) / {\sigma} $, and $\eta_\alpha = \sum_{\beta \in A_\alpha} \xi_\beta$ where $A_{\alpha} = \{ \beta = (t,t') \in  {\mathcal I} : \alpha \cap \beta \ne \emptyset \} \subset \mathcal{I}$; we note that  $|A_{\alpha}| = 2N -1$. 

To bound the remainder terms, we need some moments. First we re-arrange \eqref{eq:f} using that $x_s^2 = x_s, (1 - x_s)^2 = 1-x_s,$ and $x_s( 1-x_s) = 0$ and the symmetry of $l_s$ to obtain 
\begin{eqnarray}
f(x_s) &=& x_s \{  ( p_s (1-x_s) - (1-p_s) x_s)(l_s(1, x_s) - l_s(0, x_s))\} \nonumber \\
&& + (1-x_s) \{  ( p_s (1-x_s) - (1-p_s) x_s)(l_s(1, x_s) - l_s(0, x_s))\} \nonumber \\
&=& - x_s (1-p_s) (l_s(1, x_s) - l_s(0, x_s)) + (1-x_s) p_s (l_s(1, x_s) - l_s(0, x_s)) \nonumber \\
&=& - x_s (1-p_s) (l_s(1, 1) - l_s(0, 1)) + (1-x_s) p_s (l_s(1, 0) - l_s(0, 0))
\nonumber \\
&=& - x_s (1-p_s) (1 - l_s(1, 0)) + (1-x_s) p_s (l_s(1, 0) - 1) \nonumber 
\\
&=& (l_s(1, 0) - 1)\{ p_s (1-x_s) + x_s (1-p_s) \}  \nonumber \\
& =&  (l_s(1, 0) - 1)
\{ p_s + x_s ( 1 - 2 p_s) \}. \label{eq:fexpression}
\end{eqnarray} From this expression we deduce the moments 
\begin{eqnarray}\label{eq:moments}
\E f(X_s) &=&  (l_s(1, 0) - 1)
\{ p_s + p_s ( 1 - 2 p_s) \} \nonumber \\
&=& 2 p_s(1-p_s) (l_s(1, 0) - 1) \label{eq:meanf}\\
\mathrm{Var} f(X_s) &=& (l_s(1, 0) - 1)^2 ( 1 - 2 p_s)^2 p_s (1-p_s) 
\label{eq:varf}\\
\E (f(X_s)^2) &=& (l_s(1, 0) - 1)^2 ( 1 - 2 p_s)^2 p_s (1-p_s) + 4 p_s^2(1-p_s)^2(l_s(1, 0) - 1)^2 \nonumber \\
&=& (l_s(1, 0) - 1)^2 p_s (1-p_s)
\{ 1 - 4 p_s + 4 p_s^2 + 4 p_s - 4 p_s^2 \} \nonumber \\
&=& (l_s(1, 0) - 1)^2 p_s (1-p_s) \label{eq:meansquf}\\
\E ( f(X_s)^3 ) &=& (l_s(1, 0) - 1)^3\{ p_s (p_s + 1 - 2 p_s)^3 + (1-p_s) p_s^3\} \nonumber \\
&=& (l_s(1, 0) - 1)^3 p_s (1-p_s) \{(1 - p_s)^2 +  p_s^2\}  \label{eq:3rdmomf}.
\end{eqnarray}
We now distinguish two cases; $s\ne s'$ and $s = s'$.

{\bf The case $s\ne s'$.}
For $s \ne s'$, $f(X_s)$ and $f(X_{s'})$ are independent in the IRG model. We therefore obtain 
\begin{eqnarray}
    \E h_\mat{X} (s,s') &=& 
    4 p_s (1-p_s) p_{s'} (1-p_{s'})(l_s(1, 0) - 1) (l_{s'}(1, 0) - 1) \label{eq:meanh}\\
    \E \{ h_\mat{X} (s,s')^2 \}  
    &=& (l_s(1, 0) - 1)^2 (l_{s'}(1, 0) - 1)^2 p_s (1-p_s) p_{s'} (1-p_{s'}) 
    \label{eq:meanqsquare}\\
    \mathrm{Var} (h_\mat{X} (s,s') ) &=& 
  (l_s(1, 0) - 1)^2 (l_{s'}(1, 0) - 1)^2 p_s (1-p_s) p_{s'} (1-p_{s'}) \nonumber \\
  && \{1 - 4 p_s (1-p_s) p_{s'} (1-p_{s'}) \}  \label{eq:var}  .
\end{eqnarray}
For the covariance  $\mathrm{Cov}(h_\mat{X} (s,s') , h_\mat{X} (t,t') $ we first notice that it vanishes when $\{ t,t'\} \cap \{s,s'\} =0$. When $s, s', t$ are distinct then by independence 
\begin{eqnarray} \label{eq:covdistinct}
 \lefteqn{  \mathrm{Cov}(h_\mat{X} (s,s') , h_\mat{X} (s,t)}  \\
   &=& \E f(X_s)^2 \E f(X_{s'}) \E f(X_t) - (\E f(X_s))^2 \E f(X_{s'}) \E f(X_t) \nonumber \\
   &=& \mathrm{Var} ( f(X_s)) \E f(X_{s'}) \E f(X_t)  \\
   &=&  4 (l_s(1, 0) - 1)^2 (l_{s'}(1, 0) - 1)  (l_{t}(1, 0) - 1) ( 1 - 2 p_s)^2 p_s (1-p_s)  p_{s'}(1-p_{s'}) 
   p_{t}(1-p_{t}) \nonumber
\end{eqnarray}
where we used \eqref{eq:varf} and \eqref{eq:meanf} in the last step.

{\bf{The case $s=s'$.}} If $s=s'$ then from \eqref{eq:equals} we obtain the moments 
\begin{eqnarray}
    \E h_\mat{X}(s,s) &=& p_s^2 g_s(0) - p_s ( (1-p_s)^2 g_s(1) + p_s^2 g_s(0 )   \label{eq:meanss}\\
    \mathrm{Var} (  h_\mat{X}(s,s)) &=& 
    p_s(1 - p_s) ( (1-p_s)^2 g_s(1) + p_s^2 g_s(0 ) )^2 \label{eq:varss}.
\end{eqnarray}Moreover, for $s, t$ distinct, 
\begin{eqnarray}
   \lefteqn{  \mathrm{Cov} (h_\mat{X}(s,s), h_\mat{X}(s,t))} \nonumber\\  
    &=& \E f(X_s)^3 \E f(X_t) - 
    \E f(X_s)^2 \E f(X_s) \E f(X_t) \nonumber \\
    &=& 2(l_s(1, 0) - 1)^3 p_s (1-p_s) \{(1 - p_s)^2 +  p_s^2\}  (l_t(1, 0) - 1)  p_t (1-p_t) \nonumber \\
    && - 4(l_s(1, 0) - 1)^2 p_s (1-p_s)(l_s(1, 0) - 1)^2 ( 1 - 2 p_s)^2 p_s (1-p_s) (l_t(1, 0) - 1)  p_t (1-p_t) \nonumber \\
    &=& 2 (l_s(1, 0) - 1)^3(l_t(1, 0) - 1) p_s (1-p_s) p_t (1-p_t) \nonumber \\
    && \{(1- p_s)^2 + p_s^2 - 2 (1-2p_s)^2 p_s (1-p_s)  \} 
\end{eqnarray} where we used \eqref{eq:3rdmomf}, \eqref{eq:meanf} and \eqref{eq:meansquf}.
This gives 
\begin{eqnarray*}
   \lefteqn{ N^4  \sigma^2 }\\
   &=&  \sum_\alpha {\mathrm{Var}} (h_\mat{X} (\alpha)) + \sum_\alpha \sum_{\beta: \beta \ne \alpha} \mathrm{Cov } (h_\mat{X} (\alpha),h_\mat{X} (\beta))   \\
     &=& \sum_s   p_s(1 - p_s) \{  (1-p_s)^2 g_s(1) + p_s^2 g_s(0 ) )^2\} 
     \\
     &&+ \sum_s \sum_{s': s'  \ne s}  (l_s(1, 0) - 1)^2 (l_{s'}(1, 0) - 1)^2 p_s (1-p_s) p_{s'} (1-p_{s'})  \{1 - 4 p_s (1-p_s) p_{s'} (1-p_{s'}) \} \\
     &&+ 2 \sum_s \sum_{t: t \ne s}  (l_s(1, 0) - 1)^3(l_t(1, 0) - 1) p_s (1-p_s) p_t (1-p_t) \{(1- p_s)^2 + p_s^2 - 2 (1-2p_s)^2 p_s (1-p_s)  \} \\
     && + 16 \sum_s \sum_{s': s' \ne s} \sum_{t: t \ne s, s'}   (l_s(1, 0) - 1)^2 (l_{s'}(1, 0) - 1)  (l_{t}(1, 0) - 1) ( 1 - 2 p_s)^2 p_s (1-p_s)  p_{s'}(1-p_{s'}) 
   p_{t}(1-p_{t}).\end{eqnarray*}
By items (iv) and (v) in Assumption \ref{ass:kernel}, recalling that $l_s(0,0) =1$, all non-zero summands in this expression are positive. Recall that  
   \begin{eqnarray*}
       N_0 = | \{ s: p_s \ne 0, 1\}|. 
   \end{eqnarray*}
The last sum has the largest number of summands, namely $N_0^2(2N_0-1).$ Hence with $d_{\mathrm{min}}$ as in \eqref{eq:d} and $p_{\mathrm{min}}$ as in \eqref{eq:p}, we have, as stated in, \eqref{eq:sigmalower}
\begin{equation*}
       \sigma^2 \ge \frac{32 N_0^3}{N^4} d_{\mathrm{min}}^4 p_{\mathrm{min}}^3 \ge 32  d_{\mathrm{min}}^4 p_{\mathrm{min}}^3 \frac{1}{N} \left( \frac{N_0}{N}\right)^3. 
   \end{equation*}

For bounding $\sum_{\alpha \in  {\mathcal I}} \EE | \xi_\alpha \eta_\alpha^2|$ we use the Cauchy-Schwarz inequality to obtain 
\begin{eqnarray*}
   {\sum_{\alpha \in  {\mathcal I}} \EE | \xi_\alpha \eta_\alpha^2|} 
   &\le & \left( \EE ( \xi_\alpha^2 \eta_\alpha^4 \right)^{\frac{1}{2}}. 
\end{eqnarray*}
   With  $d_{\mathrm{max}}$  as in \eqref{eq:dmax}, with $p_{\mathrm{max}}$ as in \eqref{eq:pmax}, and  with \eqref{eq:fexpression}, 
   we obtain \begin{eqnarray}\label{eq:xibound}
   | \xi_\alpha | \le  2 \frac{d_{\mathrm{max}}^2  p_{\mathrm{max}}^2}{N^2 \sigma}
   \end{eqnarray}
   and \begin{eqnarray*}\label{eq:etabound}
   | \eta_\alpha | \le  2 d_{\mathrm{max}}^2  p_{\mathrm{max}}^2\frac{2N_0}{N^2 \sigma}. 
   \end{eqnarray*}
   Hence 
\begin{eqnarray}\label{eq:etaboundinter}
      \EE ( \xi_\alpha^2 \eta_\alpha^4 ) \le 16 \left(  d_{\mathrm{max}}^2  p_{\mathrm{max}}^2\frac{2N_0}{N^2 \sigma}\right)^4  \EE ( \eta_\alpha^2 ) .
   \end{eqnarray}
For $ \EE ( \eta_\alpha^2 )$ we carry out a similar calculation as for $\sigma^2$, but now have one less summation due to the constraint that $\beta \in A_\alpha$, giving a term of order $\sigma^2 / N$, and an overall bound of order $N^{-1} $. Taking the square root gives the claimed order $N^{-1/2} $. In detail, for $\alpha =(s,s')$, noting that $\xi_\beta$ and $\xi_\gamma$ are independent if $\beta\in \{(s,t), (s',t)\} $ and $\gamma \in \{(s,u), (s',u)\}$ do not share an element, for $s \ne s'$
   \begin{eqnarray*}
    {\EE ( \eta_\alpha^2 ) }
     &=& \sum_{\beta \in A_\alpha} \sum_{\gamma \in A_\alpha} \mathrm{Cov}(\xi_\beta,  \xi_\gamma ) \\
     &= & \sum_{t} \sum_u \mathrm{Cov}(\xi_{(s,t)},  \xi_{(s,u)}) + \sum_{t} \sum_u \mathrm{Cov}(\xi_{(s',t)},  \xi_{(s',u)}).
   \end{eqnarray*}
   Then, \begin{eqnarray*}
  \sum_{t} \sum_u \mathrm{Cov}(\xi_{(s,t)},  \xi_{(s,u)}) &=& \sum_t \mathrm{Var}(\xi_{(s,t)}) + \sum_t \sum_{u \ne t} \mathrm{Cov}(\xi_{(s,t)},  \xi_{(s,u)}) \\
  &=& \frac{1}{N^4 \sigma^2} \Big\{ \sum_t \mathrm{Var} (h_{\mat{X}} (s, t) + \sum_t \sum_{u \ne t} \mathrm{Cov}(h_{\mat{X}} (s, t) ,  h_{\mat{X}} (s, u) ) \Big\} .      
   \end{eqnarray*}
For $\alpha = (s,s')$ we now use \eqref{eq:varss} and \eqref{eq:covdistinct} to bound 
   \begin{eqnarray} \label{eq:overallvar}
       \sum_{t: (s,t) \in A_\alpha}  \mathrm{Var} (h_{\mat{X}} (s, t)
       &\le& d_0 p_{\rm{max}} + 2 N d_{\rm{max}}^4 p_{\rm{max}}^2
   \end{eqnarray}
where 
   \begin{equation}\label{eq:d0} 
       d_0 = \max_s (g_s(0) + g_s(1))^2 .
   \end{equation}Similarly with \eqref{eq:var} and \eqref{eq:covdistinct}, 
   \begin{eqnarray*}
       \sum_{t: (s,t) \in A_\alpha} \sum_{u \ne t: (s,u) \in A_\alpha} \mathrm{Cov}(h_{\mat{X}} (s, t) ,  h_{\mat{X}} (s, u) ) \} 
       &\le& 2 N d_{\rm{max}}^4 p_{\rm{max}}^2 + 4 N^2 d_{\rm{max}}^4 p_{\rm{max}}^3. 
   \end{eqnarray*}
Collecting these bounds yields
   \begin{eqnarray*}
    \EE ( \eta_\alpha^2 )    
    &\le& \frac{4}{N^4 \sigma^2} d_{\rm{max}}^4 C(N,d_0, p_{\rm{max}}, d_{\rm{max}} )
   \end{eqnarray*}
where 
   \begin{equation*}
   C(N,d_0, p_{\rm{max}}, d_{\rm{max}} )
   =  p_{\rm{max}}^4 +  \frac{N+2}{N^2}p_{\rm{max}}^2 + \frac{d_0 p_{\rm{max}}}{N^3 d_{\rm{max}}^4} 
   \end{equation*}
as given in \eqref{eq:cn}. Thus, with \eqref{eq:etaboundinter}, \begin{eqnarray}\label{eq:etatermbound}
       \sum_{\alpha \in  {\mathcal I}} \EE | \xi_\alpha \eta_\alpha^2|
       &\le& \frac{8}{N} \left(  d_{\mathrm{max}}^2  p_{\mathrm{max}}^2\frac{2N_0}{\sqrt{N\sigma^2}}\right)^2   \frac{1}{\sqrt{N \sigma^2}} d_{\rm{max}}^2 \sqrt{C(N,d_0, p_{\rm{max}}, d_{\rm{max}} )}.
   \end{eqnarray}
In particular, with \eqref{eq:sigmalower}, this bound is of order $N^{-1}$ if $N_0 /N$ is bounded away from 0.

The last term to bound is 
$\sqrt{Var \left( \sum_{\alpha \in  {\mathcal I} } \xi_\alpha \eta_\alpha \right)}.$
Now, $$ \mathrm{Var} \left( \sum_{\alpha \in  {\mathcal I} } \xi_\alpha \eta_\alpha \right) = 
\mathrm{Var}\left( \sum_{\alpha \in  {\mathcal I} } \sum_{\beta \in A_\alpha }\xi_\alpha \xi_\beta  \right) 
=  \sum_{\alpha \in  {\mathcal I} } \sum_{\beta \in A_\alpha }
 \sum_{\gamma \in  {\mathcal I} } \sum_{\delta \in A_\gamma } \mathrm{Cov} (\xi_\alpha \xi_\beta , \xi_\gamma \xi_\delta).  
 $$
 The covariance of $\xi_\alpha \xi_\beta$ and $\xi_\gamma \xi_\delta$ is zero unless two of the 3-tuples of vertex pairs $\{\alpha, \beta\} $ and $\{\gamma, \delta\} $ share at least one vertex pair between them,  where $\beta \in A_\alpha$ and $\delta \in A_\gamma$. Thus, using the rough bound  \eqref{eq:xibound} to obtain
 \begin{eqnarray*}
   \mathrm{Cov} (\xi_\alpha \xi_\beta , \xi_\gamma \xi_\delta) 
   &\le& \frac{16}{N^6}   d_{\mathrm{max}}^8  p_{\mathrm{max}}^8  \frac{1}{( N \sigma^2)^2} 
 \end{eqnarray*}
gives 
 \begin{eqnarray*}
   \mathrm{Var} \left( \sum_{\alpha \in  {\mathcal I} } \xi_\alpha \eta_\alpha \right) 
   &\le& \frac{16 N (2N)^4 }{N^6}   d_{\mathrm{max}}^8  p_{\mathrm{max}}^8  \frac{1}{( N \sigma^2)^2} \\
   &=& \frac{144}{N}   d_{\mathrm{max}}^8  p_{\mathrm{max}}^8  \frac{1}{( N \sigma^2)^2}. 
 \end{eqnarray*}
Hence
\begin{eqnarray}\label{eq:vartermbound}
     \sqrt{\mathrm{Var} \left( \sum_{\alpha \in  {\mathcal I} } \xi_\alpha \eta_\alpha \right)  }
     &\le& \frac{16}{\sqrt{N}}   d_{\mathrm{max}}^4 p_{\mathrm{max}}^4  \frac{1}{N \sigma^2}. 
 \end{eqnarray}
 Combining \eqref{eq:etatermbound} and \eqref{eq:vartermbound} with \eqref{chen_theo} gives the assertion. 
\end{proof}

\begin{remark}\label{rem:power}
Remark \ref{rem:sssparse} shows that this result also covers sparse networks, as long as the expected number of edges increases with $n$. The remark also indicates that Theorem \ref{normal_approx} can 
be used for theoretical power guarantees when the alternative distribution is a model with independent edges also. Here we give the details. First, we recall that the Wasserstein distance can be used to bound the Kolmogorov distance; in particular, for a random variable $X$ with distribution $\mathcal{L}(X)$, and $\mathcal{N}(0,1)$ denoting the standard normal distribution, 
\begin{align}\label{eq:dtv}
    \sup_x | \PP ( X \le x) - \Phi (x) | \le 2 \sqrt{|| \mathcal{L}(X) - \mathcal{N}(0,1)||}, 
\end{align}
see for example, (C.2.6) in \cite{nourdin2012normal}. Theorem \ref{normal_approx} hence gives    
\begin{align} \label{power1}
    \PP_1 (W_0 < x \mbox{ or } W_0 > y) 
    & \ge 1 - \Phi \left( \frac{1}{\sigma_1} ( \sigma_0 y  + \mu_0 - \mu_1) \right) + \Phi \left( \frac{1}{\sigma_1} ( \sigma_0 x  + \mu_0 - \mu_1) \right) - 4 
\left( \frac{R_1}{\sqrt{N}} \right)^\frac12.
\end{align}  
If $\alpha$, the level of the test, is fixed, and if $x$ and $y$ are such that $\PP_0 (x < W_0 < y) = 1 - \alpha,$ where $\PP_0$ denotes the probability under $\mat{p}_0$, then by Theorem \ref{normal_approx} and the bound  \eqref{eq:dtv} we have  
\begin{align*}
  1 - \alpha &=  \PP_0 (x < W_0 < y) = \Phi(y) - \Phi (x) +  \frac{R_0}{\sqrt{N}}. 
\end{align*}
 Thus,  
$1 - \Phi(y) +  \Phi(x) = \alpha + \frac{R_0}{\sqrt{N}}$, and with \eqref{power1}, using these $x$ and $y$ for the test, 
\begin{align*}
\PP_1 ( \text{ reject }  H_0 )
=&\,   \PP_1 (W_0 < x \text{ or } W_0 > y) \\ 
  \ge &\, 1 - \Phi \left( \frac{1}{\sigma_1} ( \sigma_0 y  + \mu_0 - \mu_1) \right) + \Phi \left( \frac{1}{\sigma_1} ( \sigma_0 x  + \mu_0 - \mu_1) \right) - 4 
\left( \frac{R_1}{\sqrt{N}} \right)^\frac12\\ 
=&\, \Phi(y) - \Phi (x) - 1 + \alpha - \frac{R_0}{\sqrt{N}}\\&
+ 1 - \Phi \left( \frac{1}{\sigma_1} ( \sigma_0 y  + \mu_0 - \mu_1) \right) + \Phi \left( \frac{1}{\sigma_1} ( \sigma_0 x  + \mu_0 - \mu_1) \right) - 4 
\left( \frac{R_1}{\sqrt{N}} \right)^\frac12
\\
=& \alpha + 
\Phi(y) -  \Phi \left( \frac{1}{\sigma_1} ( \sigma_0 y  + \mu_0 - \mu_1) \right)
+ \Phi \left( \frac{1}{\sigma_1} ( \sigma_0 x  + \mu_0 - \mu_1) \right) - \Phi (x) \\
&  - \frac{R_0}{\sqrt{N}} - 4 
\left( \frac{R_1}{\sqrt{N}} \right)^\frac12.
\end{align*}
\end{remark}

\section{COMPUTATIONAL COMPLEXITY}\label{app:complex}

The computation time of IRG-gKSS$^2$ depends on the size $n$ of the network and the computational complexity of the underlying graph kernel, here denoted by $C(K)$. To compute a single value of IRG-gKSS$^2$, the algorithm computes the kernel matrix for a list of $\binom{n}{2}+1$ networks of size $n$ and uses it to calculate a value of the final test statistic in a way that requires some simple matrix operations, resulting in computational complexity of order $C(K) O(n^2)$ to compute a single value of the IRG-gKSS test statistic. A full test requires computing an IRG-gKSS$^2$ for the observed network and a set of IRG-gKSS$^2$ for all $M$ simulated networks under the null model, with an overall computational complexity of order $M C(K)O(n^2)$. For the estimated IRG-gKSS$^2$ with edge-resampling (using $B$ edge indicators instead of the full set of all possible edge indicators), the list of networks for which we need to compute the kernel matrix reduces to $C(K)O(MB)$. For a WL kernel, $C(K)= O(hm)$, with $h$ denoting the height of the subtree in the WL kernel and $m$ the number of edges of the network. 
    
Figure \ref{fig:execution_time} shows the execution time (in seconds) for ER graphs with $p=0.06$ and varying size, using the WL kernel with height $h=3$ on a system with an Intel Core i7-8700, 32 GB DDR4-2666 RAM, and Windows 11 Enterprise using \texttt{R} 4.5.0. The scaling with respect to size is almost quadratic, reflecting that  $O(n^2)$ network comparisons are carried out, whereas the scaling with respect to $M$, the number of simulated networks, is almost linear. The gain in computational complexity due to edge resampling is reflected in Table \ref{tab:comp_resamp} and Figure \ref{fig:execution_time_resamp}.
    
\begin{figure}[ht] \label{fig:comptime}
  \centering
  \includegraphics[width=\textwidth]{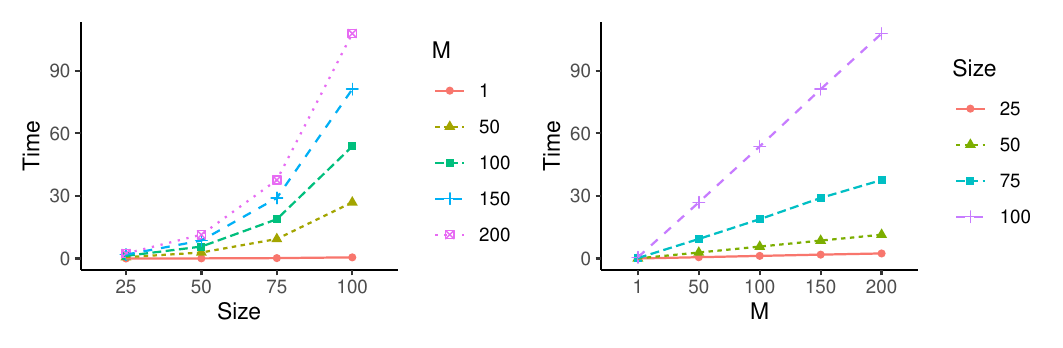}
  \vspace{.1in}
  \caption{Execution time {(minutes)} to calculate IRG-gKSS, using the WL kernel with $h=3$, for networks simulated from ER$(n, 0.06)$ models. Left: dependence on size $n$, for different numbers $M$ of simulated networks; right: dependence on $M$, for different sizes $n$.
  }  \label{fig:execution_time}
\end{figure}

\begin{figure}[ht] \label{fig:comptime_resamp}
  \centering
  \includegraphics[width=0.6\textwidth]{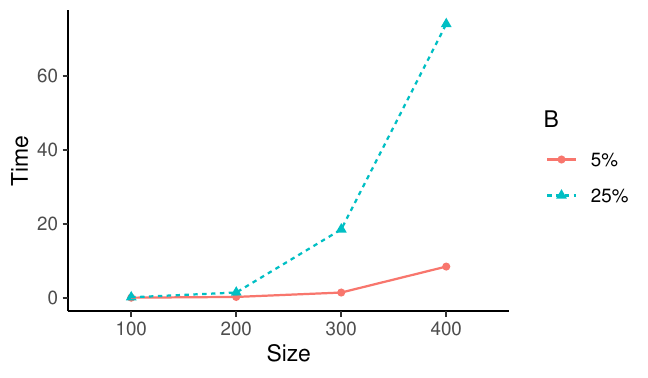}
  \vspace{.1in}
  \caption{Execution time (minutes) for a single computation of the re-sampled version IRG-gKSS,  obtained by edge resampling, using the WL kernel with $h=3$, for networks simulated from ER$(n, 0.06)$ models. The size $n$ is given on the $x$-axis, and the edge resampling proportion $B$ is given in the legend.
  }  \label{fig:execution_time_resamp}
\end{figure}
\begin{table}
    \centering
        \caption{Execution time (minutes) to calculate estimated IRG-gKSS using edge resampling}
    \begin{tabular}{lcccc}
    \textbf{n}     & \textbf{100} & \textbf{200} & \textbf{300} & \textbf{400} \\
    \hline \\
    IRG-gKSS\_5 & 0.03 & 0.21 & 1.16 & 8.98 \\
    IRG-gKSS\_25 & 0.13 & 1.32 & 17.55 & 73.98 \\
    \end{tabular}
    \label{tab:comp_resamp}
\end{table}
\begin{figure}
  \centering
  \includegraphics[width=0.7\textwidth]{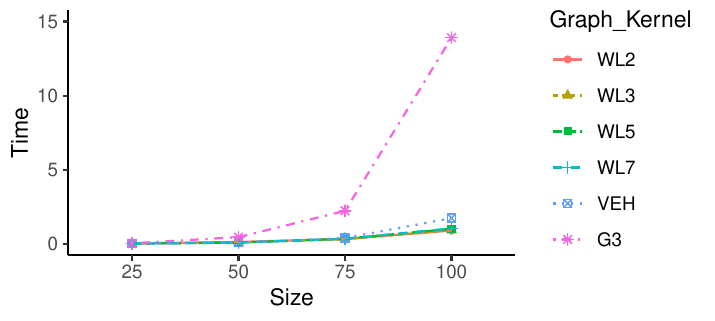}
  \vspace{.1in}
  \caption{Execution time (minutes) to calculate a single value of IRG-gKSS test statistic using different graph kernels. The kernels used are the Weisfeiler Lehman (WL) kernel with different subtree height $h$ named as WLh in the legend, the vertex edge histogram Gauss kernel (VEH) and the graphlet kernel with graphlets of size $k=3$. The networks used are simulated from ER$(n, 0.06)$ models of size $n$ given on the x-axis.
  }  \label{fig:execution_time_diff_kernels}
\end{figure}

\paragraph{Code details} 
For simulating networks, we use a similar implementation as in the {\texttt{GoodFitSBM}} (0.0.1) {\texttt{R}} package. For graph kernels in {\texttt{R}} we use the  {\texttt{graphkernel}} ($1.6.1$) package. The code to compute IRG-gKSS from the kernels extends the code of \cite{xu2021stein}. The code also uses the {\texttt{R}} libraries {\texttt{Matrix}} and {\texttt{HelpersMG}}. For the graphlet kernel, we also use {\texttt{grakel}} version {0.1.8} from Python. The code is available at the GitHub site \url{https://github.com/AnumFatima89/IRG-gKSS}.
   
\section{ADDITIONAL DETAILS AND RESULTS FOR SYNTHETIC EXPERIMENTS} \label{sec:add_syn_exp}

This section gives further details of the synthetic experiments from the main paper as well as some additional results. All tests are carried out at the 5\% level.

\subsection{Other Tests for Assessing the Fit to Inhomogeneous Random Graph Models} \label{sec:background}

In this paper, we use the following two testing methods to compare the performance of IRG-gKSS. 

\subsubsection{The Generalised Likelihood Ratio Test}
For an observed network $\mat{y}$ with $q(\mat{y}) = \mathbb P (\mat{Y} = \mat{y})$, the true likelihood of the observed network, a classical goodness-of-fit test for a specific IRG among a larger class of IRGs is a generalised likelihood ratio (GLR) test, based on the ratio of the likelihoods \eqref{eq:irg} for the null model and the alternative model. Let $\mat{Q}_0$ be the $n \times  n$ probability matrix with  entries $p_{0;u,v},$ for $1 \le u < v \le n$, specifying a fixed $\ERMM$ model. To test the null hypothesis $H_0: \mat{Y} \sim \ERMM(\mat{n},\mat{{Q}}_0)$ against a general IRG alternative model $\mat{p} \in \Omega$, the GLR test statistic is 
\begin{equation}
    \Lambda = \frac{\prod_{{1 \le u < v \le n}} \left(1-p_{0;u,v} \right)^{1 - x_{u,v}} \left(p_{0;u,v} \right)^{x_{u,v}}}{\sup_{\mat{p} \in \Omega}\left( \prod_{{1 \le s < r \le n}} \left(1-p_{s,r} \right)^{1 - x_{s,r}} \left(p_{s,r}\right)^{x_{s,r}} \right)}.\label{eq:glrt}
\end{equation} 
Then under the null hypothesis, by Wilks' Theorem, $-2\ln \Lambda$ is asymptotically $\chi^2_{k}$ distributed, where $k$ is the number of free parameters under $\Omega$. The GLR test here is a one-sided test;  the null hypothesis is rejected at level $\alpha$ if $-2\ln \Lambda > \chi^2_k(\alpha)$. In many test situations, this test is optimal, see \cite{zeitouni1992generalized}. However, it strongly relies on the alternative distribution being an IRG.

\subsubsection{The Spectral Gap Test by Lei}

The second test we consider for comparison is the spectral gap test proposed by \citet{Lei2016}. The test statistic is based on the largest singular value of a residual matrix obtained by subtracting the estimated block mean effect from the adjacency matrix. For large networks, this statistic asymptotically follows a Tracy-Widom distribution of index $1$. The corresponding goodness-of-fit test evaluates the null hypothesis of $K_0$ communities in a stochastic blockmodel (SBM) against the alternative that there are more than $K_0$ communities. \citet{Lei2016} further suggests using this test sequentially to estimate the number of communities in the network. We summari{s}e the test statistic and procedure below.

Let $A$ denote the observed adjacency matrix, and let $g \in \{1, \ldots, L\}^n$ denote the group membership vector, where $L$ is the number of communities in an SBM on $n$ nodes. Let $B \in [0,1]^{L \times L}$ denote the symmetric community-wise edge probability matrix, so that the edge probability between vertices $u$ and $v$ is
$$\PP(u \sim v) = P_{uv} = B_{g_u, g_v},$$ 
\citet{Lei2016} defines the residual matrix
$$\Tilde{A}^*_{uv} = \frac{A_{uv} - P_{uv}}{\sqrt{(n-1)P_{uv}(1-P_{uv})}},  \quad i \neq j,  \quad \text{and} \quad \Tilde{A}^*_{uv} = 0, i = j,$$
which subtracts the block mean effect and rescales the entries. The test assumes relatively balanced communities; in detail, for $K=K^{(n)}$ denoting the number of communities and community assignments $g_i^{(n)}, i=1, \ldots, n$,  it assumes that there is a $c_0>0$  such that $\min_{1 \le k \le K^{(n)}} \{i: g_i^{(n)} = k\} \ge c_0 n / K^{(n)} $   for
all $n$. Moreover, the community assignment is assumed to be asymptotically consistent, $K^{(n)} = O(n^{1/6 - \tau})$ for some $\tau >0$, and it is assumed that the edge probabilities are uniformly bounded away from 0 and 1. Under these conditions,
$$n^{2/3}[\lambda_1(\Tilde{A}^*) - 2] \xrightarrow{d} TW_1 \quad \quad \text{and} \quad \quad n^{2/3}[\lambda_n(\Tilde{A}^*) - 2] \xrightarrow{d} TW_1,$$
where $\lambda_1$ and $\lambda_n$ denote the largest and smallest eigenvalues of $\Tilde{A}^*$, respectively, and $TW_1$ is the Tracy-Widom distribution of index 1. When $g$ and $B$ are unknown, they are replaced by estimates $(\hat{g}, \hat{B})$. \citet{Lei2016} argues that, under the null hypothesis $K = K_0$, the estimates $(\hat{g}, \hat{B})$ yield similar asymptotic behaviour for the eigenvalues of the estimated residual matrix $\Tilde{A}$. Consequently, in \cite{Lei2016} the {\it spectral gap test}  statistic
$$T_{n,K_0} = n^{2/3}[\sigma_1(\Tilde{A}) -1]$$
is introduced, where $\sigma_1$ denotes the largest singular value of $\Tilde{A}$. The test rejects $H_0$ if $T_{n,K_0} \ge t(\alpha/2)$, where $t(\alpha/2)$ is the $(1-\alpha/2)$ quantile of the $TW_1$ distribution.

The spectral gap test is asymptotic; for a given realisation of one network, it is not obvious whether the assumptions are satisfied because there is no sequence of networks available. For smaller networks, \citet{Lei2016} recommends a bootstrap-corrected version, for which no theoretical guarantees are given. In this approach, the theoretical centring and scaling are replaced by empirical estimates obtained from $M$ simulated adjacency matrices generated under the SBM with $(\hat{g}, \hat{B})$. The bootstrap-corrected test statistic is
\begin{equation} \label{eq:bootspec}
    T_{n,K_0}^{(boot)} = \mu_{tw} + s_{tw}\max\left(\frac{\lambda_1(\Tilde{A}) - \hat{\mu}_1}{\hat{s}_1}, -\frac{\lambda_n(\Tilde{A}) - \hat{\mu}_n}{\hat{s}_n}\right),
\end{equation}
where $\mu_{TW}$ and $s_{TW}$ are the mean and standard deviation of the Tracy-Widom distribution of index 1, and $(\hat{\mu}_1, \hat{s}_1)$ and $(\hat{\mu}_n, \hat{s}_n)$ are the sample mean and standard deviation of the largest and smallest eigenvalues of the simulated matrices, respectively. In this paper, we use this bootstrap corrected version of Lei's spectral gap statistic; we denote it by ST. 

\subsection{Additional Experiment: A Non-Linear Preferential Attachment Model}
\label{ssec:nlpa}

As an additional experiment, in which the alternative distribution has dependent edges, we simulate networks on 50 vertices from a non-linear preferential attachment model (NLPA) with different parameter settings and test the ER$(n,p)$ model fit on these simulated networks. The algorithm to simulate a network under an NLPA$(m, \alpha)$ model starts with an initial network with $m_0 \ge m$ vertices. At each step, we add a vertex, and sample $m$ distinct vertices from the existing vertices in the network, sampling vertex $i$ with probability 
$$p_i = \frac{k_i^{\alpha}}{\sum_j k_j^{\alpha}}.$$
Here $\alpha >0$ is the {\it power} parameter of the model, $k_i$ is the degree of vertex $i$ at the current step, and the sum is over all pre-existing vertices $j$ (so that the denominator results in twice the current number of edges in the network). If vertex $i$ is sampled, then an edge between the new vertex and vertex $i$ is created. The process continues until we have a network with $n$ vertices. We use the `\textit{sample\_pa}' function from the \textit{igraph} library in \texttt{R} to simulate these networks. The resulting networks are different from ER networks through their construction; edge indicators are not independent, and in particular, for large networks, one would anticipate to see more vertices with large degrees. Figure \ref{fig:nplamodel} shows some instances of this model on $50$ vertices and different parameter settings.

\begin{figure}[ht]
    \centering
    \includegraphics[width=0.6\linewidth]{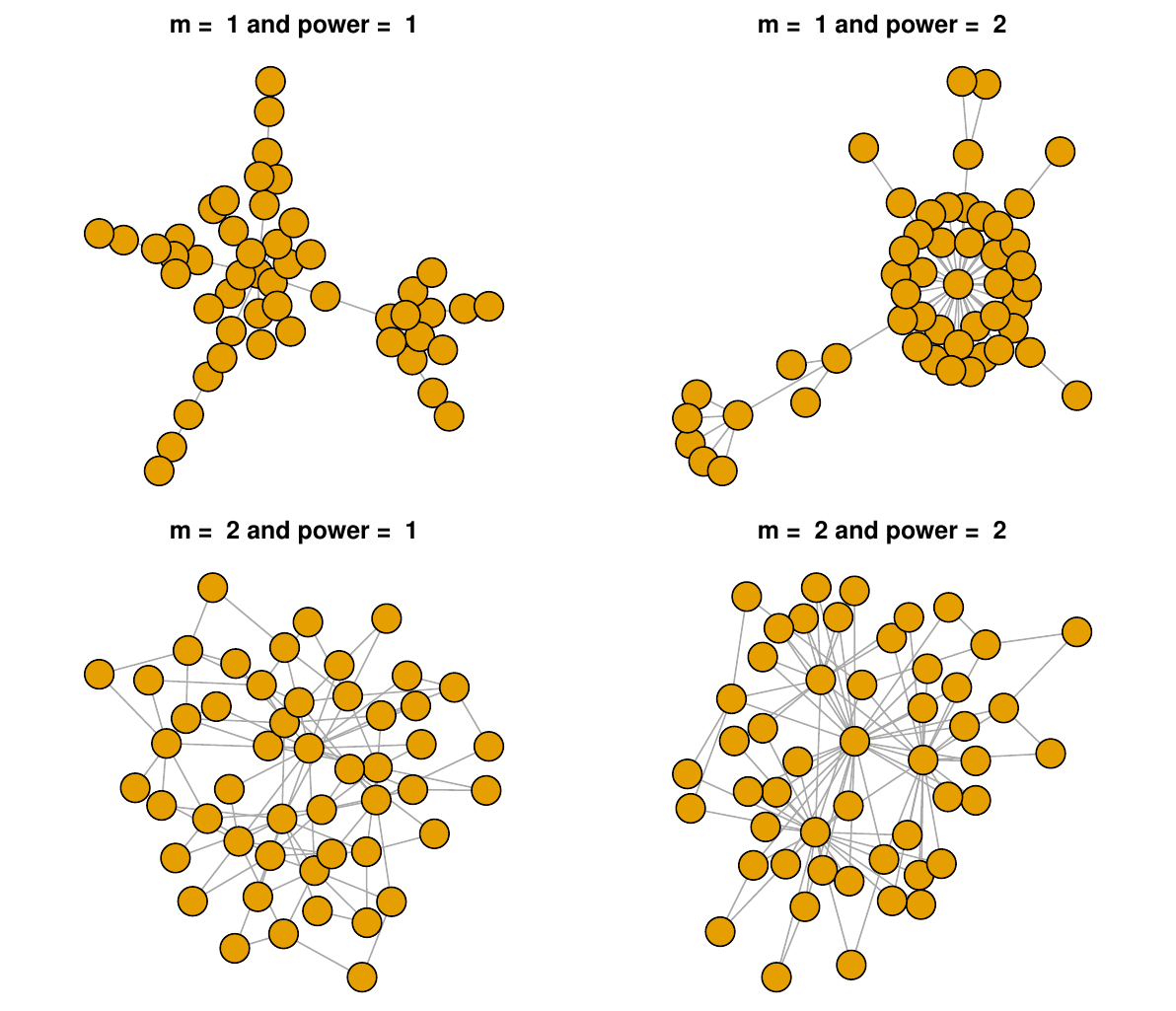}
    \vspace{.1in}
    \caption{Networks on $50$ vertices created from an NLPA model with different parameter settings.}
    \label{fig:nplamodel}
\end{figure}

The results in Table \ref{table:ER_PA} give the proportion of times that the IRG-gKSS test at level $5\%$ rejected the fit of an ER$(n,p)$ model to a network simulated from an NLPA$(m, \alpha)$ model in 50 repetitions. We used IRG-gKSS with three kernels, namely the WL kernel with $h=2$, the WL kernel with $h=3$ and the graphlet kernel with $k=3$,  and recorded their results separately in Table \ref{table:ER_PA}. 
We find that for large $n$, the difference to an ER model is clearly detected. For small $n$, only some instances are deemed by the test to be clearly different from an ER model. For the WL kernels, the difference is easier to detect for $m=1$ compared to $m=2$. In contrast, for the graphlet kernel, the difference is easier to detect for $m=2$ compared to $m=1$. In particular, none of the kernels has uniformly more power than the other kernels; also, the spectral test (ST) performs comparably. 

\begin{table}
  \caption{Testing the fit of the ER model on a network simulated from a preferential attachment model}
  \label{table:ER_PA}
  \centering 
  \begin{tabular}{lllcccc}
&&&&&&\\ 
\multicolumn{3}{l}{\textbf{MODEL}} & \multicolumn{4}{c}{\textbf{PROPORTION REJECTED}} \\
\cmidrule(lr){4-7}
 \multicolumn{3}{l}{\textbf{PARAMETERS}} &  \textbf{ WL with $h=2$} &\textbf{ WL with $h=3$} & \textbf{GRAPHLET with $k=3$} & \textbf{ST} \\
\hline \\
$n=20$ & $\alpha = 1$ & $m = 1$ & 0.42 & 0.46 & 0.14 & 0.24 \\

&& $m = 2$ & 0.06 & 0.02 & 0.42 & 0.14 \\

&$\alpha = 2$ & $m = 1$ & 1.00 & 1.00 & 0.94 & 0.92 \\

&& $m = 2$ & 0.80 & 0.90 & 1.00 & 0.90 \\
 \addlinespace
 
$n=50$ & $\alpha = 1$ & $m = 1$ & 0.94 & 0.78 & 0.04 & 0.64 \\

&& $m = 2$ & 0.38 & 0.18 & 0.82 & 0.80 \\

&$\alpha = 2$ & $m = 1$ & 1.00 & 1.00 & 1.00 & 1.00 \\

&& $m = 2$ & 1.00 & 1.00 & 1.00 & 1.00 \\
 \addlinespace
 
$n=100$ & $\alpha = 1$ & $m = 1$  & 1.00 & 0.98 & 0.06 & 0.92 \\

&& $m = 2$ & 0.98 & 0.96 &  0.96 & 1.00 \\

&$\alpha = 2$ & $m = 1$ & 1.00 & 1.00 & 1.00 & 1.00 \\ 

&& $m = 2$ & 1.00 & 1.00 & 1.00 & 1.00 \\
\end{tabular}
\end{table}

\subsection{Additional Experiment: Testing the Fit of an ERMM} \label{experi:1}

When testing within a parametric model family, a (generalised) likelihood ratio test (LRT)  is usually most powerful, see for example \cite{zeitouni1992generalized}. How much less powerful is an IRG-gKSS test? To assess this question, in this experiment, which complements those in the main text, we test the fit of an ERMM$(\mat{n}, \mat{Q}_0)$ to a network simulated from an ERMM$(\mat{n}, \mat{Q}_1)$ using Algorithms \ref{alg:algo1} and \ref{alg:algo2} and a likelihood ratio test (LRT). More specifically, we simulate networks of size $50$ from an ERMM$(\mat{n}, \mat{Q}_1)$ with two groups, where the within-block probabilities are identical but the between-block probabilities are different. We use $n=50$ and the following setting for $\mat{Q}_0$ and $\mat{Q}_1$;
\begin{align}\label{eq:matq01}
    \mat{Q}_{01} = \begin{pmatrix}
    0.20 & 0.01 \\
    0.01 & 0.20
\end{pmatrix} \quad \mbox{ and } \quad 
\mat{Q}_{02} = \begin{pmatrix}
    0.6 & 0.1 \\
    0.1 & 0.6
\end{pmatrix},
\end{align} 
and 
\begin{align} \label{eq:matq11}
    \mat{Q}_{11} = \begin{pmatrix}
    0.20 & Q_{12} \\
    Q_{12} & 0.20
\end{pmatrix} \quad \mbox{ and } \quad 
\mat{Q}_{12} = \begin{pmatrix}
    0.6 & Q_{12} \\
    Q_{12} & 0.6
\end{pmatrix},
\end{align}
with $Q_{12}$ a parameter. We use unbalanced (subscript ub) and balanced group splits  given by 
\begin{align} \label{eq:nnumbers}
    \mat{n}_{ub} = (10,40) \quad \text{and} \quad \mat{n} = (25,25).
\end{align}

We compare the IRG-gKSS test, without and with edge resampling, using a WL kernel with height $h=3$, to the likelihood ratio (LR) test, for which we specify the parameters in the alternative model as given in \eqref{eq:nnumbers} and \eqref{eq:matq11}.

 \begin{figure}
         \centering
    \includegraphics[width=\columnwidth]{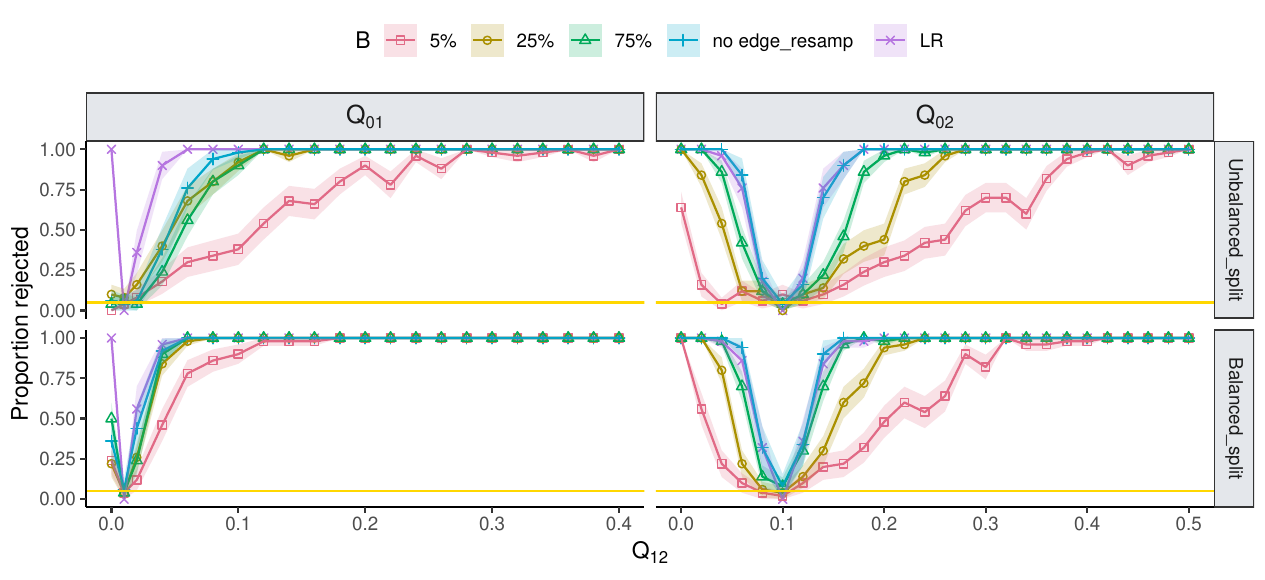} 
    \vspace{.1in}
         \caption{Power of the test at level $5\%$ for  the fit of an $\text{ERMM}(\mat{n}, \mat{Q})$, with $\mat{n}$ and $\mat{Q}$ from \eqref{eq:nnumbers} and \eqref{eq:matq01}, to $50$ networks simulated from an $\text{ERMM}(\mat{n}, \mat{Q})$ model, with $\mat{n}$ and $\mat{Q}$ from \eqref{eq:nnumbers} and \eqref{eq:matq11}, using a WL kernel with ${h = 3}$. The yellow line indicates the significance level (0.05).
         }
         \label{fig:experiment_ERMM}
     \end{figure}

Figure \ref{fig:experiment_ERMM} shows the proportion of tests that reject the null model at level $5\%$. For the test with edge-resampling using Algorithm \ref{alg:algo2}, we repeat the test for different values of $B$ in the test statistic \eqref{t_stat_s_net_resamp}. In Figure \ref{fig:experiment_ERMM}, the values of $B$ are reported as a proportion of $N = \binom{n}{2}$, the number of possible edges in the network of size $n$. 

From Figure \ref{fig:experiment_ERMM}, we observe that the power of IRG-gKSS is lower than that of the LR test, which is supposed to perform best in this experimental setting, but not by a large margin. Also, with edge-resampling, IRG-gKSS performs reasonably well even when sampling only $25\%$ of the vertex pairs. Moreover, in the sparse network setting, using as few as $5\%$ of the vertex pairs still gives a test with some power. 

\subsection{Additional Details: Planted Hubs}
\label{ssec:planted}

As mentioned in Section \ref{sec:plantedhubs} of the main text, the planted hubs experiment is based on two choices, the intended number $R$ of hubs, and the intended excess degree parameter $k$, and uses different network models. We denote by $s_d=s_d(G)$ the standard deviation of the vector of degrees in a network $G$. We simulate a series of networks $G$ from an ERMM$(\mat{n}, \mat{Q})$ model. We then attempt to plant $R$ hubs in the these networks which increase the largest degree by $k  s_d$ without disturbing the edge density of the original network, using Algorithms \ref{alg:hubs} and \ref{alg:hub}, and test the fit of the ERMM$(\mat{n}, \mat{Q})$ model on these networks with planted hubs.

\begin{algorithm}[ht]
   \caption{Iterative Planted Hubs Construction}
   \label{alg:hubs}
\begin{algorithmic}
\State {\bfseries Input:} \, $G$:  Input graph with $n$ vertices \\
   \qquad  \quad \quad $R$: Number of hub planting attempts\\
   \qquad  \quad \quad $k$: Scaling factor for maximum degree increase \\
   \qquad  \quad \quad $C$: A numeric vector of vertex group labels
   
   \State {\bfseries Output:} A modified graph $G_{hubs}$ with up to $R$ planted hubs

   \State {\bfseries Initialize:} \\
    \quad $G_{hubs} = G$ (copy of the original graph) \\
    \quad Get {$\mat{d}=(d_0, \ldots, d_{n-1})$}, the vector with entry $d_\ell$ the {degree of vertex $\ell$} in $G$ \\
    \quad Get $d_m = \max(\mat{d})$, 
    the  maximal degree
    \\
    \quad Set iteration counter $i = 1$

    \While{$d_m != \text{NA}$ and  $i \le R$}
    \State Apply Algorithm \ref{alg:hub} (planting a hub) to $G_h$ using, $k$, and $C$ from input and current $d_m$ 
    \State Update $G_{hubs}$ with the modified graph 
    \State Increment $i = i + 1$ 
    \State Recompute the updated degree vector  $\mat{d}$ from the updated graph $G_{hubs}$
    \State Sort the entries of the updated $\mat{d}$ in descending order 
    \State Set $d_m$ to equal the $i$-th largest unique degree 
   \EndWhile
   \State \textbf{return} The modified graph $G_{hubs}$
\end{algorithmic}
\end{algorithm}

\begin{algorithm}[ht!]
   \caption{Planting a hub in a simulated network}
   \label{alg:hub}
\begin{algorithmic}
    \State {\bfseries Input:} $G$: Input graph  \\
   \qquad \quad \, $d_m$: Initial degree of the hub vertex\\
   \qquad \quad \, $k$: Scaling factor for maximum degree increase  \\
   \qquad \quad \, $C$: A numeric vector of vertex group labels 
   
   \State {\bfseries Output:} A modified graph $G_{hub}$ with a "planted hub" (or original $G$ if no hub planting is possible).

   \State {\bfseries Compute vertex degrees and standard deviation:}\\
    \quad Get $\mat{d}$, the vector with entry $d_\ell$ the {degree of vertex $\ell$} in $G$ 
    \\
    \quad Compute $s_d$, the standard deviation of the vector $\mat{d}$ 
    
    \State {\bfseries Identify hub candidate:} \\
    \quad Set $u_{hub}$, the first vertex in the vertex list of graph $G$ with degree exactly $d_m$
    
    \State {\bfseries Set the number of new neighbours for $u_{hub}$:} \\
    \quad Get $d^* = \max(1, \lceil k * s_d \rceil)$ \\
    \quad Set $n_{new\_nei} = d^* - d_{u_{hub}}$, where $d_{u_{hub}}$ is the degree of $u_{hub}$  

    \State {\bfseries Find the potential additional neighbour set $target\_V$ for $u_{hub}$:} \\
    \quad $target\_V = V(G)$ excluding the current neighbours of $u_{hub}$ and $u_{hub}$ itself \\
    \quad \textbf{if} $d^* = d_{u_{hub}}$ or $target\_V$ is empty \textbf{then} return $G$ (in this case no hub planting is possible)
    \State {\bfseries Filter $target\_V$ by a group membership constraint:} \\
    \quad For each $v$ in $target\_V$, keep $v$ if  at least one neighbour of $v$ has the group membership of $u_{hub}$ \\
    \quad Store these in $filtered\_targets$ \\
    \quad \textbf{if} $filtered\_targets$ is empty \textbf{then} return $G$ (in this case no hub planting is possible)
    \State {\bfseries Sample potential neighbours to add:}
    \If{$length(filtered\_targets) \le n_{new\_nei}$} 
     set $pot\_nei = filtered\_targets$
    \Else \, Randomly sample $n_{new\_nei}$ vertices from $filtered\_targets$ and store in $pot\_nei$
    \EndIf

    \State {\bfseries Prepare for edge replacement:} \\
    \quad Initialize empty vectors $to\_del$ and $new\_nei$ 

    \For{$v$ in $pot\_nei$} \\
    Get neighbors of $v$ of the same type as $u_{hub}$ and store in $N_v$
    \If{$N_v$ is empty} skip to next $v$ 
    \Else \, Sample a neighbour $del\_vertex$ from $N_v$ to disconnect with $v$  \\
    \qquad \quad \quad Concatenate the pair $(v, del\_vertex)$ with $to\_del$ vector \\
    \qquad \quad \quad Check whether the resulting vector ($to\_del$), does not contain duplicate pairs (edges) 
    \\
    \qquad \quad \quad If it does: resample the neighbour to delete until the resulting vector does not contain  \\
    \qquad \quad \quad duplicate pairs; if no such neighbour can be found: drop $v$ and move to the next candidate \\
    \qquad \quad \quad  in $pot\_nei$\\
    \qquad \quad \quad Add $v$ to $new\_nei$
    \EndIf
    \EndFor

    \State {\bfseries Post-processing edge list:} \\
    \quad \textbf{if} $to\_del$ is empty or does not match expected size ($2 * length(new\_nei)$) \textbf{then} return $G$
    \State {\bfseries Apply graph edits:} \\
    \quad {Delete the edges in $to\_del$ from $G$}  \\
    \quad Add an edge between $u_{hub}$ and each vertex in $new\_nei$ 
    \State {\textbf{return} The modified graph $G_{hub}$}
\end{algorithmic}
\end{algorithm}

For the simulations, we choose $n=30$. In Section \ref{sec:plantedhubs}, results are given for a two-group ERMM with even split between the groups and probability matrix 
\begin{align} \label{eq:matQ}
    {\mat{Q}}
    = \begin{pmatrix}
    0.2{9} & 0.01 \\
    0.01 & 0.2{2}
\end{pmatrix}
\end{align} 
resulting in the within-group subgraphs being fairly dense, and the between-group subgraphs being fairly sparse. In the main text, we showed results from the unbalanced split $\mat{n}_{ub} = ({8},{22})$; here, we show results for the equal split, given by
\begin{align}
 \mat{n}
= ({15},{15}).  \label{eq:nnumber} 
\end{align}

In each of these plots, we also report the results for ST and for the GLR test. The conclusion is as in the main paper: The GLR test is not able to detect the planted structure; ST fails to detect it in the majority of the runs.

\begin{figure}
  \centering
  \includegraphics[width=\textwidth]{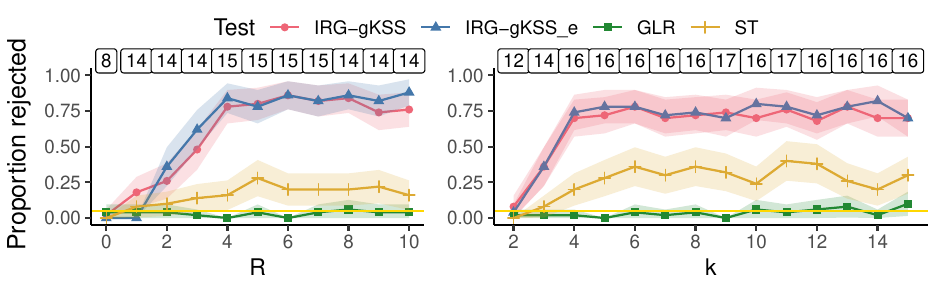}
  \vspace{.1in}
  \caption{Power of the tests for the fit of an ERMM$(\mat{n}, \mat{Q})$ model to a network of size 30 with planted hubs, with $\mat{n}$ from \eqref{eq:nnumber} and $\mat{Q}$  from \eqref{eq:matQ}. The numbers in the boxes at the top of the plot are the average maximum degrees observed in $m=50$ repetitions of the test for each setting on the $x$-axis. For the figure on the left side, we fix $k=4$ and let $R$ vary, and in the right side figure, we fix $R=3$ and let $k$ vary. For this experiment, we use a WL kernel with $h = 3$.}
  \label{fig:plantedhub_b2}
\end{figure}

\subsection{Additional Details: Planted Cliques}\label{ssec:clique}

In the next experiments, we simulate a network from an ER$(n,p)$ or an ERMM$(\mat{n}, \mat{Q})$ model, move its existing edges to plant a clique of size $K$, if possible, keeping the number of edges of each type as of the original simulated network, and test the fit of the ER$(n,p)$ or the ERMM$(\mat{n}, \mat{Q})$ to the transformed network with a planted clique. If planting a clique is not possible, because the network contains fewer than ${K \choose 2}$ edges, we discard that network and simulate a new network to plant the clique; we also record the frequency with which this happens. By planting a clique, we change the structure of the network while keeping the edge density of each edge type constant. Thus, the null model is misspecified for the resulting network with a planted clique. For ER$(n,p)$, in the figures in this subsection, the numbers inside the boxes in the plots are the total number of networks in $m=50$ samples from an ER$(n,p)$ that had fewer than $\binom{K}{2}$ edges and were, therefore, not used in the experiment. We repeat the pure hypothesis test for 50 simulated networks with a planted clique of size $K$, while using the set of $M=200$ IRG-gKSS values calculated from $M$ networks simulated from the null model  ERMM$(\mat{n}, \mat{Q})$. 

\begin{algorithm}
   \caption{Planting a clique in a simulated network}
   \label{alg:clique}
\begin{algorithmic}
    \State {\bfseries Input:} $G$: Input graph \\
   \qquad \quad \, $K$: Size of the clique to add \\
   \qquad \quad \, $C$: A numeric vector of vertex group labels \\
   \qquad \quad \, max\_rep: Maximum number of repetitions for the search loop
   
   \State {\bfseries Output:} A modified graph $G_{K}$ with a ``planted clique" of size $K$, original $G$ if the selected vertices already form a clique or an empty graph or if planting a clique of size $K$ is not possible

   \State {\bfseries Get lists of edges and non-edges:}\\
    \quad Get \emph{edges}, the edge list of graph $G$\\
    \quad Get \emph{edges\_type}, a vector assigning type labels to all edges in \emph{edges} \\
    \quad Get \emph{non\_edges}, the edge list of the complement graph of $G$\\
    \quad Get \emph{edges\_type}, a vector assigning type labels to all edges in \emph{non\_edges} \\
    \quad Get \emph{edge\_labels} a set of all unique edge labels
    
    \State {\bfseries Create containers for edges/non\_edges  with respect to their type:} 
    \quad \For{$i$ in \emph{edge\_labels}} \\
      \quad\emph{edge\_type\_i} = edges[WHERE edges\_type == $i$] \\
      \quad \emph{non\_edge\_type\_i} = edges[WHERE non\_edges\_type == $i$]
      \quad \EndFor
    \State {\bfseries Iteratively attempt to add a clique:} \\
    \quad $counter = 0$
    \Repeat \\
    \quad $counter += 1$ 
    \State{\bfseries Sample K nodes and form all possible edges (clique):} \\
    \quad Get  a random sample \emph{smpl} of $K$ vertices from $G$ \\
    \quad Create  a list \emph{new\_edges} of all 2-node combinations from \emph{smpl} \\
    \quad Sort each pair in \emph{new\_edges} 
    \State{\bfseries Identify which of these new edges are non-edges in $G$:} \\
    \quad Get  the  list of pairs \emph{edges\_to\_add} from \emph{new\_edges} that do  not already hold an edge \\
    \quad Get  the  vector \emph{add\_edge\_type} containing labels of edges from \emph{edges\_to\_add}

    \If{number of $edges\_to\_add == 0$}
            Break the loop and \textbf{return} $G$, so that, $G_{K} = G$
    \EndIf
    \State{\bfseries Count new edges by type:} \\
    \quad Get the frequency table \emph{new\_type}  of \emph{add\_edge\_type} 
    \State{\bfseries Identify candidate edges to delete (those not in new\_edges):} \\
    \quad Get  the  list of edges \emph{poten\_del} to delete, which are not in \emph{new\_edges} \\
    \quad Get the  vector \emph{to\_del\_edge\_type}  of corresponding types of \emph{poten\_del}
    \State{\bfseries Check if enough edges of each type exist to delete:}\\
    \quad Check if there are at least as many edges  of each type available for deletion in \emph{poten\_del} as there are new edges to add in \emph{edges\_to\_add} \\
    \quad Break the loop if there are enough edges to delete or
    \If{$counter > max\_rep$} 
        Break the loop \\
        \quad Print the message " Network too sparse. Not enough edges to delete." \\
        \quad \textbf{return} An empty graph and the counter
        \EndIf
    \Until{One of the Breaks are executed}
    \State{\bfseries Sample edges to delete by type:} \\
    \quad Get \emph{edges\_to\_del} randomly selecting edges, of each type, from \emph{poten\_del}, with number of selected edges corresponding to the number of edges added of each type
    \State{\bfseries Update the graph:} \\
    \quad Delete \emph{edges\_to\_del} edges from $G$ and add \emph{edges\_to\_add} edges to $G$ and save the new graph in $G_K$
    \State {\textbf{return} The modified graph $G_K$}
\end{algorithmic}
\end{algorithm}

In addition to Figure \ref{fig:planted_clique_main} in the main paper, here we give results for more parameter settings. Figures \ref{fig:planted_clique_SI2} and \ref{fig:planted_clique_SI} illustrate the power of the edge resampling version compared to the version without edge resampling of the IRG-gKSS test to test the fit of ER$(50,p)$ models using the WL kernel with $h=2$, for different $p$. In Figure \ref{fig:planted_clique_SI2}, ST and the GLR test are added for comparison; the behaviour is very similar to that in Figure \ref{fig:planted_clique_main} in the main text. 

These figures illustrate the power of the IRG-gKSS test to detect discrepancies from the null model in the presence of a clique in the network, with the minimum size of the clique detected by the test depending on the edge density of the network. As the expected degree of a vertex in an ER$(n,p)$ network is $(n-1)p$, we see that cliques can be detected even if they are only slightly larger than the expected degree. In our experiments, the GLR test does not detect the difference.
 
We repeat this experiment for ERMMs with parameters given in \eqref{eq:matQ} and \eqref{eq:nnumber}.
The results in Figure \ref{fig:planted_clique_ERMM} illustrate the power of the IRG-gKSS test to detect the presence of a clique in the network, depending on the size of the clique, the density of the network, and the type of split, as well as the GLR test and ST; these tests are provided with the true group memberships of the vertices. While for the networks generated with a single group and a planted clique, ST performs better than the IRG-gKSS test, for the networks simulated from an ERMM, either with equal or unbalanced numbers of vertices in two groups, the IRG-gKSS test performs better than ST. The GLR test struggles to detect the signal. 

In a case when ST needs to estimate the group memberships before testing the hypothesis of the number of groups in a network with a planted clique, it estimates the clique as a third group, which is a complete subgraph with unit within-group edge probabilities, and hence the standardisation in ST fails, resulting in NaN results. Hence, ST is not able to detect planted cliques when there is more than one group in the original network, unless it is provided with the true group memberships of the vertices. Therefore, in our experiments, ST is given the true group memberships.

\begin{figure}
         \centering
    \includegraphics[width=0.87\textwidth]{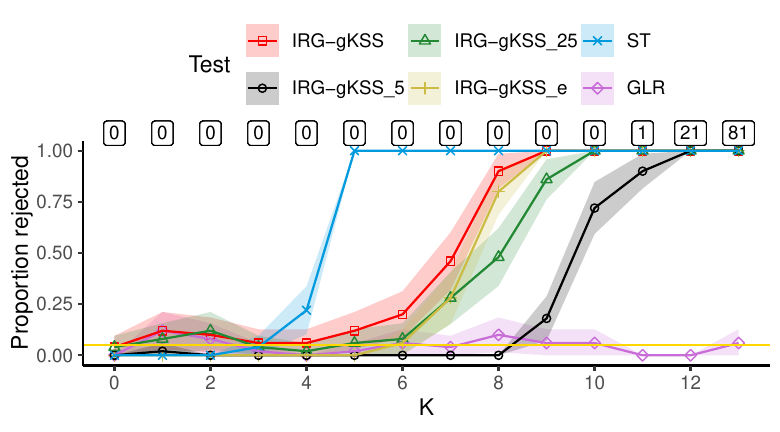} 
        \vspace{.1in}
         \caption{Power of the tests for the fit of an ER$(50, 0.06)$ to the network of size 50 with a planted clique of size $K$ using different proportions of edge resampling as well as the no edge resampling version of the IRG-gKSS test statistic. The numbers in the boxes are the total number of networks sampled from the ER$(50, 0.06)$ model that had fewer than $\binom{K}{2}$ edges and were, therefore, not used in the experiment. The number is the total of $m=50$ repetitions of the test. We use a WL kernel with $h = 2$.}
         \label{fig:planted_clique_SI2}
     \end{figure}
     
    \begin{figure}
         \centering
        \includegraphics[width=\textwidth]{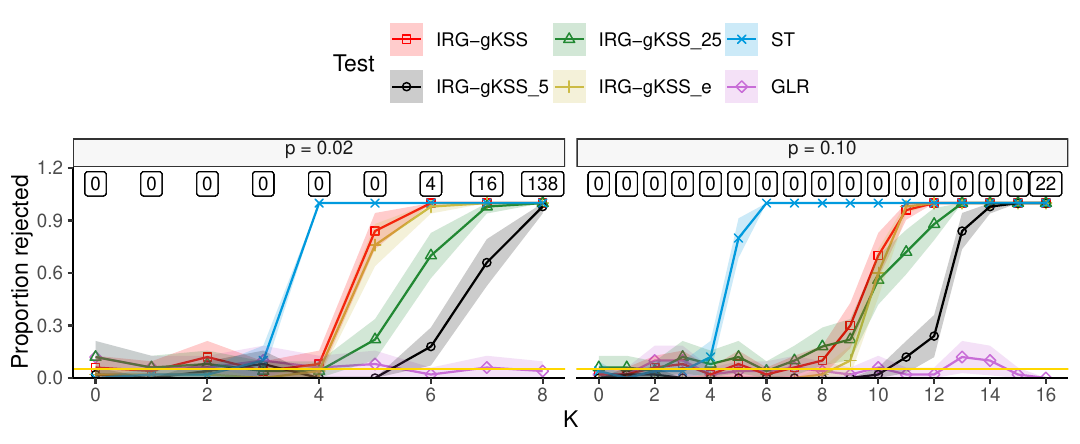} 
        \vspace{.1in}
         \caption{Power of the tests for the fit of an ER$(50,p)$ to the network of size 50 with a planted clique of size $K$ using different proportions of edge resampling as well as the no edge resampling version of the IRG-gKSS test statistic. The numbers in the boxes are the total number of networks sampled from the ER$(50,p)$ model that had the number of edges less than $\binom{K}{2}$ to plant a clique of size $K$ and were, therefore, not used in the experiment. The number is the total of $m = 50$ repetitions of the test.
         }
         \label{fig:planted_clique_SI}
     \end{figure}

    \begin{figure}
         \centering \includegraphics[width=\textwidth]{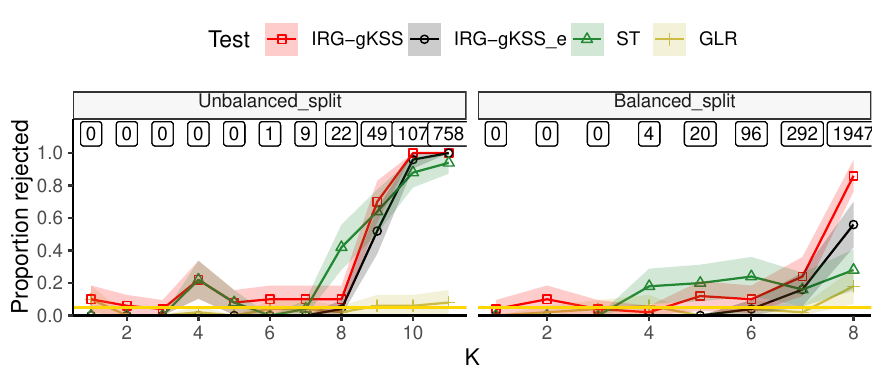} 
         \vspace{.1in}
         \caption{Power of the tests for the fit to the network with a planted clique; left: of an ERMM$((8,22), \mat{Q})$; right: of an ERMM$((15, 15), \mat{Q})$, with $\mat{Q}$ from \eqref{eq:matQ}. For this experiment, we use a WL kernel with $h = 3$.
         }
         \label{fig:planted_clique_ERMM}
     \end{figure}

\subsection{Kernel Choice} \label{app:kernel}

To calculate IRG-gKSS we use \eqref{t_stat_s_net}, in \eqref{t_stat_s_net}, $h_{\mat{x}}(s, s')$ is the inner product of the Stein operator \eqref{GD_Stein_Eq} for the vertex pairs $s$ and $s'$, calculated using the underlying kernel $K$ and inner product $\langle \cdot, \cdot\rangle$ of the associated RKHS. The test statistic is the average of $N^2$ values of $h_{\mat{x}}(s, s')$ calculated for all pairs of $s$ and $s'$. To calculate $h_{\mat{x}}(s, s')$, we create a list of $N+1$ networks, the observed network and others by flipping one edge each time and then calculate the graph kernel matrix for this list. Hence, the underlying graph kernel gives the type of signal from the networks which we use to calculate $h_{\mat{x}}(s, s')$ scores. In this section, we describe the graph kernels used in this paper and discuss their effect on the power of IRG-gKSS. 

\paragraph{Weisfeiler-Lehman kernels} For many of our experiments, we use a Weisfeiler-Lehman (WL) kernel, which is based on the Weisfeiler-Lehman graph isomorphism test. It assigns an attribute to each vertex. In our experiments,  we use group membership as an initial vertex attribute for WL kernels; for ER graphs, all vertices are labelled as group 1. Iteratively, for each vertex, the attributes of its immediate neighbours are aggregated to compute a new attribute for the target vertex. This relabelling is carried out for $h$ iterations. For each graph $G$, after $h$ iterations the algorithm has created $h$ related graphs $G=G_0, G_1, \ldots, G_h.$ For two graphs $G $ and $G'$, with sequences $G=G_0, G_1, \ldots, G_h$ and $G'=G'_0, G'_1, \ldots, G'_h,$   the WL  kernel is given by 
$$ K_{WL}(G,G') = \sum_{\ell=0}^h k(G_\ell, G'_\ell)$$
where $k(G_\ell, G'_\ell)$ is the ``subtree'' kernel which counts the common labels in the two graphs $G_\ell$ and $G'_\ell$; in this instance, $h$ is also referred to as the {\emph{(subtree) height}}. 
This way, the WL kernel compares the local structure in the two networks. Each step of the WL iterations incorporates information about the next order neighbourhood of the vertex into its new label. Since WL kernels compare the local structure in the two networks, they prove a reasonable choice to be used in IRG-gKSS for detecting structural anomalies in the network. 

The $K_{WL} $ values naturally increase with increasing number of iterations $h$. We observe empirically that the choice of $h$ affects the power of the IRG-gKSS test. As will be detailed in Section \ref{ssec:planted} and Section \ref{ssec:clique}, we find that the WL kernel with $h=2$ performs better in detecting the presence of cliques (an anomaly in the local structure), while the WL kernel with $h=3$ performs better in detecting hubs in the network (a more global structure).  

While a WL kernel does not directly depend on graph density, graph density affects the label refinement process, which in turn influences the kernel value. We observe anecdotally that in dense graphs, the neighbourhood structure can change substantially with each WL iteration. In sparse graphs, in contrast, the labels which the WL algorithm assigns often stabilise quickly, so that there will then be fewer changes in the WL refinement process. Thus, denser graphs often lead to higher feature complexity and better discrimination, while sparse graphs may suffer from label similarity, reducing kernel effectiveness.    

\paragraph{Graphlet kernels} The graphlet kernel is based on counts of subgraph patterns ({\emph graphlets}) of a fixed size in a graph. Motifs based on graphlets are building blocks of complex networks. Thus, they are a plausible choice for detecting structural anomalies. The time required to compute the graphlet kernel scales exponentially with the size of the considered graphlets and the edge density of the network, making them computationally more expensive than the WL kernel. Hence, we use a graphlet kernel to illustrate its effectiveness in some of the experiments and use the WL kernel for most of our experiments due to its computational efficiency.

\paragraph{Vertex-histogram kernels} The vertex-edge histogram kernel counts how many times each label (vertex and edge) appears in both networks and how many of those common labels are shared. It thus only signals the structural similarities of the vertex and edge labels of the networks compared, making it fast to compute. In our setting, however, we assume that the number of vertices and labels are known a priori. Thus, the only signal we obtain from using the vertex-edge histogram kernel is the density of edges of each type. 

To assess the effect of kernel choice in IRG-gKSS, we carry out the planted anomaly experiments using WL kernels with heights $h = 2, 3, 5,$ and $ 7$. We also use the graphlet kernel with graphlet size 3 (G3) and the Gaussian vertex-edge histogram kernel (VEH).

Figure \ref{fig:planted_hubs_diff_WL_h} depicts the results. For the planted hubs in ERMM networks with an unbalanced split $(8,22)$ into two groups, the IRG-gKSS tests computed using the graphlet kernel have the highest power. The IRG-gKSS tests computed using the WL kernel with $h=3$ have slightly higher power than the IRG-gKSS tests using the WL kernel with $h=2, 5$, and $7$. The IRG-gKSS tests computed using the VEH kernel have almost no power to detect model misspecifications. 

One could think of a hub as being a global structure, for which a larger $h$ in a WL kernel could be more appropriate. We also conclude that this example shows that there is no clear choice about which $h$ to employ; the choice of $h$ should be informed by the nature of the alternative hypothesis.

To further explore the effect of the choice of $h$ on the IRG-gKSS test, we simulate a network from an ERMM$(\mat{n}, \mat{Q})$ model, plant a hub using Algorithm \ref{alg:hubs} with $R=2$ and $k=3$, and test the fit of the ERMM$(\mat{n}=(15,15), \mat{Q})$ using the IRG-gKSS test. The original simulated network and one with the hubs are plotted in Figure \ref{fig:example_hubs_20}. The IRG-gKSS test rejects the fit when using the WL kernel with $h=3$ but does not reject it when using $h=2$. The heatmap plots of the matrix representing $h(s, s')$, for the original network simulated from the ERMM$(\mat{n}, \mat{Q})$ model and the corresponding network with planted hubs,  are given in Figure  \ref{fig:example_hubs_heatmap}. The signal is more pronounced for the WL kernel with $h=3$ than for $h=2$.

We repeat this experiment for the planted clique problem and use the IRG-gKSS test to test the fit of the ER$(30, 0.06)$ model on networks with a planted clique. Figure \ref{fig:planted_clique_diff_WL_h} shows the results. The IRG-gKSS test using the WL kernel with $h=2$ has the highest power to detect larger cliques in this experiment compared to the other graph kernels. For smaller cliques, the WL kernel with $h=2$ and the graphlet kernel with $k=3$ perform comparably. We further observe that, in this setting, increasing the height $h$ in the WL kernel leads the IRG-gKSS test to detect only larger cliques. Figure \ref{fig:execution_time_diff_kernels} shows that graphlet kernels are computationally more expensive than the WL kernel. Moreover, for the graphlet kernel, we do not observe a substantial improvement in power when using $k=5$, despite the significantly higher computational cost.

We then repeat the same experiment for the planted clique problem to zoom into the $h(s, s')$ level. For this experiment, we simulate a network from the ER$(30, 0.06)$ model and plant a clique of size $6$; an example is shown in Figure \ref{fig:example_clique}. The IRG-gKSS test rejects the fit of the ER$(30, 0.06)$ model for the network with a planted clique when using the WL kernel with $h=2$, but it does not reject the fit when using the WL kernel with $h = 3$. The heatmap plots for the $h(s, s')$ matrix in \eqref{t_stat_s_net}, in Figure \ref{fig:example_clique_heatmap}, illustrate how this discrepancy could arise. For $h=2$, the difference between the matrices is more pronounced, as this setting captures the direct neighbourhood of vertices and their immediate neighbours. Larger values of $h$, however, attenuate these differences by incorporating similarities in extended neighbourhoods. Notably, cliques and non-cliques become harder to distinguish at higher $h$ because their multi-hop neighbourhood structures become more similar, obscuring local distinctions.

    \begin{figure}
         \centering
         \includegraphics[width=\textwidth]{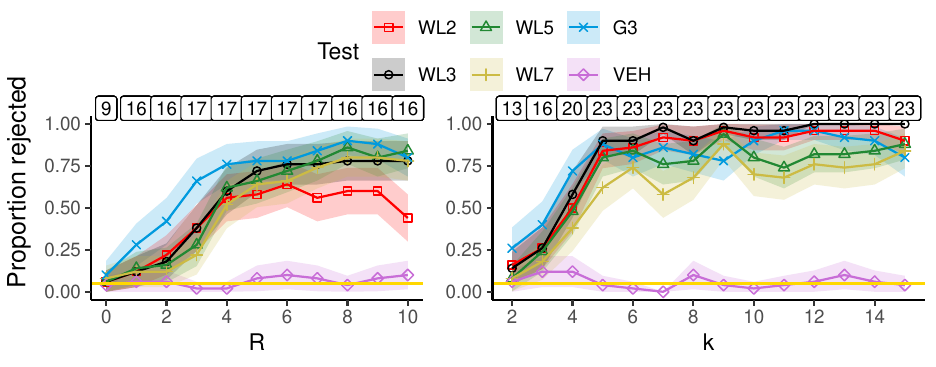} 
         \vspace{.1in}
         \caption{Power of the tests for the fit of an ERMM$(\mat{n}_{ub}, \mat{Q})$  model to the network of size 30 with planted hubs, with $\mat{n}_{ub}=(8,22)$ and $\mat{Q}$  from \eqref{eq:matQ}. The numbers in the boxes at the top of the plot are the average maximum degrees observed in $m = 50$ repetitions of the test for each setting on the x-axis. For the figure, we fix $k=$ and let $R$ vary. We compare the results using WL kernels with different subtree heights $h$, the graphlet kernel with graphlet size 3 and a vertex-edge histogram kernel (VEH).}
    \label{fig:planted_hubs_diff_WL_h}
    \end{figure}

    \begin{figure}
         \centering        \includegraphics[width=0.6\textwidth]{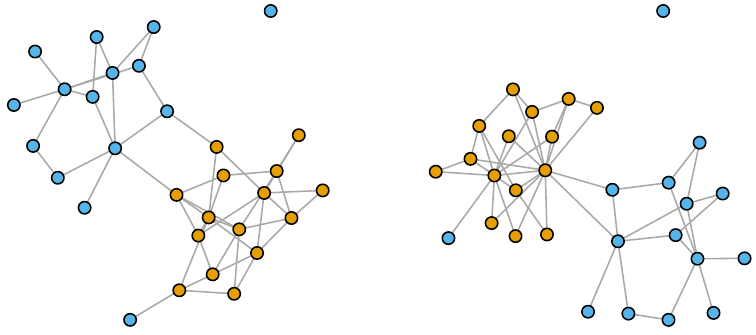} 
         \vspace{.1in}
         \caption{Simulated networks simulated from left: an ERMM$((15, 15), \mat{Q})$ model with $\mat{Q}$ as in \eqref{eq:matQ}; right: starting from the same network, the network with planted hubs using $R=2$ and $k=3$.
         }
         \label{fig:example_hubs_20}
    \end{figure}

    \begin{figure}[ht]
         \centering
        \includegraphics[width=\textwidth]{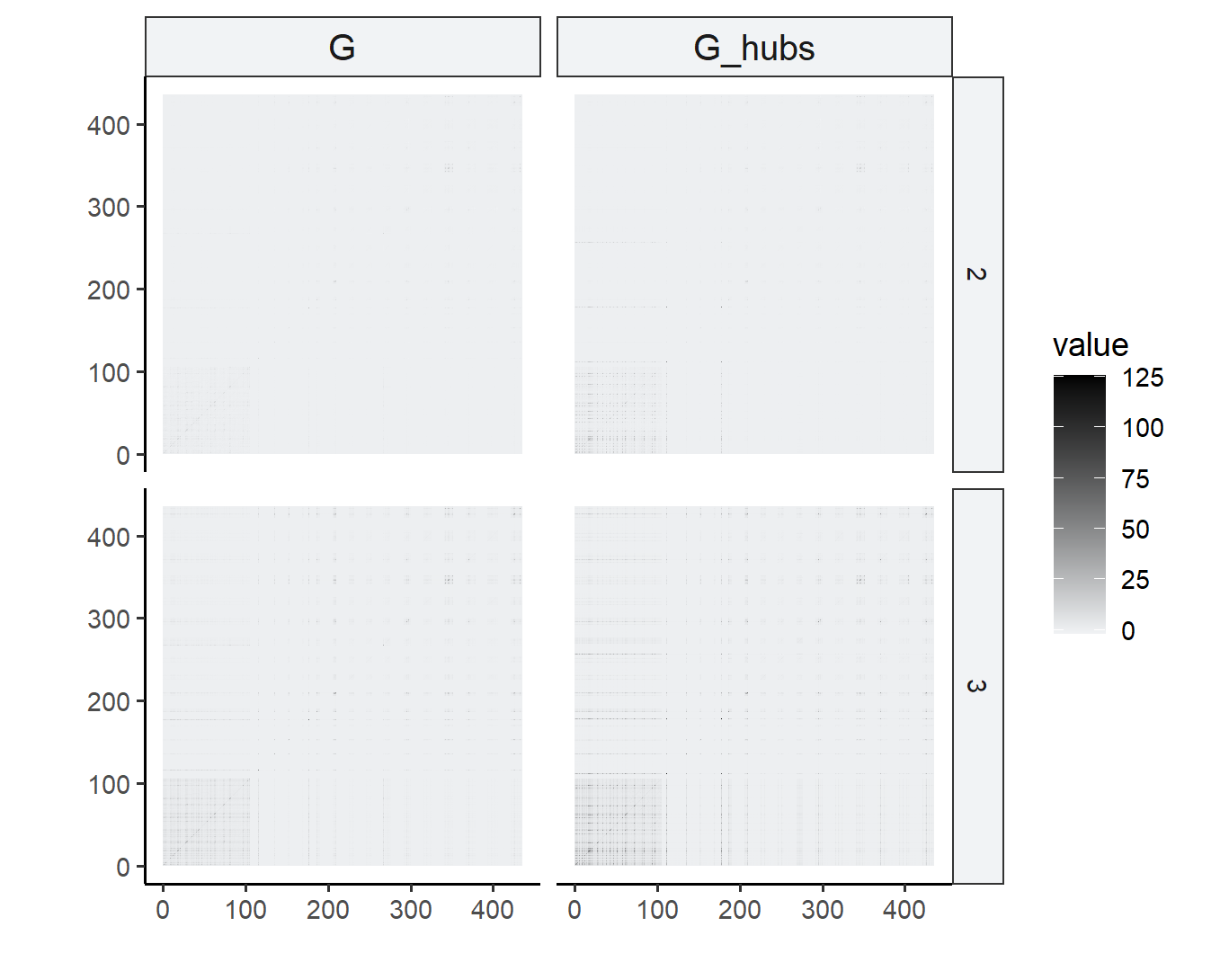} 
        \vspace{.1in}
         \caption{Heatmap plot of $h(s,s')$ matrix from left: original network simulated from an ERMM$((15,15), \mat{Q})$ model; right: the network with planted hubs using $R=2$ and $k= 3$. For this plot, we use WL with $h=2$ (top row) and $h=3$ (bottom row).}
         \label{fig:example_hubs_heatmap}
     \end{figure}

    \begin{figure}
         \centering
        \includegraphics[width=0.80\textwidth]{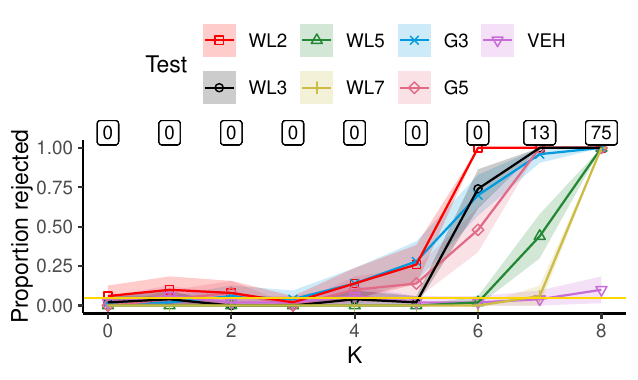} 
         \vspace{.1in}
         \caption{Power of the tests for the fit of an ER$(30, 0.06)$ model to a network with a planted clique of size $K$ using WL kernels with different subtree heights $h$, graphlet kernel with graphlet size 3 (G3) and 5 (G5), and a vertex-edge histogram kernel (VEH).}
         \label{fig:planted_clique_diff_WL_h}
    \end{figure}

    \begin{figure}
         \centering
    \includegraphics[width=0.8\textwidth]{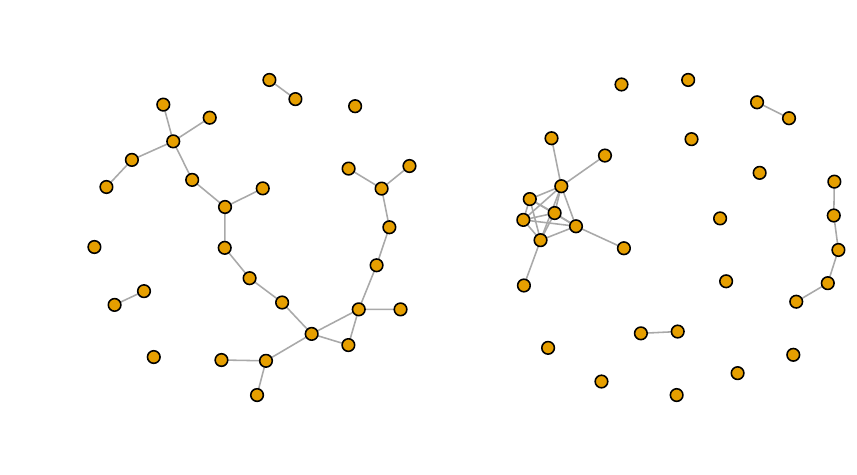} 
        \vspace{.1in}
         \caption{Simulated networks simulated from left: ER$(30, 0.06)$ model; right: starting from the same network, the network with a planted clique of size $K=6$.}
         \label{fig:example_clique}
     \end{figure}

     \begin{figure}
         \centering
    \includegraphics[width=\textwidth]{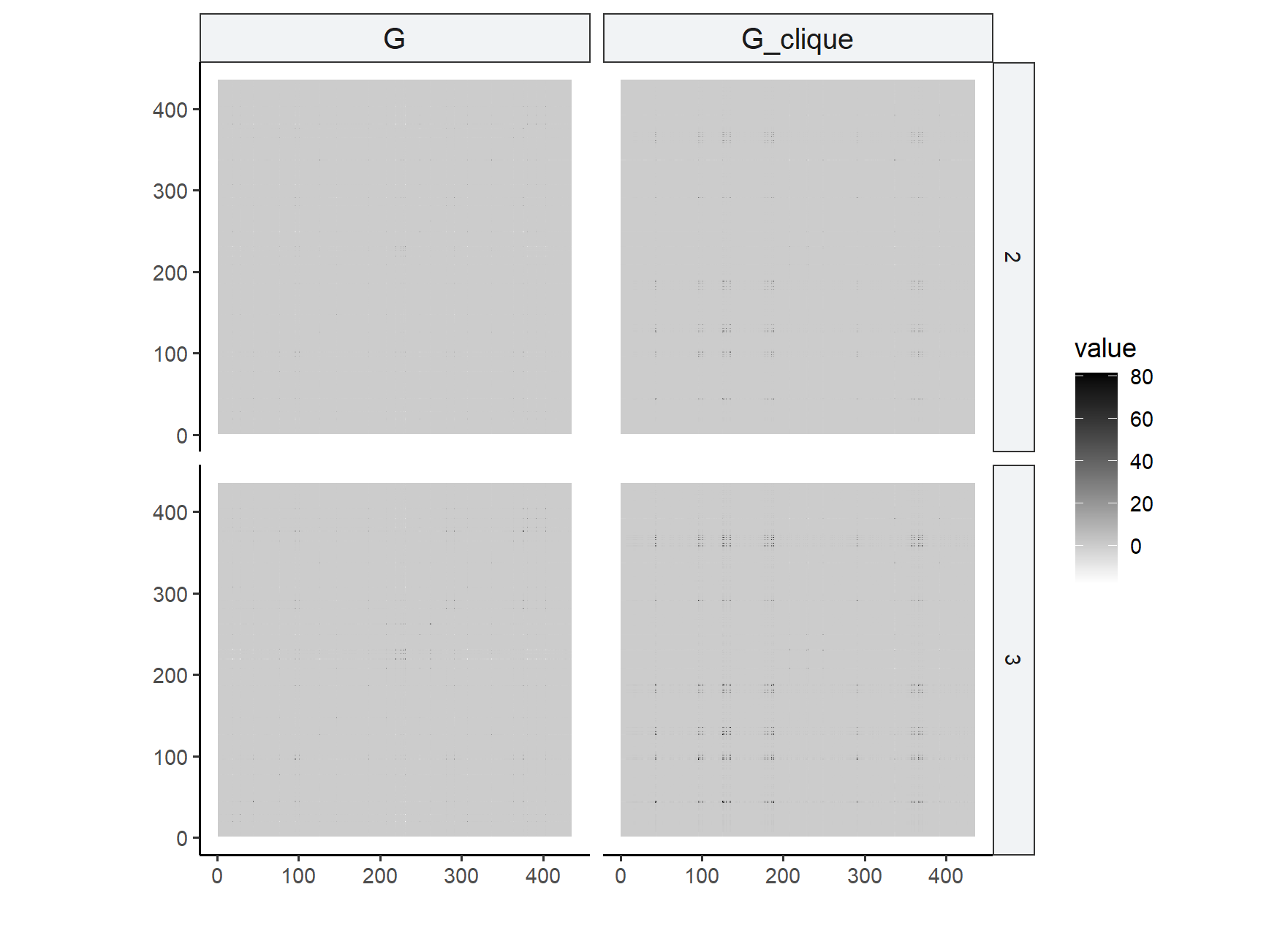} 
        \vspace{.1in}
         \caption{Heatmap plot of $h(s,s')$ matrix from left: original network simulated from an ER$(30, 0.06)$ model; right: the network with planted clique of size $K=6$. For this plot, we use WL with $h=2$ (top row) and $h=3$ (bottom row).}
         \label{fig:example_clique_heatmap}
     \end{figure}

Concluding this exploration, we did not find a uniformly best choice of the kernel to compute IRG-gKSS and recommend using more than one kernel when testing the fit of a network model. However, Figure \ref{fig:planted_clique_diff_WL_h} shows that the vertex-edge histogram does not perform very well. Hence, in general, we do not recommend the use of the vertex-edge histogram kernel to compute IRG-gKSS. 

\clearpage 

\section{REAL DATA APPLICATIONS} 
\label{sec:applic_contd}

\subsection{Parameter Estimation}

For the real data applications, the model parameters have to be estimated. For ERMM models with given group labels, the standard maximum likelihood estimators are used. When the group assignments are not given, but the number of groups is suggested, then we employ the Louvain algorithm by {\cite{blondel2008fast}} to obtain the group assignments. This algorithm assigns vertices to communities based on modularity as a quality measure of a partition. As this algorithm is random, we run it 100 times and select the partition with the highest modularity.

For the DCSBM, we estimate the parameters using the `DCSBM.estimate' function from the \texttt{randnet} package in {\tt{R}}. This package fits a Poissonized DCSBM to the data, as is customary. Here we use these parameter estimates to take $p_{u,v} = 1-e^{-Q_{g_u,g_v}\theta_u \theta_v}$ for all $u$ and $v$ in the network as parameters in the null model to test the fit of a DCSBM.

\subsection{Lazega's Lawyers Networks}\label{ssec:lazega}

\begin{figure}
         \centering
        \includegraphics[width=0.9\textwidth]{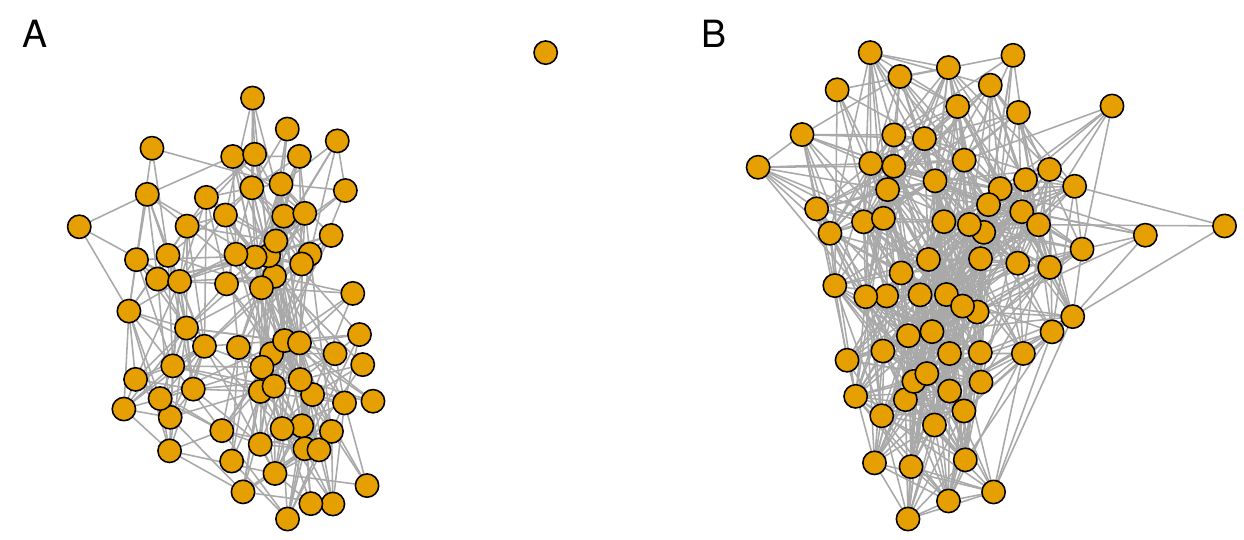} 
        \vspace{.1in}
         \caption{Lazega's lawyers networks: (A) work, (B) advice.
         }
         \label{ll_wa}
     \end{figure}

     \begin{figure}
         \centering
        \includegraphics[width=\textwidth]{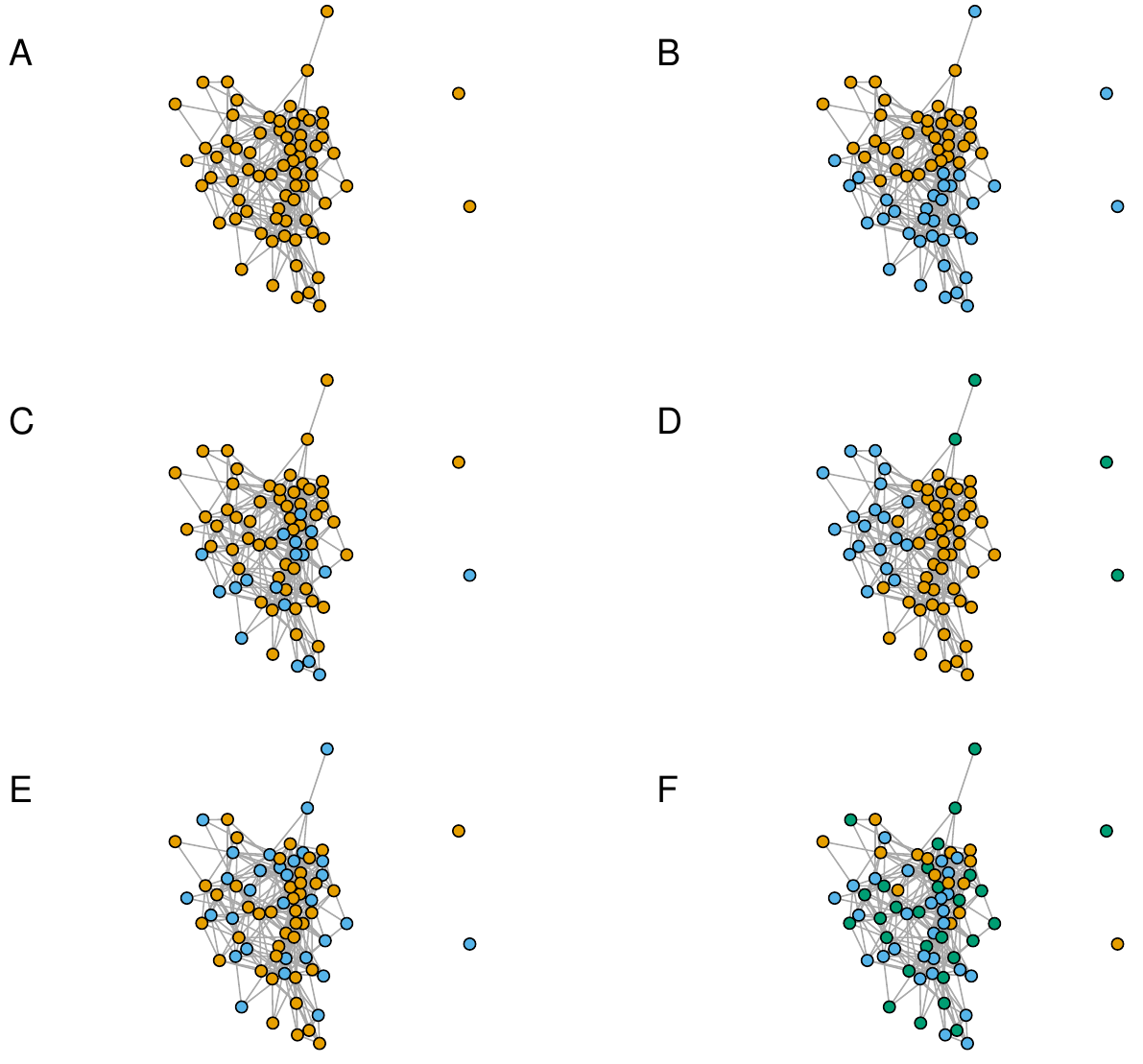} 
        \vspace{.1in}
         \caption{Lazega's lawyers' friendship network with vertices coloured according to (A) single group, (B) status (orange: partner; blue: associate), (C) gender (orange: male; blue: female), (D) office (orange: Boston; blue: Hartford; green: Providence), (E) practice (orange: litigation; blue: corporate), and (F) law school (orange: Harvard, Yale; blue: UCon; green: other).
         }
         \label{ll_f}
     \end{figure}

The collection of Lazega's lawyers' networks \citep{lazega2001collegial} is constructed from a data set collecting results of a study on relationships between 71 attorneys (partners and associates) of a Northeastern US corporate law firm, during 1988-1991.

The data set includes (among others) information on work relationships, advice relationships, and friendship relationships among the 71 attorneys. From this data set, several undirected networks are constructed, with lawyers represented by vertices. In the {\emph{work network}}, two vertices are connected by an edge if the two lawyers have spent time together on at least one case, or have been assigned to the same case, or either of them have read or used a work product created by the other. In the {\emph{advice network}}, an edge represents that they consult each other for basic professional advice for their work, and an edge in the {\emph{friendship network}} indicates that the two lawyers socialise with each other outside work.
 
Various attributes of the lawyers are also part of the dataset, such as formal status (2 groups), office in which they work (3 groups), gender (2 groups), type of practice (2 groups), and law school attended (3 groups). The ethnography, organisational, and network analyses of this data set are available in Lazega (2001).

For each of these networks, we create groups according to different attributes. The parameter estimates of these networks that we use as parameters of the null model, and the IRG-gKSS test results at $5\%$ level, are as follows.

\begin{enumerate}
    \item  For the work network, we propose an ER model with $n = 71$ and $p = 0.152$. For this model, the IRG-gKSS test does not reject the fit.

    \item For the advice network, we propose an ER model with $ n=71 $ and $p = 0.2885$. For this model, the IRG-gKSS test does not reject the fit.

    \item For the friendship network, the situation is more complex. First, we test the fit of an ER model with $n = 71$ and $p = 0.1605634$; IRG-gKSS rejects the fit of this model. Thus, there is a need for a more complex model. We therefore test the fit of different ERMMs and DCSBMs depending upon groups created according to different vertex attributes, as follows. 
\begin{enumerate}
\item {\bf Status.}
For groups created according to status, we test an ERMM with 
$$\mat{n} = (36, 35) \mbox{ and } \mat{Q} = \begin{pmatrix}
    0.2968 & 0.0730 \\
    0.0730 & 0.2017 \\
\end{pmatrix}.$$ 
For a corresponding DCSBM we use $$\mat{n} = (36,35), \quad \quad {\mat{Q} = \begin{pmatrix}
    374 & 92 \\
    92 & 240\\
\end{pmatrix}}$$ and 
\begin{align*}
\bm{\theta} = (&0.0172, 0.0215, 0.0086, 0.0472, 0.0129, 0.0043, 0.0064, 0.0150,  0.0300, 0.0300, \\ & 0.0300, 0.0472, 0.0515, 0.0193, 0.0086, 0.0279, 0.0536, 0.0193, 0.0129, 0.0300, \\ & 0.0322, 0.0236, 0.0150, 0.0536, 0.0322, 0.0515, 0.0408, 0.0279, 0.0279, 0.0172, \\ & 0.0601, 0.0150, 0.0279, 0.0258, 0.0279, 0.0279, 0.0030, 0.0452, 0.0452, 0.0392, \\ & 0.0602, 0.0482, 0.0542, 0.0000, 0.0241, 0.0181, 0.0000, 0.0211, 0.0181, 0.0271, \\ & 0.0241, 0.0422, 0.0090, 0.0361, 0.0090, 0.0361, 0.0452, 0.0422, 0.0181, 0.0241, \\ & 0.0090, 0.0151, 0.0060, 0.0633, 0.0693, 0.0482, 0.0151, 0.0271, 0.0181, 0.0211, \\ & 0.0181).
\end{align*}
\item {\bf Gender.}
For groups created according to gender, we test an ERMM with $$\mat{n} = (53, 18) \mbox{ and } \mat{Q} = \begin{pmatrix}
    0.1821 &  0.1184 \\
    0.1184 & 0.2287 \\
\end{pmatrix}.$$
For a corresponding DCSBM we use $$\mat{n} = (53, 18), \quad \quad 
\mat{Q} = \begin{pmatrix}
    502.001 & 113.001 \\
    113.001 & 70.001 \\
\end{pmatrix}$$ and 
\begin{align*}
\bm{\theta} = (&0.0130, 0.0163, 0.0065, 0.0358, 0.0098, 0.0033, 0.0049, 0.0114, 0.0228, 0.0228, \\ & 0.0228, 0.0358, 0.0390, 0.0146, 0.0065, 0.0211, 0.0407, 0.0146, 0.0098, 0.0228, \\ & 0.0244, 0.0179, 0.0114, 0.0407, 0.0244, 0.0390, 0.1038, 0.0211, 0.0710, 0.0130, \\ & 0.0455, 0.0114, 0.0211, 0.0656, 0.0211, 0.0211, 0.0016, 0.0820, 0.0820, 0.0211, \\ & 0.0325, 0.0260, 0.0984, 0.0000, 0.0130, 0.0328, 0.0000, 0.0383, 0.0098, 0.0146, \\ & 0.0437, 0.0228, 0.0049, 0.019,5 0.0049, 0.0195, 0.0820, 0.0228, 0.0328, 0.0437, \\ & 0.0164, 0.0081, 0.0033, 0.1148, 0.0374, 0.0260, 0.0273, 0.0146, 0.0328, 0.0114, \\ & 0.0328).
\end{align*}

\item {\bf Office.}
For groups created according to office, we test an ERMM with $$\mat{n} = (48, 19,  4) \mbox{ and } \mat{Q} = \begin{pmatrix}
    0.2464 & 0.0647 & 0.0156 \\
    0.0647 & 0.3391 & 0.0000 \\
    0.0156 & 0.0000 & 0.1666 \\
\end{pmatrix}.$$
For a corresponding DCSBM we use $$\mat{n} = (48, 19,  4), \quad 
\mat{Q} = \begin{pmatrix}
    556.001 & 59.001 & 3.001 \\
    59.001 & 116.001 & 0.001 \\
    3.001 & 0.001 & 2.001 \\
\end{pmatrix}$$ and 
\begin{align*}
\bm{\theta} = (&0.0129, 0.0162, 0.0229, 0.0356, 0.0343, 0.0114, 0.0171, 0.0113, 0.0227, 0.0227, \\ & 0.0227, 0.0356, 0.0388, 0.0514, 0.7995, 0.0210, 0.0405, 0.0514, 0.0097, 0.0227, \\ & 0.0243, 0.0178, 0.0113, 0.0405, 0.0857, 0.0388, 0.0307, 0.0743, 0.0210, 0.0457, \\ & 0.1600, 0.0400, 0.0743, 0.0194, 0.0743, 0.0210, 0.1999, 0.0243, 0.0243, 0.0210, \\ & 0.0324, 0.0259, 0.0291, 0.0000, 0.0129, 0.0343, 0.0000, 0.0113, 0.0097, 0.0514, \\ & 0.0457, 0.0227, 0.0049, 0.0194, 0.0049, 0.0194, 0.0243, 0.0800, 0.0343, 0.0129, \\ & 0.0049, 0.0081, 0.0114, 0.0340, 0.0372, 0.0259, 0.0081, 0.0146,  0.0097, 0.0113, \\ & 0.0097).
\end{align*}
\item {\bf {Practice.}}
For groups created according to practice, we test an ERMM with $$\mat{n} = (41, 30) \mbox{ and } \mat{Q} = \begin{pmatrix}
    0.2060 & 0.1268 \\
    0.1268 & 0.1701 \\
\end{pmatrix}.$$
For a corresponding DCSBM we use $$\mat{n} = (41, 30), 
 \quad \mat{Q} = \begin{pmatrix}
    338.001 & 156.001 \\
    156.001 & 148.001 \\
\end{pmatrix}$$ and 
\begin{align*}
\bm{\theta} = (&0.0162, 0.0329, 0.0081, 0.0724, 0.0121, 0.0040, 0.0099, 0.0142, 0.0461, 0.0461, \\ & 0.0283, 0.0724, 0.0486, 0.0296, 0.0132, 0.0428, 0.0822, 0.0182, 0.0197, 0.0283, \\ & 0.0304, 0.0223, 0.0142, 0.0506, 0.0493, 0.0486, 0.0385, 0.0428, 0.0428, 0.0162, \\ & 0.0567, 0.0142, 0.0263, 0.0395, 0.0428, 0.0263, 0.0033, 0.0304, 0.0304, 0.0263, \\ & 0.0405, 0.0526, 0.0364, 0.0000, 0.0263, 0.0197, 0.0000, 0.0230, 0.0121, 0.0296, \\ & 0.0162, 0.0283, 0.0099, 0.0243, 0.0061, 0.0243, 0.0304, 0.0283, 0.0121, 0.0263, \\ & 0.0099, 0.0164, 0.0066, 0.0691, 0.0466, 0.0324, 0.0101, 0.0182, 0.0121, 0.0230, \\ & 0.0121).
\end{align*}
\item {\bf Law school.}
For groups created according to which law school the person attended, we test an ERMM with $$\mat{n} = (15, 28, 28) \mbox{ and } \mat{Q} = \begin{pmatrix}
    0.2571 & 0.1476 & 0.1190 \\
    0.1476 & 0.2089 & 0.1696 \\
    0.1190 & 0.1696 & 0.1269 \\
\end{pmatrix}.$$
For a corresponding DCSBM we use $$\mat{n} = (15, 28, 28) , \quad 
\mat{Q} = \begin{pmatrix}
    54.001 & 62.001 & 50.001 \\
    62.001 & 158.001 & 133.001 \\
    50.001 & 133.001 & 96.001 \\
\end{pmatrix}$$
and 
\begin{align*}
\bm{\theta} = (&0.0482, 0.0602, 0.0241, 0.0789, 0.0170, 0.0120, 0.0108, 0.0251, 0.0843, 0.0502, \\ & 0.0843, 0.0623, 0.0680, 0.0542, 0.0143, 0.0783, 0.1506, 0.0255, 0.0361, 0.0843, \\ & 0.0425, 0.0394, 0.0198, 0.0708, 0.0425, 0.0860, 0.1145, 0.0368, 0.0466, 0.0287, \\ & 0.0793, 0.0251, 0.0466, 0.0340, 0.0466, 0.0466, 0.0036, 0.0425, 0.0904, 0.0783, \\ & 0.0567, 0.0453, 0.0510, 0.0000, 0.0287, 0.0170, 0.0000, 0.0251, 0.0170, 0.0255, \\ & 0.0287, 0.0502, 0.0108, 0.0430, 0.0108, 0.0340, 0.0425, 0.0502, 0.0170, 0.0227, \\ & 0.0108, 0.0142, 0.0057, 0.0595, 0.0824, 0.0573, 0.0142, 0.0323, 0.0215, 0.0198, \\ & 0.0170).
\end{align*}

\end{enumerate}
\end{enumerate}

\begin{table}
\caption{P-values for testing the fit of IRG models on Lazega lawyers' networks (using WL with $h=2$)} 
\label{tab:goflazega}
\centering
\begin{tabular}{lccccc}
&&&&&\\ 
 \textbf{NETWORK} &\textbf{ GROUP LABELS} & \textbf{ERMM} & \textbf{DCSBM}  \\
\hline \\
Work & Single group & 0.1692 & - \\
Advice & Single group & 0.5174 & - \\
Friendship & Single group & 0.0597 & 0.0299 \\
& Status & 0.0995 & 0.0299 \\
& Gender & 0.0299 & 0.0398 \\ 
& Office & 0.0995 & 0.00995 \\
& Practice & 0.0299 & 0.0299 \\
& Law school & 0.0498 & 0.00995 \\
\end{tabular}
\end{table}

\begin{table}
\caption{P-values 
for testing the fit of IRG models on Lazega lawyers' friendship network (using WL with $h=3$)} 
\label{tab:lazegah3}
\centering
\begin{tabular}{lcccc}
\textbf{GROUP LABELS} & \textbf{ERMM} & \textbf{DCSBM}  \\
\hline \\
 Gender & 0.00995 & 0.6467 \\ 
 Practice & 0.00995 & 0.6368 \\
 Law school & 0.01990 & 0.8557 
\end{tabular}
\end{table}

\noindent 
\begin{table}
\caption{P-values for testing the fit of ER and ERMM models on Lazega lawyers' networks using ST
} 
\label{tab:goflazega_ST}
\centering
\begin{tabular}{lccccc}
&&&&&\\ 
\textbf{NULL} & \textbf{WORK} & 
\textbf{ADVICE} & 
\textbf{FRIENDSHIP} \\
\hline \\
$K_0 = 1$ & 0.00 & 0.00 & 0.00 \\
$K_0 = 2$ (Status) & 0.00 & 0.00 & 0.00 \\
$K_0 = 2$ (Gender) & 0.00 & 0.00 & 0.00  \\ 
$K_0 = 3$ (Office) & 0.00 & 0.00 & 0.00 \\
$K_0 = 2$ (Practice) & 0.00 & 0.00 & 0.00 \\
$K_0 = 3$ (Law school) & 0.00 & 0.00 & 0.00 \\
\end{tabular}
\end{table}

As an aside, from Table \ref{dis_boundsold} we see that for an IRG model in which communities are created using the office where the lawyers work,  the discrepancy between the ER and the corresponding ERMM is higher than the discrepancy between ERMM and a corresponding DCSBM.

\begin{table}[ht]
\caption{GLR tests between IRG models for Lazega's lawyers' friendship network, with the 5\% critical value from the approximating chi-square distribution in parentheses. The degrees of freedom are $k=2$ for Status, Gender, and Practice, and $k = 5$ for Office and Law school, for the tests of ER against ERMM. For the tests of ERMM against DCSBM, the degrees of freedom are $k=69$ for Status, Gender, and Practice, and $k=68$ for Office and Law school.} 
\label{dis_bounds}
\centering
\begin{tabular}{lcc}
&&\\ 
& \multicolumn{2}
{@{}c@{}@{}}
{\textbf{GLR TEST STATISTIC AND CRITICAL $\chi^2_k(\alpha)$}}  \\
 \textbf{GROUP LABELS}  & \textbf{ER AGAINST ERMM} & \textbf{ERMM AGAINST DCSBM}  \\
\hline\\
 Status & 333.2 (5.991) & 685.4 (89.39) \\ 
 Gender & 45.94 (5.991) & 710.7 (89.39)\\ 
 Office & 474.9 (11.07) & 550.2 (88.25) \\
 Practice & 46.21 (5.991) & 704.5 (89.39) \\
 Law school & 45.30 (11.07) & 696.60 (88.25) \\
\end{tabular}
\end{table}

In a related experiment, reported in Table \ref{dis_bounds}, we used a likelihood ratio test for the Status network to test the ERMM model against the ER model; the test rejected the ER model in favour of the ERMM model. Then we tested the DCSBM model against the ERMM model; the test rejected the ERMM model in favour of the DCSBM model. This experiment makes it clear that, in contrast to the IRG-gKSS test, the likelihood ratio test is not a goodness-of-fit test against the general alternative; rather, the results depend on the specified alternative model. Table \ref{tab:goflazega_ST} reports the P-values of the fit of a stochastic blockmodel (SBM) with $K_0$ groups for ST. ST rejects the fit of a one-group SBM for all three work-advice and friendship networks. It also rejects two-group and three-group SBMs when group membership is defined by the observed vertex attributes: status, gender, office, practice, and law school.

\subsection{Zachary's Karate Club Network}
\label{ssec:karate}
\begin{figure}
        \centering
    \includegraphics[width=0.8\textwidth]{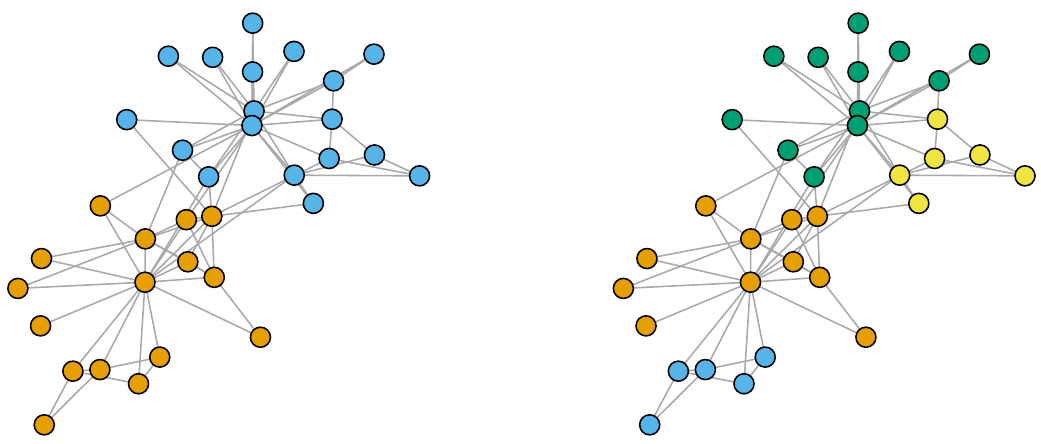}  
        \vspace{.1in}
         \caption{Zachary's Karate club network (left) with the observed split into two groups and (right) the network with $4$ groups.}
         \label{fig:karate}
     \end{figure}

Here we give further details about Zachary's Karate club network example from Section \ref{kclub_main} of the main text, with $n=34$ vertices.

For ER and ERMM models, we use the maximum likelihood estimates from the observed network as parameters in the null model, namely $p = 0.139$ for the ER model; for the ERMM model, we take two groups reflecting the split, with sizes $\mat{n}_1 = (16, 18)$, and parameter estimates $$\mat{Q}_1 = \begin{pmatrix}
    0.2750 &  0.0347 \\
     0.0347 & 0.2288
\end{pmatrix}.$$

We also test an ERMM model with four communities as suggested in \cite{karwa2024monte}. Figure \ref{fig:karate} shows the groups for both the observed split and an output from the Louvain algorithm with the highest modularity. 

The maximum likelihood estimates of the parameters of an ERMM model with four groups are $\mat{n_2} = (11, 5, 12, 6)$ and 
$$\mat{Q}_2 = \begin{pmatrix}
    0.4182 & 0.0727 & 0.0530 & 0.0455 \\
    0.0727 & 0.6000 & 0.0000 & 0.0000\\
    0.0530 & 0.0000 & 0.3182 & 0.0972\\
    0.0455 & 0.0000 & 0.0972 & 0.4667\\
\end{pmatrix},$$
which we use as parameters in the null model for testing the fit of an ERMM with four groups.

For the DCSBM, we estimate the parameters `DCSBM.estimate' function from the \texttt{randnet} package in {\tt{R}}. The estimates are $\mat{n} = (16,18)$, $\mat{Q} = \begin{pmatrix}
    66.001 & 10.001 \\
    10.001 & 70.001 \\
\end{pmatrix}$ and 
\begin{align*}
\bm{\theta} = (&0.2105, 0.1184, 0.1316, 0.0789, 0.0395, 0.0526, 0.0526, 0.0526, 0.0625, 0.0250, \\ & 0.0395, 0.0132, 0.0263, 0.0658, 0.0250, 0.0250, 0.0263, 0.0263, 0.0250, 0.0395, \\ & 0.0250,  0.0263, 0.0250, 0.0625, 0.0375, 0.0375, 0.0250, 0.0500, 0.0375, 0.0500, \\ & 0.0500, 0.0750, 0.1500, 0.2125).
\end{align*}
We use $p_{u,v} = 1-e^{-Q_{g_u,g_v}\theta_u \theta_v}$ for all $u$ and $v$ in the network as null parameters to test the fit of a DCSBM. 

\begin{table}
  \caption{Testing the fit of some IRG models on Zachary's Karate club network (using WL with $h=2$)}
  \label{tab:karate}
  \centering
  \begin{tabular}{lcccc}
        & \multicolumn{4}{@{}c@{}@{}}{\textbf{MODEL}} \\
        & \textbf{ER$(n,p)$} & \textbf{ERMM$(\mat{n}_1, \mat{Q}_1)$} & \textbf{ERMM$(\mat{n}_2, \mat{Q}_2)$ }& \textbf{DCSBM$(\mat{n}, \mat{Q}, {\bm{\theta}})$} \\
       \hline\\
    P-value & 0.02985 & 0.00995 & 0.00995 & 0.00995 \\
  \end{tabular}
\end{table}

\subsection{Padgett's Florentine Marriage Network}

As another real-world example, we consider the subgraph of the Florentine marriage data from \citet{padget} constructed in \citet{BREIGER1986215}. It consists of 20 marriage ties (edges) between 16 families (vertices) in Florence, Italy. The data set contains two more variables for the families in the network: net wealth in 1427, and the number of seats on the city council from 1288–1344. 
 
We test the fit of the Erd\H{o}s-R\'enyi (ER) random graph model with $n=16$ and $p_0 = 20/120 =1/6$ on this network. The results in Table \ref{table:all_net} show that the IRG-gKSS test does not reject the fit of an ER$(16,1/6)$ model to the Florentine marriage network. This finding agrees with the finding in \cite{xu2021stein}. ST also does not reject the null hypothesis at the 5\% level. As the GLR test requires specifying an alternative, we did not employ it here.
\begin{figure}
         \centering
    \includegraphics[width=0.8\textwidth]{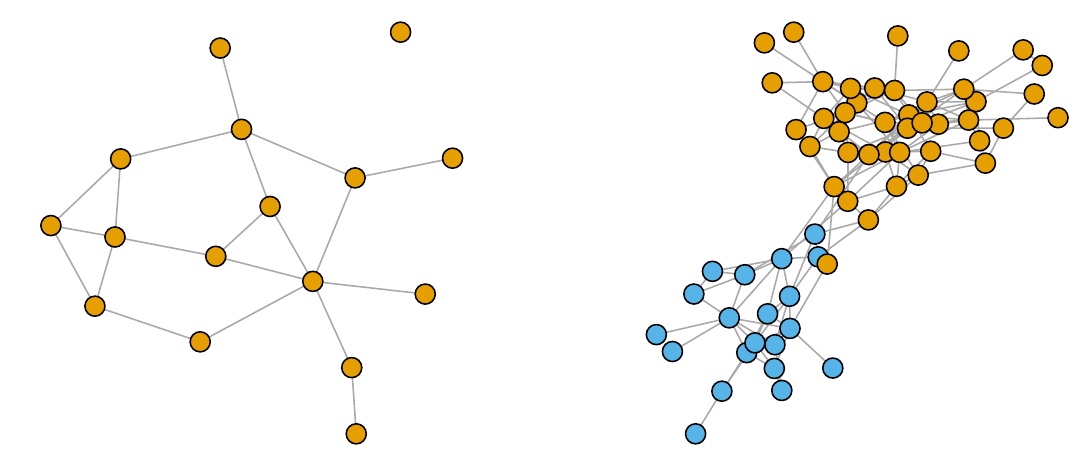} 
        \vspace{.1in}
         \caption{Left: Padgett's Florentine marriage network; right: Lusseau's dolphin network.
         }
         \label{flor_dolph}
     \end{figure}
     
\subsection{Lusseau's Dolphin Network}

As a final real-world example, we consider the social network of bottlenose dolphins in Doubtful Sound, New Zealand, from \citet{lusseau2003bottlenose}, constructed from observations of a community of 62 bottlenose dolphins over seven years from 1994 to 2001. The vertices in the network represent the dolphins, and the undirected edges between the vertices represent associations between a pair of dolphins occurring more often than expected by chance.

This network is one of the benchmark networks in social network analysis. \citet{Lusseau_David_2003} argues that the connectivity of individuals follows a complex distribution that has a scale-free power-law distribution for large degrees $k$. \citet{jin2013bayesian} and \citet{CAIMO201141} found that the dolphin network is inhomogeneous; a few vertices have higher degrees while others have only one or two edges. They analysed this network using an ERGM with the degree and shared partnership statistics. \citet{ouyang_dey_zhang_2023} fits a mixed membership model to the dolphin network. None of these papers assessed the fit of the suggested model using a statistical test procedure.

Here we test the fit of IRG models to the dolphin network. The results of IRG-gKSS, given in Table \ref{table:all_net}, show that we do not reject the fit of an ER$(62,0.0841)$ for the dolphin network. In contrast, ST clearly rejects this model. An ER network model would be a much simpler explanation for the data than those previously suggested. Investigating this discrepancy further will be part of future work. 

\begin{table}
  \caption{P-values for testing the fit to an ER model for the Florentine marriage network and the dolphin network}
  \label{table:all_net}
  \centering
  \begin{tabular}{llcccc}\\
    \textbf{NETWORK}     & \textbf{MODEL}     & \textbf{WL WITH $h = 2$} & \textbf{WL WITH $h = 3$} & {\textbf{ST}} \\
    \hline \\
    Florentine Marriage & ER$(16, 1/6)$ & 0.8557 & 0.9651 & {0.0755} \\
    Dolphin & ER$(62, 0.0841)$ & 0.7860 & 0.1791 & 0.0000 \\
        & ERMM$(\mat{n},\mat{Q})$ & 0.5572 & 0.6667 & 0.1201 \\
  \end{tabular}
\end{table}

Table \ref{table:all_net} shows that for the Florentine marriage network and for the dolphin network, the number of iterations (or height) of the WL kernel does not affect the conclusion of the test. It does, however, affect the $P$-value.

\subsection{Discussion of the Results on Real-World Networks}

For real-world networks, we rely on models that were proposed in the literature. For some networks, such as Zachary's karate club network, there is no consensus regarding the best model.  Domain knowledge about the networks opens up a discussion about network representations. Zachary's karate club network, for example, is a binary friendship network; friendship is determined by social interactions outside the club, which could happen via different social groups. In his original paper, however, Zachary attaches weights to edges which represent ``capacities", or different strengths, of relationships. Moreover, the observations were taken not at a single time point but between 1970 - 72 and then aggregated into a single binary network. One could argue that a temporal weighted network model could be more appropriate for the data; friendships change over time. However, the raw data are not available, and hence, this approach cannot be carried out. Similarly, Lazega's lawyer networks could perhaps be more appropriately represented as a multilayer network, because the different networks may influence each other. Lusseau's dolphin network again is based on temporal observations, but the time stamps are not available. Hence, for all these real-world networks, one could argue that the binary network representation may not be appropriate.

Asymptotic tests such as the ones suggested by \cite{dan2020goodness} vary depending on the asymptotic sparsity of the network. For a fixed small or moderately sized network, such as Zachary's karate club network with 34 vertices and 78 edges, it is impossible to say which sparsity regime it belongs to. Different sparsity regimes lead to different asymptotic distributions and hence different test results. Similarly, the tests in \cite{Jin2025} are different for different asymptotic regimes, and the asymptotic regime cannot plausibly be determined for a fixed small or moderately sized network. Hence, in this paper, we do not compare to such asymptotic tests.

\section{STEIN'S METHOD FOR IRG MODELS} \label{app:stein}

Here we underpin the proposed IRG Stein operator given in \eqref{SE_ermm_op}, with special emphasis for its use in a KSD-type test. Recall that, for an IRG$(\mat{p})$ as in Subsection \ref{sec:irg}, with adjacency matrix $\mat{X}$, and  $p_s = \mathbb{P}(s=(u,v) \in \mathcal{E}| g_u, g_v )$, {the probability that $\mat{X} = \mat{x}$ is given by
$$\mu(\mat{x}) = \prod_{s \in E} p_s^{x_s}(1-p_s^{x_s}), \quad \quad \mat{x} \in \Omega$$
where $\Omega = \{0,1\}^N$, $p_s \in [0,1]$, and $E = \{(i,j): 1 \le i \le j \le n \}$ the set of all vertex pairs so that $|E| = N = \binom{n}{2}$. Under this model, for $s \in E$,} the conditional probability to have $x_s = 1$, given the rest of the graph, $\mat{x}_{-s}$, is given by 
\begin{equation} \label{trans_ERMM}
    q(\mat{x}^{(s,1)} | \mat{x}_{-s}) = p_s.
\end{equation}

To construct a Stein operator for the IRG models, we use the insight from \citet{barbour1988stein} and \citet{gotze1991rate}, that if $\mathcal{L}_0$ is the stationary distribution of a homogeneous Markov process, then under some regularity assumptions the generator of the Markov process is a Stein operator for $\mathcal{L}_0$; the (infinitesimal) generator of a homogeneous Markov process $\{X(t)\}_{t \ge  0}$ is the operator $\mathcal{A}$ operating on smooth functions $f$ such that 
$ {\mathcal{A}} f(x) =  \lim_{h \rightarrow 0} \{ \EE [ f(X(h))| X(0) = x] - f(x) \}/h$, if the limit exists. 

We define a Markov process $\mathcal{X} = \{X(t)\}_{t \ge 0}$ with state space $\{0,1\}^{N}$ as follows. Given the current state $\mat{x}$, each vertex pair in $E$ has a clock which rings at i.i.d.\,exponentially distributed times with mean $N$, independently of the network. When the clock rings for a vertex pair $s \in E$ we resample the edge indicator $x_s$, setting an edge at $s$ with probability $p_s$ and not setting an edge at $s$ with probability $1-p_s$. We then resample the edge indicator for the next vertex pair for which their clock rings and continue this process. This construction, called {\it Glauber dynamics}, gives a homogeneous Markov chain $\{{\bf X}(t)\}_{t \ge  0}$ in continuous time, with transition rates
$$ q_{\mat{x}, \mat{x}^{(s,1)}}
= \lim_{h \rightarrow 0} \frac1h \PP( X (t+h) = \mat{x}^{(s,1)} | X(t) = \mat{x}) = \frac1N p_s, \qquad \qquad  q_{\mat{x}, \mat{x}^{(s,0)}} = \frac1N (1-p_s).$$

For all $f : \{ 0, 1\}^N \rightarrow \R$, the generator of the Markov process is 
\begin{equation} \label{eq:IRG_gen}
 \AAA_{\IRG} f(\mat{x}) = \frac{1}{N} \sum_{s \in E} \left[ p_s \left( f(\mat{x}^{(s,1)}) - f(\mat{x})\right) + \left(1-p_s \right) \left( f(\mat{x}^{(s,0)}) - f(\mat{x}) \right) \right].
\end{equation}

Note that for the Markov process with generator \eqref{eq:IRG_gen}, there is no interaction between different vertex pairs,  and the updates on all sites are independent. We can write this operator as the sum of individual operators for each site $s \in E$, 
\begin{equation*}
    \AAA_{\IRG} f(\mat{x}) = \frac{1}{N} \sum_{s \in E} \AAA_{\IRG}^{(s)} f(\mat{x}),
\end{equation*}
where 
\begin{equation*}
    \AAA_{\IRG}^{(s)} f(\mat{x}) = p_s \left( f(\mat{x}^{(s,1)}) - f(\mat{x})\right) + \left(1-p_s \right) \left( f(\mat{x}^{(s,0)}) - f(\mat{x}) \right).
\end{equation*}

\begin{proposition}
\label{prop:unique_stationary} 
The Markov process $\{X(t)\}_{t \ge 0}$ with generator $\AAA_{\IRG}$ given in \eqref{eq:IRG_gen} has a stationary distribution given by
  $$\mu(\mathbf{x}) = \prod_{s \in E} p_s^{x_s} (1-p_s)^{1-x_s},
  \qquad \mathbf{x} \in \{0,1\}^E.$$
\end{proposition}

\begin{proof}
Using \eqref{GD_Stein_Eq}, for each $s \in E,$
\begin{align*}
  &\E\AAA_{\IRG}^{(s)} f(\mat{X}) 
    =   \sum_{\mat{x}} \left[\mathbb{P} (\mat{X} = \mat{x}^{(s,0)}) \AAA_{\IRG}^{(s)} f(\mat{x}^{(s,0)}) + \mathbb{P} (\mat{X} = \mat{x}^{(s,1)})\AAA_{\IRG}^{(s)} f(\mat{x}^{(s,1)})\right] \\
    = &\sum_{\mat{x}}
    \mathbb{P}(\mat{X}_{-s} =\mat{x}_{-s}) \left[ (1 - p_s) \left\{p_s \left( f(\mat{x}^{(s,1)}) - f(\mat{x}^{(s,0)})\right) + \left(1-p_s \right) \left( f(\mat{x}^{(s,0)}) - f(\mat{x}^{(s,0)}) \right)
     \right\} \right.\\
     & \left.  + p_s \left\{p_s \left( f(\mat{x}^{(s,1)}) - f(\mat{x}^{(s,1)})\right) + \left(1-p_s \right) \left( f(\mat{x}^{(s,0)}) - f(\mat{x}^{(s,1)}) \right) \right\}\right] \\
    =& 0.
\end{align*}
Thus, $\E\AAA_{\IRG}^{(s)} f(\mat{X})  = 0,$ and with \eqref{SE_ermm_op} $\E\AAA_{\IRG} f(\mat{X})  = 0$. Using Theorem 3.3.7 in \cite{liggett2010continuous} it follows that $\mu$ is a stationary distribution for the process $\{X(t)\}_{t \ge 0}$ with generator \eqref{eq:IRG_gen}. For more details, we refer the readers to Chapter 4 of \cite{liggett2010continuous}.
\end{proof}

As $\E\AAA_{\IRG} f(\mat{X})  = 0$ for $\mat{X} \sim IRG(\mat{p})$, we use \eqref{eq:IRG_gen} as a Stein operator for the IRG$(\mat{p})$ model given in \eqref{eq:irg} for $f$ in the Stein class $\mathcal{F}(\AAA) = \{ f: \{0,1\}^{N} \rightarrow \R\}$. We next illustrate its use for comparing different IRG models. To this purpose, we start with a so-called {\it Stein equation}. For this IRG Stein operator the {Stein equation} for a test  function $h:\{0 , 1\}^N \rightarrow \mathbb{R}$ is 
\begin{equation} \label{SE_irg}
    \frac{1}{N} \sum_{s \in E} \left[ p_s \Delta_s f(\mat{x}) + \left( f(\mat{x}^{(s,0)}) - f(\mat{x}) \right) \right] = h(\mat{x}) - \E h(\mat{X}).
\end{equation}
For a given $h$ it is often possible to find a solution $f=f_h$ of \eqref{SE_irg}. 
Then, for any random $\mat{W} \in \{0 , 1\}^N$, replacing $\mat{x}$ by $\mat{W}$ and taking expectations gives
\begin{equation*} 
  \E h(\mat{W}) - \E h(\mat{X}) =  \frac{1}{N} \sum_{s \in E} \left[ p_s \E \Delta_s f(\mat{W}) + \E \left( f(\mat{W}^{(s,0)}) - f(\mat{W}) \right) \right] .
\end{equation*}
If the distribution of $\mat{W}$ also has a Stein operator of the form \eqref{SE_ermm_op}, with parameters $q_s$, then
\begin{equation} \label{eq:opcomp} 
  \E h(\mat{W}) - \E h(\mat{X}) =  \frac{1}{N} \sum_{s \in E}  ( q_s - p_s )\E \Delta_s f(\mat{W}).
\end{equation}
Thus, bounds on  $\Delta f_h$, for $f=f_h$ solving \eqref{SE_irg} could be used to bound the difference in expectations $ \E h(\mat{W}) - \E h(\mat{X})$. From \citet{reinert2019approximating} the  Stein equation  \eqref{SE_irg} is solved by 
\begin{equation} \label{sol_gen}
    f_h(\mat{x}) := - \int_0^{\infty} \E \left[ h(\mat{X}(t)) - \E h(\mat{X})| \mat{X}(0) = \mat{x}\right] dt
\end{equation}
where $\{{\bf X} (t)\}_{t \ge  0}$ is the above  Markov process. While \cite{reinert2019approximating} provide a general framework of which IRG models can be seen as special cases, some of their results considerably simplify when, instead of a general exponential random graph model, an IRG model is used. As an instance, the next lemma gives the desired bounds on $f_h$ in \eqref{sol_gen}.

\medskip 
\begin{lemma}\label{lem:steinsol}
For $f_h(\mat{x})$ in \eqref{sol_gen}, the solution of the Stein equation \eqref{SE_irg} in \eqref{sol_gen}, with $\mat{X}$ following an IRG model, we have 
\begin{equation} \label{sol_SE_ERM}
         |\Delta_s f_h (\mat{x})| \le \| \Delta_s h \| N.
    \end{equation}
\end{lemma}
\begin{proof}
Let $h : \{0,1\}^{N} \rightarrow \mathbb{R}$ and $f_h$ be given in \eqref{sol_gen}. Suppose that  $\left(U^{[\mat{x},s]}(m) , V^{[\mat{x},s]}(m)\right)$ is a coupling such that for $m\ge0$, $\mathcal{L}\left(U^{[\mat{x},s]}(m)\right) = \mathcal{L} \left(\mat{X}(m)|\mat{X}(0) = \mat{x}^{(s,1)}\right) $ and $\mathcal{L}\left(V^{[\mat{x},s]}(m)\right) = \mathcal{L} \left(\mat{X}(m)|\mat{X}(0) = \mat{x}^{(s,0)}\right)$. Lemma 2.5 from \citet{reinert2019approximating}  gives
\begin{align*}
    |\Delta_s f_h (\mat{x})| \le \sum_{r \in E, m \ge 0} \|\Delta_r h\| \mathbb{P} \left(U_r^{[\mat{x},s]}(m) \neq V_r^{[\mat{x},s]}(m)\right).
\end{align*}
Since in an IRG, the edge indicators are independent, for $r \neq s$, we can take $U^{[\mat{x},s]}(m) = V^{[\mat{x},s]}(m)$, and thus 
\begin{equation} \label{sum_in_bound}
    |\Delta_s f_h (\mat{x})| \le \|\Delta_s h\| \sum_{m \ge 0}  \mathbb{P} \left(U_s^{[\mat{x},s]}(m) \neq V_s^{[\mat{x},s]}(m)\right).
\end{equation}
Now, let $T$ be the first time a clock rings in the  Markov process $\{{\bf X}(t)\}_{t \ge  0}$. Then $U^{[\mat{x},s]}(T) = V^{[\mat{x},s]}(T)$ and the two processes can be coupled to coincide after this time $T$; in detail,  $U^{[\mat{x},s]}(m) = V^{[\mat{x},s]}(m)$ for $m \ge T$, while $U^{[\mat{x},s]}(m) \neq V^{[\mat{x},s]}(m)$ for $m<T$. Thus, $ \mathbb{P}(U^{[\mat{x},s]}(m) \neq V^{[\mat{x},s]}(m)) = \mathbb{P} \left( T > m\right) $. As  $T \sim \text{Geometric}(\frac{1}{N})$ with $t=1, 2, 3, ...$, 
\begin{align*}
    \mathbb{P} \left(U_s^{[\mat{x},s]}(m) \neq V_s^{[\mat{x},s]}(m)\right) &=  \mathbb{P} \left( T > m\right) = \sum_{t=m+1}^{\infty} \left(1-\frac{1}{N}\right)^{t-1} \frac{1}{N}  =  \left(1-\frac{1}{N}\right)^m.
\end{align*}
Using this result in \eqref{sum_in_bound} proves \eqref{sol_SE_ERM}.
\end{proof}

Employing the Stein equation and the solution of the Stein equation for a test function $h \in \mathcal{H}$, we can bound the Stein discrepancy \eqref{eq:steindiscr} between an $\IRG(\mat{p})$ and an $\IRG(\mat{p}^*)$ on the same set of vertices using \eqref{eq:opcomp}; Theorem \ref{theo_ERMM_ERMM} shows that the Stein discrepancy measures the similarity between the edge probabilities and thus provides a useful measure of similarity between two IRG models.

\begin{theorem} \label{theo_ERMM_ERMM}
    Let $\mat{X} \sim \IRG(\mat{p})$ and $\mat{Y} \sim \IRG(\mat{p}^*)$, where $p^*_s = \mathbb{P}(s=(u,v) \in \mathcal{E}^*| \ell_u, \ell_v)$, with $\ell_s$ encoding the features of the vertices in {$\bf{Y}$}. Then, for any  test function $h:\{0 , 1\}^N \rightarrow \mathbb{R}$,
    \begin{equation} \label{bound_IRG}
     |\mathbb{E}h(\mat{X}) - \mathbb{E}h(\mat{Y})| \le \|\Delta h\| \sum_{s \in E} |p_s -  p^*_s|.
\end{equation} 
In particular if $\mathcal{H} = \{ h:\{0 , 1\}^N \rightarrow \mathbb{R}:  \|\Delta h\| \le K(N)\}$ then we have for the Stein discrepancy that
$$S(\IRG(\mat{p}),  \IRG(\mat{p}^*), \mathcal{H}) \le  K(N) \sum_{s \in E} |p_s -  p^*_s|.$$
\end{theorem}

\begin{proof}
For $\mat{Y} \sim \IRG(\mat{p^*})$ the Stein equation, as in \eqref{SE_irg}, is 
\begin{equation} \label{SE_ermm2}
    \frac{1}{N} \sum_{s \in E} \left[ p^*_{l_s} \Delta_s f(\mat{y}) + \left( f(\mat{y}^{(s,0)}) - f(\mat{y}) \right) \right] = h(\mat{y}) - \E h(\mat{Y}).
\end{equation}
Next we use Lemma 2.4 in \citet{reinert2019approximating} which gives that 
\begin{equation} \label{comparison_trans}
    |\mathbb{E}h(\mat{X}) - \mathbb{E}h(\mat{Y})| \le \frac{1}{N} \sum_{s \in E} \mathbb{E} \left[|q_X(\mat{Y}^{(s,1)}|\mat{Y}) - q_Y(\mat{Y}^{(s,1)}|\mat{Y})| |\Delta_s f_h (\mat{Y})|\right],
\end{equation}
where $f_h$ is the solution of the Stein equation \eqref{SE_irg} for the distribution of $\mat{X}$. Hence, substituting the two transition probabilities in \eqref{comparison_trans} and simplifying using \eqref{sol_SE_ERM} gives the bound \eqref{bound_IRG}. The bound for the Stein discrepancy is immediate. 
\end{proof}

As a special case, we compare two ERMM graphs, on the same set of vertices and also with the same group assignments, but possibly different edge probabilities. Its proof is immediate.
\begin{corollary}
    For $\mat{X} \sim \ERMM(\mat{n}, \mat{Q})$ and $\mat{Y} \sim \ERMM(\mat{n}^*, \mat{Q}^*)$, 
    with the same number of blocks and the same group assignments $g_s$ for all $s \in E$,  the bound \eqref{bound_IRG} simplifies to 
  \[  |\mathbb{E}h(\mat{X}) - \mathbb{E}h(\mat{Y})| \le  \|\Delta h\| \sum_{i \le j}^{L} N_{i,j} |Q_{i,j} -  Q^*_{i,j}|,
\] 
where $N_{i,j} = \binom{n_i}{2}$ for $i=j$, $N_{i,j} = n_i n_j$ for $i \neq j$, $N = \sum_{i \le j}^L N_{i,j}$, and $L$ is the number of blocks. 
In particular, for $\mat{Z} \sim \text{ER}(n,p)$ and $\mat{X} \sim \ERMM(\mat{n}, \mat{Q})$ and any $h:\{0 , 1\}^N \rightarrow \mathbb{R}$, we bound
    \begin{equation} \label{bound_ERMM_ER}
     |\mathbb{E}h(\mat{X}) - \mathbb{E}h(\mat{Z})| \le \|\Delta h\| \sum_{i \le j}^{L} N_{i,j} |Q_{i,j} -  p|.
\end{equation}
\end{corollary} 
\begin{remark} \label{bound_convergence}
Functions $h$ of interest include proportions of subgraphs of different types, where a count of a connected subgraph on $v\ge 2$ vertices is divided by $n^v$, as used in graphon convergence. For such functions $h$, the bound $\|\Delta h\| \le O(N^{-1})$ shows that \eqref{bound_ERMM_ER} is small when, on average, the edge probabilities in IRG are close to $p$. 

As an example we consider the test function $h$, in the bound \eqref{bound_IRG}, to be the proportion of triangles in the network; then  $\| \Delta h \| \le \frac{n-2}{\binom{n}{3}} = \frac{3}{\binom{n}{2}}$. Hence, the bound \eqref{bound_IRG} simplifies to 
\begin{equation}
 \label{eq:boundondis}   
    |\mathbb{E}h(\mat{X}) - \mathbb{E}h(\mat{Y})| \le  \frac{3}{N} \sum_{s \in E} |p_s -  p^*_s|.
\end{equation}
Table \ref{dis_boundsold} uses this bound to give distances between the models tested for fit for Lazega’s lawyers’ friendship network in Section \ref{sec:applications} and Section \ref{sec:applic_contd} of this Supplementary Material, using $K(N) = 3 N^{-1}$. For Lazega's lawyers' friendship network, the discrepancy is smallest between an ER graph and an ERMM when the ERMM is based on the grouping created using the law school which the lawyers attended; groupings created using the other explanatory variables explored show a larger discrepancy.
\end{remark} 

\begin{table}
\caption{The bounds from \eqref{eq:boundondis} on the Stein discrepancies between IRG models for Lazega's lawyers' friendship network} 
\label{dis_boundsold}
\centering
\begin{tabular}{lccc}
& \multicolumn{3}{@{}c@{}@{}}{\textbf{STEIN DISCREPANCY WITH} $K(N)=3*N^{-1}$}  \\
\cmidrule{2-4}
 \textbf{GROUP LABELS}  & \multicolumn{2}{@{}c@{}@{}}{\textbf{ER$(n,p)$ to}} & \textbf{ERMM to}  \\
 &\textbf{ ERMM} & \textbf{DCSBM} & \textbf{DCSBM} \\
\hline \\
Single group & 0 & 0.5638 & - \\
 Status & 0.5326 & 0.6517 & 0.5276 \\ 
 Gender & 0.1940 & 0.5726 & 0.5576 \\ 
 Office & 0.6155 & 0.6616 & 0.4713 \\
 Practice & 0.2003 & 0.5698 & 0.5554 \\
 Law school & 0.1717 & 0.5686 & 0.5571 \\
\end{tabular}
\end{table}

\section{REPRODUCING KERNEL HILBERT SPACES} \label{sec:RKHS}

In this section, we present the concept of a Reproducing kernel Hilbert space (RKHS), and a discussion on the RKHS used for IRG-gKSS and the related kernels.

\subsection{Definition of a Reproducing Kernel Hilbert Space}

A reproducing kernel Hilbert space (RKHS) is a Hilbert space of functions in which evaluation at any point can be expressed as an inner product with a kernel function. This property makes RKHS particularly useful in statistical learning and, in our context, in defining Stein discrepancies in a computationally tractable way. Here we give a brief introduction; for more details, see, for example, \cite{berlinet2011reproducing}.

Formally, let $\mathcal{X}$ be a non-empty set, let $\mathcal{H}$ be a Hilbert space of real-valued functions on $X$ with inner product $\langle \cdot, \cdot \rangle$, and let $k:\mathcal{X} \times \mathcal{X} \to \mathbb{R}$ be a symmetric positive-definite kernel function. For any $t \in \mathcal{X}$, we define $e_t$, the ``evaluation functional'' at point $t \in X$, such that for all $f \in \HH$, $e_t[f] = f(t).$ This functional is bounded if there exists an $M$ such that $e_t [f] \le M || f||_\HH$ for all $f \in \HH$. The Hilbert space $\HH$ is an RKHS if for each $t \in \mathcal{X}$ its evaluation functional is a bounded linear functional.  
If $\HH$ is an RKHS then, for every function $f \in \HH$,  by the Moore–Aronszajn theorem, there exists a function $K_t$ of $\HH$ such that 
$$ e_t[f] = \langle K_t, f\rangle_K := f(t). $$
In particular, for each $x \in \mathcal{X},$
$$ K_t(x) = \langle K_t, K_x \rangle_K.$$
The {\it reproducing kernel} of $\HH$ is $$k(t, x):= K_t(x).$$ One can show that for every RKHS the reproducing kernel is positive definite and, conversely, for every positive definite kernel $k$ on $\mathcal{X} \times \mathcal{X}$ there is a unique RKHS $\HH_k$ on $\mathcal{X}$ with $k$ as its reproducing kernel. Moreover $\HH_k$ has an inner product $\langle \cdot, \cdot \rangle=\langle \cdot, \cdot \rangle_{\mathcal{H}_k}$, such that 
\begin{enumerate}
    \item For every $x \in \mathcal{X}$, the function $k(x,\cdot)$ belongs to $\mathcal{H}_k$.
    \item The {\it reproducing property} holds: $$f(x) = \langle f, k(x, \cdot) \rangle_{\mathcal{H}_k}, \quad \forall f \in \mathcal{H}_k.$$
\end{enumerate}

The reproducing property ensures that evaluation of any function in $\mathcal{H}_k$ at a point $x$ is continuous and can be written as an inner product. The kernel $k$ thus acts as a \emph{feature map}, implicitly embedding data into a high-dimensional (possibly infinite-dimensional) feature space.

\subsection*{Geometric Intuition}
An RKHS can be viewed as a space where each point $x \in \mathcal{X}$ is represented by a feature vector $k(x,\cdot)$. Probability distributions $P$ and $Q$ on $\mathcal{X}$ can then be embedded into this space via their mean embeddings,
$$\mu_P := \E_{X \sim P}[k(X,\cdot)], \qquad \mu_Q := \E_{Y \sim Q}[k(Y,\cdot)].$$
The distance between $\mu_P$ and $\mu_Q$ in $\mathcal{H}_k$ captures differences between the two distributions. With a suitably chosen kernel $k$, this embedding is injective, meaning that $\mu_P = \mu_Q$ if and only if $P = Q$. 

This geometric perspective provides the bridge to kernelised Stein discrepancies as detailed in the next section: by restricting the function class in the Stein discrepancy to the unit ball of an RKHS, one can both retain strong discriminative power and obtain closed-form expressions for the resulting discrepancy measure. We now turn to this construction.

\subsection{Kernelised Stein Discrepancy}
Building on the concept of Stein discrepancy, \citet{chwialkowski2016kernel} and \citet{liu2016kernelized} introduced the kernelised Stein discrepancy (KSD). This approach restricts the function class $\mathcal{F}$ in the Stein discrepancy to the unit ball of a RKHS $\mathcal{H}$ with kernel $k$ and inner product $\langle \cdot, \cdot \rangle$. The KSD between distributions $p$ and $q$ is defined as
\begin{equation} \label{eq:ksd}
     \mbox{KSD}(p,q; k) =\sup_{f \in B_1(\mathcal H)} |  \mathbb{E}[\AAA_p f (Z) ] |,
\end{equation}
where $Z \sim q$ and $B_1(\mathcal{H})$ denotes the unit ball of $\mathcal{H}$. 

Denoting $\HH$, a RKHS of real valued function on $\R^d$ with reproducing kernel $k$ and an inner product $\langle \cdot, \cdot \rangle_{\HH}$, and by $\HH^d$ the product RKHS consisting of elements $f = (f_1, \ldots, f_d)$ with $f_i \in \HH$ and $f \in \HH^d$ and an inner product $\langle f, g \rangle_{\HH^d} = \sum_{i=1}^d \langle f_i, g_i \rangle_{\HH}$, the Langevin Stein operator is 
\begin{eqnarray}\label{eq:langevin2}
  (\AAA_p f)(x) = \sum_{i=1}^d \left( \frac{\partial \log p(x)}{\partial x_i} f_i(x) + \frac{\partial f_i(x)}{\partial x_i}\right), 
\end{eqnarray}
where $\frac{\partial}{ \partial x_i}$ is the partial derivative with respect to the element $x_i$. For the Langevin Stein operator, \cite{chwialkowski2016kernel} define the  function 
$$\xi_p(x, \cdot):= \nabla \log p(x) k(x,\cdot) + \nabla k(x,\cdot),$$
which depends on both the gradient of the log-density of $p$ and the derivatives of the kernel. Further since $\xi(x,\cdot) \in \HH^d$, by \cite{steinwart2008support} (Lemma 4.34), $\nabla k(x,\cdot) \in \HH$, and $\frac{\partial \log p(x)}{\partial x_i}$ is a scalar, giving  
\begin{align*}
   \langle \xi_p(x, \cdot), \xi_p(y, \cdot) \rangle  = &\nabla \log p(x)^\top\nabla \log p(y) k(x,y)  + \nabla \log p(y)^\top\nabla_x k(x,y) \\
& + \nabla \log p(x)^\top\nabla_y k(x,y)  + \langle \nabla_x k(x, \cdot), \nabla_y k(\cdot,y) \rangle_{\HH^d} \\
& = h_p(x,y),
\end{align*}
where the last term can be written as $\sum_{i=1}^d \frac{\partial k(x,y)}{\partial x_i \partial y_i}$, and
$$\nabla_x k(x,\cdot) = \left( \frac{\partial k(x,\cdot)}{\partial x_1} , \ldots,  \frac{\partial k(x,\cdot)}{\partial x_d}\right), \quad \nabla_y k(\cdot,y) = \left( \frac{\partial k(\cdot,y)}{\partial y_1} , \ldots,  \frac{\partial k(\cdot,y)}{\partial y_d}\right).$$
For any random variable $Z$
$$\E\|\xi_p(Z)\|_{\HH^d} \le \E\|\xi_p(Z)\|_{\HH^d}^2 = \E h_p(Z,Z) < \infty,$$
where $\|\cdot\|_{\HH^d}$ denotes the norm induced by the inner product $\langle \cdot, \cdot \rangle_{\HH^d}$. Hence $\xi$ is Bochner integrable (see Definition A.5.20 of \cite{steinwart2008support}). We also note that a linear operator $\langle f_i, \cdot \rangle$ with $f_i \in \mathcal{H}$ can be interchanged with the Bochner integral. Hence, for any $f \in  \mathcal{H}^d$, and $\AAA_p$ as in \eqref{eq:langevin2},
\begin{align*}
    \langle f, \E \xi_p (Z) \rangle_{\HH^d} = & \sum_{i=1}^d \langle f_i, \E \xi_{p,i} (Z) \rangle_{\HH} \\
    =& \sum_{i=1}^d \left \langle f_i, \E\left[ \frac{\partial \log p(Z)}{\partial x_i} k(Z, \cdot) + \frac{\partial k(Z, \cdot)}{\partial x_i}\right]\right \rangle_{\HH} \\
     = &  \sum_{i=1}^d \E \left \langle f_i,  \frac{\partial \log p(Z)}{\partial x_i} k(Z, \cdot) + \frac{\partial k(Z, \cdot)}{\partial x_i} \right \rangle_{\HH} \\
    = & \E \sum_{i=1}^d \left[  \frac{\partial \log p(Z)}{\partial x_i} f_i(Z) + \frac{\partial f_i(Z)}{\partial x_i} \right] \\
    = & \E \AAA_p f(Z),
\end{align*}
where the second to last equality is due to the reproducing property of the kernel. Taking the supremum over all $f \in B_1(\mathcal H)$ gives 
\begin{align*}
    \sup_{f \in B_1(\mathcal H)}   \E[\AAA_p f (Z) ] = \| \E \xi_p (Z) \|_{\HH^d}. 
\end{align*}
Squaring the above gives the closed-form KSD:
\begin{align*}
    \mbox{KSD}(p,q; k)^2 & = \langle \E \xi_p (Z), \E \xi_p (Z) \rangle_{\mathcal{F}^d} = \E \langle \xi_p (Z), \E \xi_p (Z) \rangle_{\mathcal{F}^d} \\
    & = \E \langle \xi_p (Z), \xi_p (Z') \rangle_{\mathcal{F}^d} = \E h_p(Z,Z')
\end{align*}
where
\begin{align*}
    h_p(x,y) := &\nabla \log p(x)^\top\nabla \log p(y) k(x,y)  + \nabla \log p(y)^\top\nabla_x k(x,y) \\
& + \nabla \log p(x)^\top\nabla_y k(x,y)  + \langle \nabla_x k(x, \cdot), \nabla_y k(\cdot,y) \rangle_{\mathcal{F}^d}.
\end{align*}

In practice, this expectation can be estimated, for example, by the V-statistic
$$  \widehat{ \mbox{KSD}(p,q; k)^2}  =  \frac{1}{n^2} \sum_{i=1}^n  \sum_{j=1}^n h_p (z_i, z_j).$$
where $(z_i)_{i=1}^n$ are i.i.d.\,samples from $q$. This statistic does not typically have a closed-form asymptotic distribution as $n\rightarrow \infty$, but it can be used in a Monte Carlo testing framework to assess the fit of the distribution $p$.

While the KSD has proven effective for continuous distributions, applying it to network data raises two challenges. First, the distribution is discrete over the space of networks, requiring a different Stein operator. Second, there is often only a single observed network, making it challenging to estimate the KSD directly. To address these issues, \citet{xu2021stein} introduce the graph kernel Stein statistic (gKSS), to use in a Monte Carlo test for assessing the fit of an Exponential Random Graph Model (ERGM). In this paper, we have developed a similar test for IRG models.

\subsection{Vector-Valued RKHS}

As discussed in \cite{xu2021stein}, to accommodate the conditional probabilities embedded in the Stein operator, a vector-valued reproducing kernel Hilbert space (vvRKHS), introduced by Xu and Reinert (2021), is more appropriate for IRG-gKSS than a standard RKHS. We summarise the vvRKHS construction as presented in \cite{xu2021stein}.

Let $E = {(u,v): 1 \le u < v \le n}$ denote the set of vertex pairs in a graph $\mat{x} \in \{0,1\}^N$. For $s \in E$, let $\mat{x}_{-s}$ denote the collection $\{x_{r,t}, 1 \le r < t \le n, (r,t) \neq (u,v) \}$, without the $s=(u,v)$- coordinate, let $\{0, 1\}^{N-1} =: \mathcal{X}^{-s}$ denote the set of collections  of edge indicators excluding vertex pair $s$, and let $x_s \in \{0,1\} =: \mathcal{X}^s$ denote the edge indicator for $s$. For each $s \in E$, let $l_s: \mathcal{X}^s \times \mathcal{X}^s \to \R$ be a reproducing kernel with associated RKHS $\mathcal{H}_{l_s}$, and let $\varphi_s: x_s \mapsto l_s(\cdot, x_s) \in \mathcal{H}_{l_s}$ denote the corresponding feature map.

The kernels $l_s$ are scalar-valued, whereas the kernel $\ell_{-s}$ takes values in $\mathcal{L}(\mathcal{H}_{l_s})$, the Banach space of bounded linear operators from $\mathcal{H}_{l_s}$ to itself. The space associated with $\ell_{-s}$ is referred to as a vvRKHS in \cite{xu2021stein}. All kernels considered here are assumed to be positive definite and bounded. As a composition of kernels, we define
$$K: (\mathcal{X}^{s} \otimes \mathcal{X}^{-s}) \times (\mathcal{X}^{s} \otimes \mathcal{X}^{-s}) \to \R$$
with associated RKHS $\mathcal{H}_K$; here, $\otimes$ denotes the Kronecker product. Assuming that $l_s \equiv l$ for all $s \in E$ and the vvRKHS $\HH_{\ell}$ is of the form
$$
\ell (x^{-s}, \mat{x}'_{-s'}) 
= k(x^{-s}, \mat{x}'_{-s'}) \mathbb{I}_{\HH_l\times \HH_l},
$$
where $\mathbb{I}_{\mathcal{H}_l \times \mathcal{H}_l}$ is the identity operator on $\mathcal{H}_l$ and $k$ is the chosen graph kernel. The composed RKHS kernel is then
$$K((x^s, \mat{x}_{-s}) 
, ((x')^{s'}, \mat{x}'_{-s'})) = k(x^{-s}, \mat{x}'_{-s'})
l(x^{s}, (x')^{s'}).$$

For a single observed network $\mat{x}$, since $\mathcal{H}_l$ and $\mathcal{H}_{\ell}$ are the same for all $s$, it follows that for any $s, s' \in E$,
$$K((x^s,\mat{x}_{-s})
), (x^{s'}, \mat{x}'_{-s'})) = l(x^{s}, x^{s'}).$$
In our implementation, the graph kernels $k(x^{-s}, \cdot) = k(x^{(s,1)}, \cdot) + k(x^{(s,0)}, \cdot)$ are defined on the set $\mat{x}_{-s}$ rather than on the entire graph $\mat{x}$; Section \ref{app:kernel} details the graph kernels used in this paper.

%%%%%%%%%%%%%%%%%%%%%%%%%%%%%%%%%%%%%%%%%%%%%%%%%%%%%%%%%%%%

\end{document}